\documentclass{article} 
\usepackage[preprint, nonatbib]{neurips_2020}
\usepackage{neurips_2020}
\usepackage[utf8]{inputenc} 
\usepackage[T1]{fontenc}    
\usepackage{hyperref}       
\usepackage{url}            
\usepackage{booktabs}       
\usepackage{amsfonts}       
\usepackage{nicefrac}       
\usepackage{thm-restate}
\usepackage{microtype}      
\usepackage{xcolor}
\usepackage{graphicx}
\usepackage{bm}
\usepackage{bbm}
\usepackage{thm-restate}
\usepackage{multirow,enumitem}
\usepackage{amsmath,amsfonts,bm}
\usepackage{amssymb,amsthm}
\renewcommand{\[}{\begin{eqnarray}}
\renewcommand{\]}{\end{eqnarray}}

\DeclareMathOperator*{\mathbb{E}}{\mathbb{E}}
\DeclareMathOperator*{\E}{\mathbb{E}}

\newcommand{\bk}{{\bf k}}
\newcommand{\bu}{{\bf u}}
\newcommand{\bv}{{\bf v}}
\usepackage{mathbbol}

\newtheorem{lemma}{Lemma}

\newtheorem{definition}{Definition}
\newtheorem{proposition}{Proposition}
\newtheorem{remark}{Remark}

\author{%
Etai Littwin\thanks{Equal Contribution} \\
School of Computer Science \\
Tel Aviv University \\
Tel Aviv, Israel \\
\texttt{etai.littwin@gmail.com} \\
\And
Tomer Galanti$^*$ \\
School of Computer Science \\
Tel Aviv University \\
Tel Aviv, Israel \\
\texttt{tomerga2@tauex.tau.ac.il} \\
\AND
Lior Wolf \\
Facebook AI Research (FAIR) \& \\
School of Computer Science \\
Tel Aviv University \\
Tel Aviv, Israel \\
\texttt{wolf@fb.com} \\
}

\begin{document}
\title{On Random Kernels of Residual Architectures}
\maketitle



\begin{abstract}
We derive finite width and depth corrections for the Neural Tangent Kernel (NTK) of ResNets and DenseNets. Our analysis reveals that finite size residual architectures are initialized much closer to the ``kernel regime'' than their vanilla counterparts: while in networks that do not use skip connections, convergence to the NTK requires one to fix the depth, while increasing the layers' width. Our findings show that in ResNets, convergence to the NTK may occur when depth and width simultaneously tend to infinity, provided with a proper initialization. In DenseNets, however, convergence of the NTK to its limit as the width tends to infinity is guaranteed, at a rate that is independent of both the depth and scale of the weights. Our experiments validate the theoretical results and demonstrate the advantage of deep ResNets and DenseNets for kernel regression with random gradient features.
\end{abstract}


\section{Introduction}

Understanding the effect of different architectures on the ability to train deep networks has long been a major research topic. A popular playing ground for studying the forward and backward propagation of signals at the point of initialization, is the ``infinite width'' regime~\cite{Neal1996PriorsFI,infinite,deep_info,mean_res, GP,exact}. In this regime, Gaussian Process behaviour emerges in pre-activations, when the weights are sampled i.i.d from a normal distribution, giving rise to tractable training dynamics~\cite{NTK,NIPS2019_9063,GP,g.2018gaussian}. 

This notion was first made precise by the Neural Tangent Kernel (NTK) paper~\cite{NTK}, in which it is shown that the training dynamics of fully connected networks trained with gradient descent can be characterized by a kernel, when the width of the network approaches infinity. Specifically, the evolution through time of the function computed by the network follows the dynamics of kernel regression. 
Let $f(x;w) \in \mathbb{R}$ denote the output of a fully connected feed forward network of width $n$, with i.i.d normally distributed weights $w$ and input $x \in \mathbb{R}^{n_0}$. The \emph{neural tangent kernel} (NTK) is given by: $\mathcal{G}(x,x';w) := \frac{\partial f(x;w)}{\partial w} \cdot \frac{\partial^\top f(x';w)}{\partial w}$. As the width of each layer approaches infinity, provided with proper scaling and initialization of the weights, it holds that $\mathcal{G}(x,x';w)$ converges in probability to the infinite width limit kernel function:
\begin{equation}\label{eq:limit}
\lim_{n \rightarrow \infty}\mathcal{G}(x,x';w) = \mathcal{K}(x,x')
\end{equation}
As shown in~\cite{NTK}, when the width tends to infinity, minimizing the squared loss $\mathcal{L}(w)$ using gradient descent is equivalent to a kernel regression with kernel $\mathcal{K}$.

Recent empirical support has demonstrated the power of NTK and CNTK (convolutional neural tangent kernel) on practical datasets, showing new state of the art results for kernel methods, surpassing other known kernels by a large margin~\cite{exact,CNTK,arora2020harnessing}.
It is, therefore, interesting to understand how far the training dynamics of practically sized architectures deviate from the ``infinite width'' regime. To that end, an important subtlety worth considering is the rate of convergence in Eq.~\ref{eq:limit}, and its dependence on other hyper parameters, such as, depth and scale. This question has recently been addressed in the case of vanilla feed forward fully connected networks~\cite{finite_ntk}, where it is shown that the normalized variance of the diagonal entries of the NTK is exponential in the ratio between the depth $L$ and width $n$:
\begin{equation}\label{eq1}
\frac{Var\big(\mathcal{G}(x,x;w)\big)}{\mathbb{E}[\mathcal{G}(x,x;w)]^2} \sim \exp \Big[\frac{CL}{n}\Big] - 1
\end{equation}
where $C>0$ is a constant. Hence, convergence to the limiting kernel cannot happen when both are taken to infinity at the same rate. {From Eq.~\ref{eq1} it is evident that for an $L$-depth vanilla network, the width should be at least $\Omega(L)$ in order to maintain a fixed ratio in the exponent of Eq.~\ref{eq1}. In this case, the total parameter complexity of the network is at least $\Omega(L^3)$. This important observation suggests that deep and narrow vanilla networks operate far from the ``infinite width'' regime at initialization. In this work, we derive finite width and depth corrections to the NTK of residual and densely connected architectures, revealing a depth invariant property unique to these architectures. From this analysis it is evident that, in contrast to vanilla ReLU networks, the required parameter complexities of $L$-depth ResNets and DenseNets is as small as $\mathcal{O}(L)$ and $\mathcal{O}(L^2)$ (resp.) in order to maintain a bounded normalized variance.} 

However, the presented analysis of the asymptotic behaviour of the ratio in Eq.~\ref{eq1} is lacking, since only individual entries along the diagonal are investigated, and it does not consider the joint distribution of the full NTK matrix. To present a more complete analysis, we conduct extensive empirical experiments on MNIST and multiple small UCI datasets using random draws of $\mathcal{G}$ as kernel approximations, demonstrating the power of random gradient features $\nabla_w f(x;w)$ of deep residual architectures. Surprisingly, for fixed width ResNets and DenseNets, the performance of kernel regression using $\mathcal{G}$ as a substitute for $\mathcal{K}$ improve with depth and approach the latter, whereas in vanilla architectures, clear degradation is observed.


Our main contributions are as follows. 
\begin{enumerate}[leftmargin=*]
\item Thms.~\ref{thm:eq} and~\ref{thm:dual} introduce a forward-backward norm propagation duality for a wide family of ReLU feedforward architectures, which is a useful tool for analyzing the rate of convergence of $\mathcal{G}(x,x;w)$, for finite sized networks.
\item In Thms.~\ref{thm:res_ntk_var} and~\ref{thm:dense_ntk_var}, we rigorously derive finite width and depth corrections for ResNet and DenseNet architectures, revealing a fundamentally different relationship between width, depth and $\mathcal{G}(x,x;w)$. Unlike vanilla architectures, when properly scaled, convergence to the limiting kernel is achieved, when taking both the width and the depth of the architecture to infinity simultaneously. 
\item Our experiments validate the convergence rates of both the diagonal $\mathcal{G}(x,x;w)$ and off-diagonal $\mathcal{G}(x,x';w)$ NTK terms. In addition, they demonstrate the advantage of deep ResNets and DenseNets over vanilla networks for kernel regression with random gradient features on MNIST and multiple small UCI datasets.
\end{enumerate}



\section{Preliminaries And Notations}\label{sec:arch}

Throughout the paper, we make use of the following notations. Let $f(x;w) \in \mathbb{R}$ denote the output of a parameterized function $f$ on input $x \in \mathbb{R}^{n_0}$ with a vector $w$ of real valued parameters. Throughout the paper, we assume that the coordinates of $w$ are i.i.d and normally distributed. With no loss of generality, we also assume that $\|x\|_2=1$. The ReLU non-linearity is denoted by $\phi(x) := \max(0,x)$. The intermediate outputs of a neural network are denoted by $\{y^l(x)\}_{l=0}^L$ (see Eqs.~\ref{constant_res} and~\ref{constant_dense}), for a fixed input $x\in \mathbb{R}^{n_0}$. For simplicity, the dependence of the outputs on $x$ is often made implicit $\{y^l\}_{l=0}^L$ when the specific input used to calculate the outputs can be inferred from context. $y_i^l$ denotes the $i$'th component of the vector $y^l$, and $n_1,...,n_L$ denote the width of the corresponding layers, with $n_0$ the length of the input vector. We denote by $\|x\|_2$ the Euclidean norm of the vector $x$ and by $\|W\|_2$ the Frobenius norm of the matrix $W$. We denote the weight matrix associated with layer $l$ by $W^l$, with lower case letters $w_{i,j}^l$ denoting the individual components of $W^l$. Additional superscripts $W^{l,k}$ are used, when several weight matrices are associated with layer $l$. Weights appearing without any superscript $w$ denote all the weights concatenated into a vector. 
The NTK of the function $f$ is denoted by $\mathcal{G}(x,x';w) := \frac{\partial f(x;w)}{\partial w} \cdot \frac{\partial^{\top} f(x';w)}{\partial w}$.


{\bf Residual networks} have reintroduced the concept of bypass connections~\cite{ResNets}, allowing the training of deep and narrow models with relative ease. A generic, residual architecture $f(x;w)$, with residual branches of depth $m$, takes the form: $f(x;w) = \frac{1}{\sqrt{n_L}} \cdot W^{L} \cdot y^L$, where, for all $l \in [L]$, $y^l$ is defined recursively as follows:
\begin{equation}\label{constant_res}
\begin{aligned}
y^l = 
\begin{cases}
\frac{1}{\sqrt{n_0}} \cdot W^0 x & l=0\\
y^{l-1} + \sqrt{\alpha_l}y^{l-1,m} &o.w
\end{cases}~\textnormal{ and }~
y^{l-1,h} = 
\begin{cases} 
      \sqrt{\frac{1}{n_{l-1,h-1}}}W^{l,h}q^{l-1,h-1} & 1<h\leq m  \\
      \sqrt{\frac{1}{n_{l-1}}} \cdot W^{l,h}y^{l-1} & h=1   \end{cases}
\end{aligned}
\end{equation}

Here, $\{\alpha_l\}_{l=1}^L$ are scaling coefficients, $W^0 \in \mathbb{R}^{n'_0 \times n_0}$, $W^{l,h} \in \mathbb{R}^{ n_{l-1,h} \times n_{l-1,h-1}},W^{l,1} \in \mathbb{R}^{n_{l-1,1} \times n_{l-1}},W^{l,m} \in \mathbb{R}^{n_l \times n_{l-1,m-1}}$, $q^{l,h} = \sqrt{2}\phi(y^{l,h})$ (see Fig.~\ref{fig:figs} for an illustration).


\begin{figure*}[t]
\centering
\begin{tabular}{cc}
\includegraphics[ width=.45\linewidth]{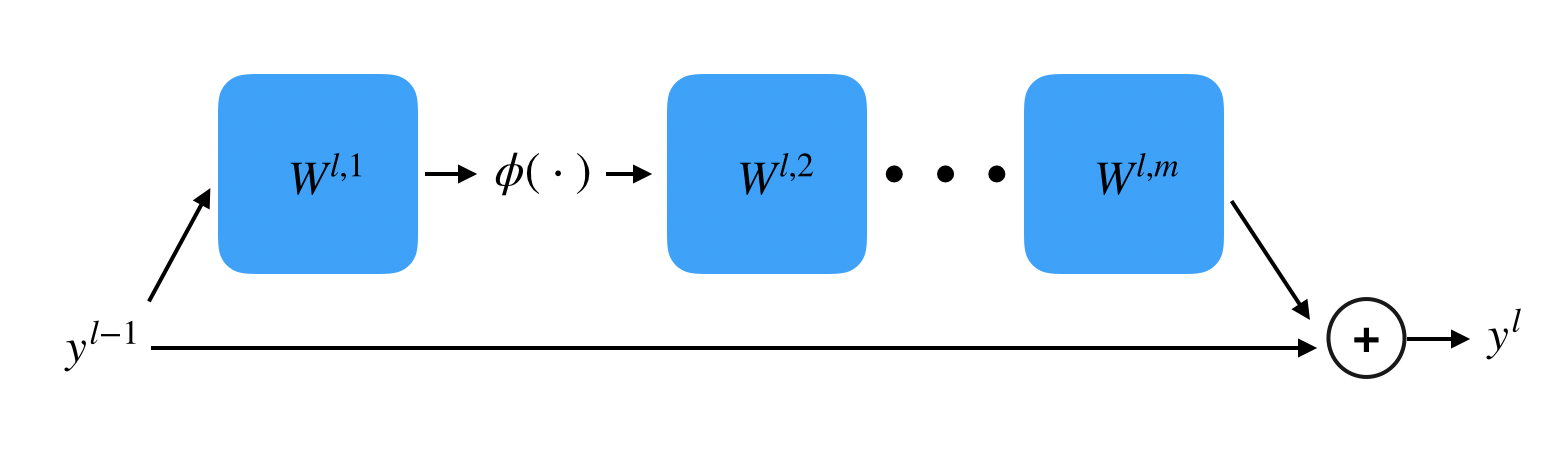}&
\includegraphics[width=.45\linewidth]{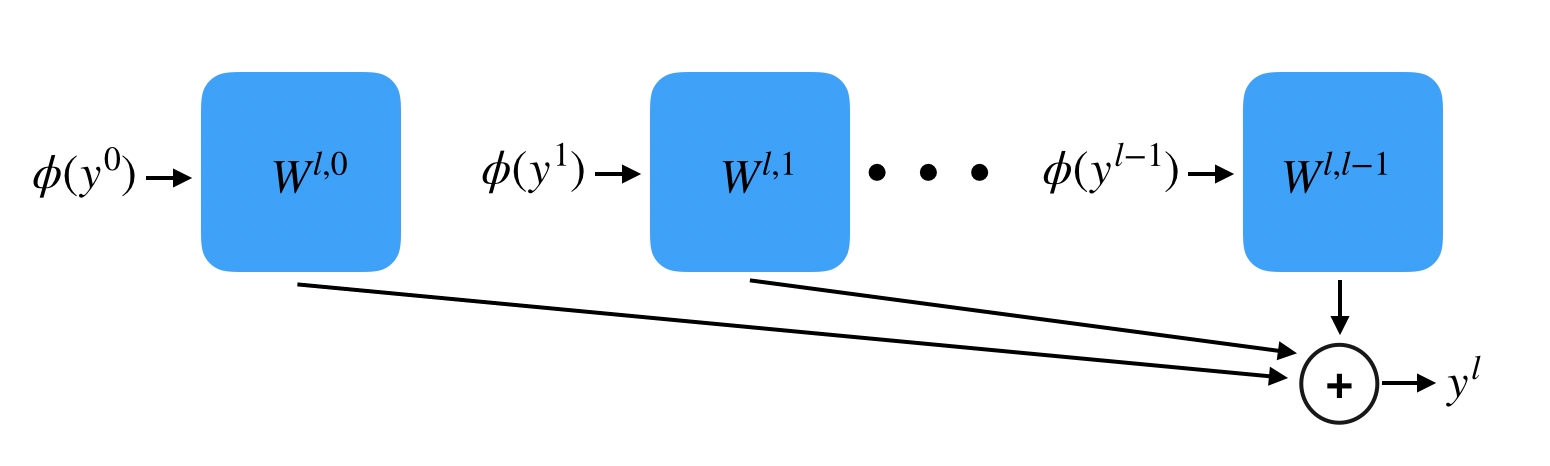}\\
(a) & (b) \\
\end{tabular}
  \caption{An illustration of {\bf (a)} ResNet and {\bf (b)} DenseNet, as given in Eqs.~\ref{constant_res} and~\ref{constant_dense} (with constant width and absent scaling coefficients).} 
  \label{fig:figs}
\end{figure*}

{\bf DenseNets} were recently introduced~\cite{DenseNets}, demonstrating faster training, as well as improved performance on several popular datasets. The main architectual features introduced by DenseNets include the connection of each layer output to all subsequent layers, using concatenation operations, instead of summation, such that the weights of layer $l$ multiply the concatenation of the outputs $y^0,...,y^{l-1}$. A DenseNet $f(x;w)$ is defined in the following manner: $f(x;w) := \frac{1}{\sqrt{n_L}} \cdot W^L \cdot y^L$, where, for all $l \in [L]$, $y^l$ is defined recursively as follows:
\begin{equation}\label{constant_dense}
\begin{aligned}
y^l = 
\begin{cases}
\frac{1}{\sqrt{n_0}} \cdot W^0 x & l=0 \\
\sqrt{\frac{\alpha}{n_{l-1} \cdot l}}\sum_{h=0}^{l-1}W^{l,h} q^h & o.w
\end{cases}
\end{aligned}
\end{equation}
where $\alpha$ is a scaling coefficient and
$W^{l,h} \in \mathbb{R}^{n_l \times n_{l-1}}$ (see Fig.~\ref{fig:figs} for an illustration).

\section{Forward-Backward Norm Propagation Duality}\label{sec:dual}

In this work, we aim to derive an expression for the first and second moments of the diagonal entries $\mathcal{G}(x,x;w)$ at the point of initialization $w$, given by the Jacobian squared norm evaluated on $x$:
\begin{equation}\label{ee}
\mathcal{G}(x,x;w) = \|J(x)\|^2_2 = \sum_{\bk}\|J^{\bk}(x)\|^2_2 
\end{equation}
where $J^{\bk}(x) := \frac{\partial f}{\partial W^{\bk}}$ denotes the per-weight Jacobian. Bold letters (a.k.a $\bk,\bu,\bv$) stand for identities of matriices in the network. For instance, in ResNets, $\bk$ can take values in $\{0,L\} \cup [L] \times [m]$. The sum $\sum_{\bk}\|J^{\bk}\|^2_2$ denotes summation over every weight matrix in the network.
In the following analysis, we assume that the output of $f$ is computed using a single fixed sample $x$. 
To facilitate our derivation, we introduce a link between the propagation of the norm of the activations, and the norm of the per-layer Jacobian in random ReLU networks of finite width and depth. This link will then allow us to study the statistical properties of the full Jacobian in general architectures incorporating residual connections and concatenations with relative ease.  Specifically, we would like to establish a connection between the first and second moments of the squared norm of the output $f(x;w)^2$, and those of the per layer Jacobian norm $\|J^{\bk}\|^2_2$. Using a path-based notation, for any weight matrix $W^{\bk}$, the output $f(x;w)$ can be decomposed to paths that go through $W^{\bk}$ (i.e, paths that include weights from $W^{\bk}$), denoted by $f_{\bk}(x;w)$, and paths that skip $W^{\bk}$, denoted by the complement $f^c_{\bk}(x;w)$. Namely: 
\begin{equation}\label{eq:complement}
f(x;w) = f_{\bk}(x;w)+f^c_{\bk}(x;w) 
=\sum_{\gamma \in S_{\bk}}c_\gamma z_{\gamma}\prod_{l=1}^{|\gamma|} w_{\gamma,l} + \sum_{\gamma \in S \setminus S_{\bk} }c_\gamma z_{\gamma}\prod_{l=1}^{|\gamma|} w_{\gamma,l} 
\end{equation}
where the summations are over paths $\gamma \in S$ from input to output, with $|\gamma|$ denoting the length of the path, and $c_\gamma$ a scaling factor. In standard fully connected networks, we have $|\gamma| = L+2$ (when considering the initial and final projections $W^0,W^L$) and the total number of paths is $\prod_{l=0}^{L} n_l$. The term $z_{\gamma}\prod_{l=1}^{|\gamma|} w_{\gamma,l}$ denotes the product of weights along path $\gamma$, multiplied by a binary variable $z_{\gamma} \in \{0,1\}$, indicating whether path $\gamma$ is active (i.e all relevant activations along the specific path are on). The set $S_{\bk}$ indicates the set of all paths that include weights from $W^{\bk}$. 

We make the following definition:
\begin{definition}[Reduced network]\label{def}
Let $f(x;w)$ be a neural network (e.g., vanilla network, ResNet or DenseNet). We define the reduced network $f_{(\bk)}(x;w)$ to be the neural network obtained by removing all connections bypassing weights $W^{\bk}$ from the network $f(x;w)$. The corresponding hidden layers of $f_{(\bk)}(x;w)$ are denoted by $y^0_{(\bk)},...,y^L_{(\bk)}$ and its weights by $w_{(\bk)}$.
\end{definition}

Note that for vanilla networks, it holds that, for all $ \bk \in [L]$, we have: $f_{(\bk)}(x;w) = f_{\bk}(x;w) = f(x;w)$ and $y^l_{(\bk)} = y^l$. In the general case, the equality $f_{(\bk)}(x;w) = f_{\bk}(x;w)$ does not hold, since $f_{(\bk)}(x;w)$ contains different activation patterns, induced by the removal of residual connections.
The following theorem states that the moments of both are equal in the family of considered ReLU networks (see Fig.~\ref{fig:dualthm} for an illustration):
\begin{restatable}{theorem}{eq}\label{thm:eq}
Let $f(x;w)$ be a ResNet/DenseNet, as described in Sec.~\ref{sec:arch}. Then, for any non-negative even integer $m$, we have:
\begin{equation}\label{eq:equi}
\forall~\bk :~ \mathbb{E}_w\left[(f_{(\bk)}(x;w))^{m} \right] =  \mathbb{E}_w\left[(f_{\bk}(x;w))^{m} \right]
\end{equation}
\end{restatable}

\begin{figure*}[t]
\centering
\begin{tabular}{cc}
\includegraphics[trim=0cm 0cm 0cm 0cm, clip, width=.49\linewidth]{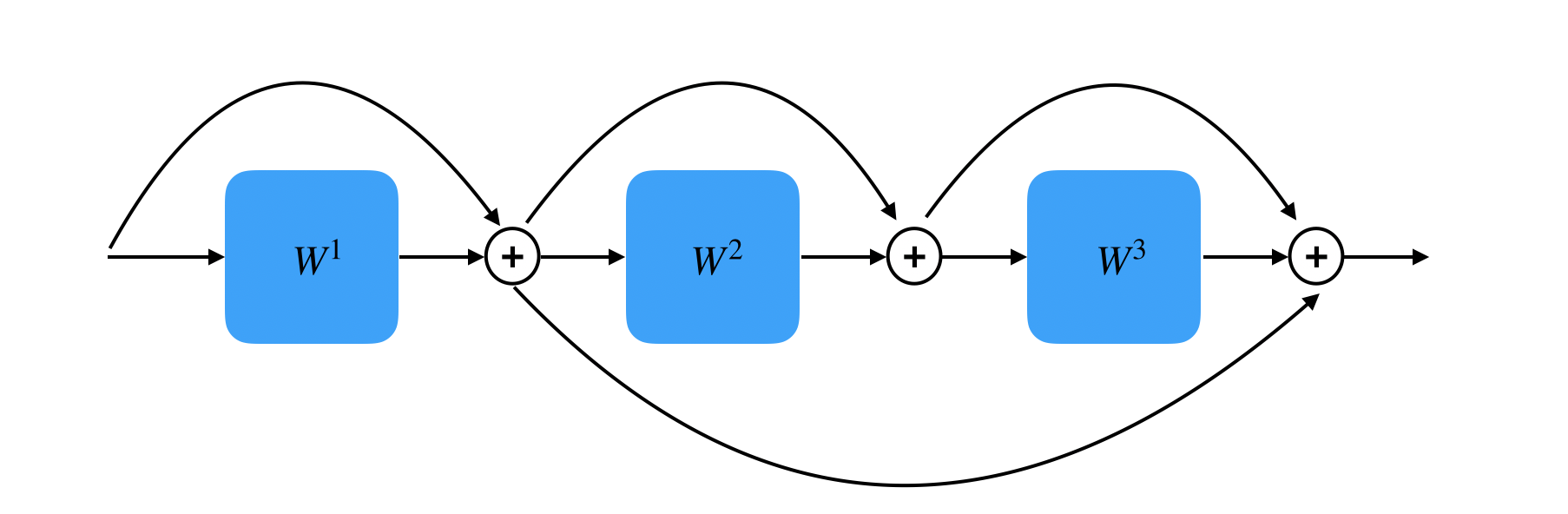}&
\includegraphics[trim=0cm 0cm 0cm 0cm, clip, width=.49\linewidth]{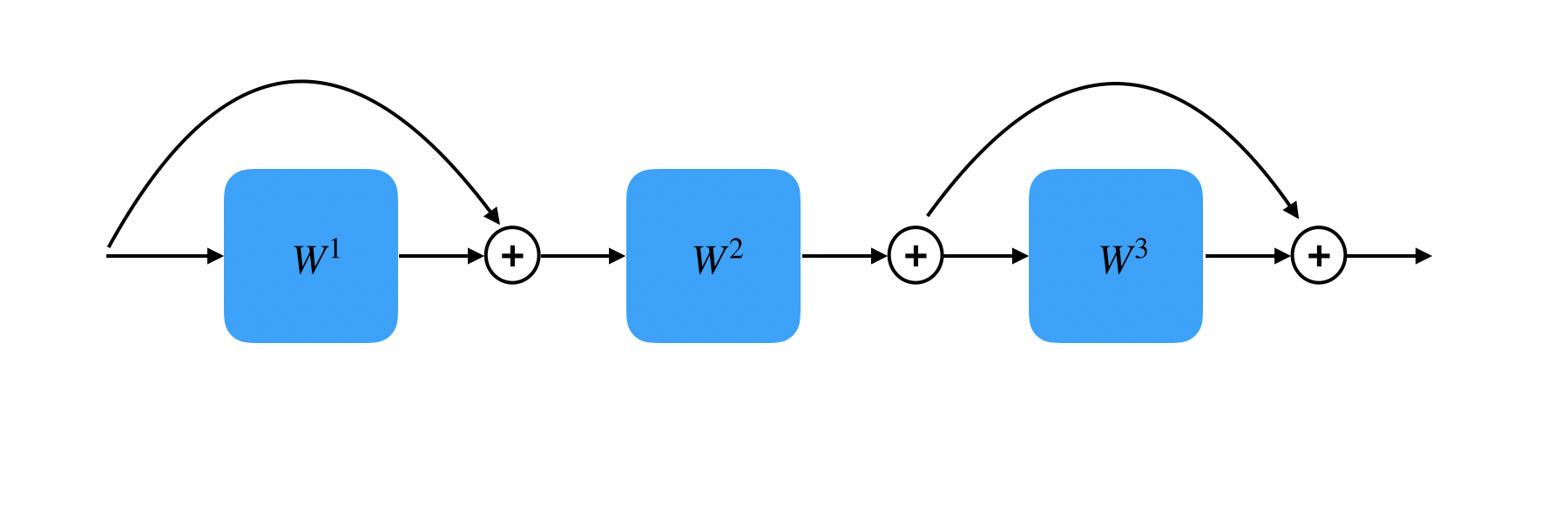}\\
(a) & (b)\\
\end{tabular}
  \caption{{\bf An illustration of Thm.~\ref{thm:eq}.} The activations of the network in {\bf (a)} are completely different from those of the network in {\bf (b)}, in which all skip connections bypassing layer $l=2$ are removed. However, the moments of the gradient norms at layer $l = 2$ are exactly the same in both {\bf (a)} and {\bf (b)}. }%
  \label{fig:dualthm}
\end{figure*}

The following theorem relates the moments of $\|J^{\bk}\|^2_2$ with those of $f_{(\bk)}(x;w)$:
\begin{restatable}{theorem}{dual}\label{thm:dual}
Let $f(x;w)$ be a ResNet/DenseNet as described in Sec.~\ref{sec:arch}. Then, we have:
\begin{enumerate}\label{eq:du}
  \item $\forall~\bk: ~\mathbb{E}_w\big[\|J^{\bk}\|^2_2 \big] = \mathbb{E}_w\big[(f_{(\bk)}(x;w))^2 \big]$.
  \item $\forall~\bk: ~\frac{\mathbb{E}_w\big[(f_{(\bk)}(x;w))^4 \big]}{3} \leq \mathbb{E}_w\big[\|J^{\bk}\|^4_2 \big] \leq \mathbb{E}_w\big[(f_{(\bk)}(x;w))^4 \big]$.
\end{enumerate}
\end{restatable}
From Eq.~\ref{ee} and Thm.~\ref{thm:dual}, we can derive bounds on the second moment of $\mathcal{G}(x,x;w)$, by observing the moments of $f_{(\bk)}(x;w)$. In addition, Thm.~\ref{thm:dual} also allows us to derive bounds on the convergence rate of $\mathcal{G}(x,x;w)$ to $\mathbb{E}_{w}[\mathcal{G}(x,x;w)] = \mathcal{K}(x,x)$, given by the ratio:
\begin{equation}\label{convergence}
\eta(n,L) := \frac{\mathbb{E}_w[\mathcal{G}(x,x;w)^2]}{\mathbb{E}_{w}[\mathcal{G}(x,x;w)]^2} 
\end{equation}
\begin{figure}[t]
    \centering
    \begin{tabular}{cc}
\includegraphics[trim=0cm 0cm 0cm 0cm, clip, width=.35\linewidth]{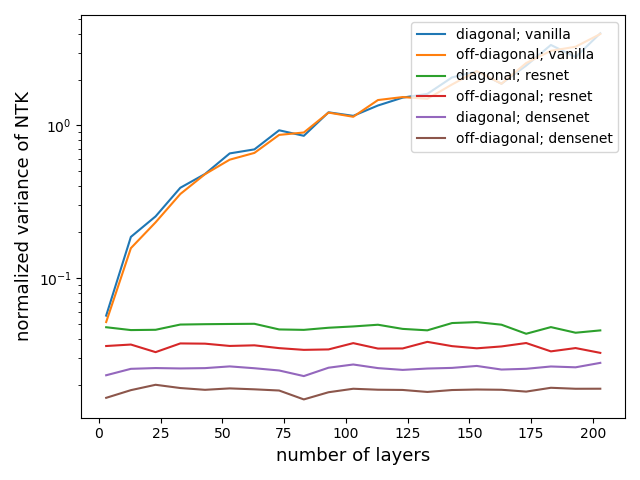}&
\includegraphics[trim=0cm 0cm 0cm 0cm, clip, width=.35\linewidth]{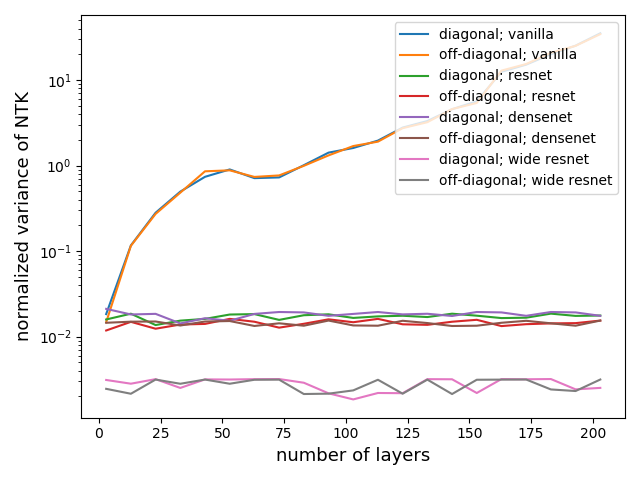}\\
(a) & (b)\\
\end{tabular}
    \caption{{\bf Normalized variance of NTK for various models.} The x-axis stands for the number of layers and the y-axis stands for the values of $\textnormal{V}(\mathcal{G}(x,x';w))$ in log-scale. The diagonal terms specify the value for $x=x'$ and the off-diagonal terms specify the value for $x \neq x'$. {\bf (a)} Results for MLP networks. {\bf (b)} Results for convolutional networks. 
    } 
    \label{fig:variance}
\end{figure}

In general, the tools developed in Thms.~\ref{thm:eq} and Thm.~\ref{thm:dual} can be used for analyzing feedforward networks of any topology. Specifically, in Thms.~\ref{thm:res_ntk_var} and~\ref{thm:dense_ntk_var}, we derive bounds on the asymptotic behavior of $\eta$ for ResNet and DenseNet architectures, with respect to both width and depth. 

\begin{restatable}{theorem}{resNTKvar}\label{thm:res_ntk_var}
Let $f(x;w)$ be a depth $L$, constant width ResNet with residual branches of depth $m$ (Eq.~\ref{constant_dense} with $n_0',n_l,n_{l,h} = n$ for all $l\in [L]$ and $h \in [m]$), with positive initialization constants $\{\alpha_l\}_{l=1}^L$. Then, there exists a constant $C>0$ such that:
\begin{multline}
\max\left[1,\frac{\sum_{\bu} \alpha_{l_u}^2}{\sum_{\bu,\bv} \alpha_{l_u}\alpha_{l_v}} \cdot \xi \right] \leq \eta(n,L) \leq \xi\textnormal{ where: }~ \xi = \exp\left[\frac{5m}{n}+\frac{C}{n}\sum_{l=1}^L \frac{\alpha_l}{1+\alpha_l} \right] \cdot \left(1 + \mathcal{O}(1/n)\right)
\end{multline}
\end{restatable}

From the result of Thm.~\ref{thm:res_ntk_var}, it is evident that the convergence rate is exponential in $\frac{m}{n}+\frac{1}{n}\sum_{l=1}^L \alpha_l$. This result supports the selection of a small $m$, as reflected in the common practice to have a small depth for the residual branches. In addition, when setting $\{\alpha_l\}_{l=1}^L$, such that, $\frac{1}{n}\sum_{l=1}^L \alpha_l$ vanishes as $n$ tends to infinity, ensures the convergence of $\eta$ to 1, regardless of depth. Note that by selecting $\{\alpha_l\}_{l=1}^L$, such that, $\sum_{l=1}^L\alpha_l \sim \mathcal{O}(1)$ is sufficient (although not necessary), and was also suggested in~\cite{fixup} as a way to train ResNets without batchnorm~\cite{10.5555/3045118.3045167}. Our results, however, reveal a much stronger implication of this initialization, as it also bounds the fluctuations of the squared Jacobian norm, implying a closer relationship with the ``kernel regime'' at the initialization of deep ResNets. From Thm.~\ref{thm:res_ntk_var}, we conclude that a proper initialization plays a crucial role in determining the asymptotic behavior of $\eta$ in deep ResNets. Surprisingly, this relationship between initialization and $\eta$ breaks down, when considering DenseNets, as illustrated in the following theorem.

\begin{restatable}{theorem}{denseNTKvar}\label{thm:dense_ntk_var}
Let $f(x;w)$ be a constant width DenseNet (Eq.~\ref{constant_dense} with $n_0',n_l = n$ for all $l\in [L]$), with initialization constant $\alpha>0$. Then, there exist constants $C_1,C_2>0$, such that: 
\begin{equation}\label{surprise}
\max\left[1,\frac{C_1}{L\log(L)^2}\cdot \xi\right] \leq \eta(n,L) \leq \xi~\textnormal{ where: }~ \xi = \exp\left[C_2/n\right]\cdot \left(1 + \mathcal{O}(1/n)\right)
\end{equation}
\end{restatable}


Surprisingly, the depth parameter $L$, as well as the initialization scale $\alpha$, are absent in the upper bound of Eq.~\ref{surprise}, revealing a depth and scale-invariant property unique to DenseNets. In other words, the convergence rate of $\eta$ to 1 is exponential in $\frac{C_2}{n}$, and does not depend on depth, or the scaling coefficient of the weights. This property represents a fundamental unique aspect of DenseNets, which might explain practical advantages observed in models incorporating dense residual connections. It is important to stress that it is impossible to replicate the guarantees presented in Thms.~\ref{thm:res_ntk_var} and~\ref{thm:dense_ntk_var} by simply normalizing the network in a different manner. That is because, the expression $\eta(n,L)$ is invariant to the scale of the weights, i.e., its value does not change when multiplying $f(x;w)$ by a constant. Therefore, maintaining a bounded normalized variance of the NTK of a $L$-depth network, comes at the cost of a different parameter complexity for each architecture. This is formulated in the following remark.  

\begin{remark} For DenseNets and ResNets (with $\alpha_l=1/L$ and $m=2$), it is possible to choose a constant width $n = \mathcal{O}(1)$ (independent of $L$) while maintaining a bounded NTK variance. In this case, the overall number of parameters in DenseNets is $\mathcal{O}(L^2)$. On the other hand, in ResNets, the overall number of parameters is $\mathcal{O}(L)$, as each one of its $L$ layers contributes a constant number of parameters $2n^2 = \mathcal{O}(1)$. However, in vanilla models, it is required that the width $n$ grow linearly with depth in order to maintain a bounded variance. Therefore, each layer contributes $\Omega(L^2)$ parameters, and the overall number of parameters is $\Omega(L^3)$. The added efficiency is the product of an inherent architectural advantage brought forth by the DenseNet architecture.
\end{remark}


\section{Experiments} 

To validate our theoretical observations, we conducted a series of experiments using the MNIST and 43 small UCI datasets (see Tab.~1 in Sec.~2 of the supplementary material for the list). Throughout the experiments, we used both fully connected architectures and convolutional architectures. For details, see Sec.~1 in the supplementary material.

\subsection{Normalized Variance of NTK} We conducted an experiment for estimating the normalized variance of the NTK, i.e.,
\begin{equation}\label{eq:normalizedvar}
\begin{aligned}
\textnormal{V}(\mathcal{G}(x,x';w)) 
=& \frac{Var\big(\mathcal{G}(x,x';w)\big)}{\mathbb{E}_w\left[\mathcal{G}(x,x';w) \right]^2}  = \frac{\mathbb{E}_w[\mathcal{G}(x,x';w)^2]}{\mathbb{E}_w\left[\mathcal{G}(x,x';w) \right]^2}-1
\end{aligned}
\end{equation}
For each model (e.g., vanilla network, ResNet, DenseNet), we fixed the width to be $n=500$, varied the number of layers and for each depth, we estimated the value in Eq.~\ref{eq:normalizedvar} for $x=x'$ and for $x\neq x'$. In order to estimate these terms, we sampled 5000 different vectors $w$ for $f(x;w)$ from a standard normal distribution and estimated $\textnormal{V}(\mathcal{G}(x,x;w))$ and $\textnormal{V}(\mathcal{G}(x,x';w))$ empirically. The inputs $x = \hat{x}/\|\hat{x}\|_2$ and $x_2 = \hat{x}'/\|\hat{x}'\|_2$ are two vectors, such that, each coordinate of $\hat{x}$ is distributed according to $\mathcal{N}(0.5,1)$ and each coordinate of $\hat{x}'$ is distributed according to $\mathcal{N}(-0.5,1)$.

In Fig.~\ref{fig:variance} we plot the normalized variance of the diagonal and off-diagonal elements of the kernel as a function of the number of layers for the various architectures. The results are plotted in log-scale. 
As can be seen, 
the diagonal and off-diagonal elements of the kernel are highly correlated for all architectures. In addition, for residual and dense architectures, the normalized variance of the NTK is relatively constant when varying the number of layers, while for vanilla networks, the normalized variance of the NTK grows exponentially.

\subsection{Kernel Regression over Random Gradient Features}\label{sec:regression}

We conducted various experiments to compare the ability of the gradients $\nabla_w f(x;w)$ of each architecture to serve as random features for kernel regression. The process is as follows: for a given network $f(x;w)$ we sampled $w_1,\dots,w_T$ at random from a standard normal distribution and used $\nabla_{w_i} f(x;w_i)$ as our random features.
In addition, the labels are being cast into one-hot vectors corresponding to their discrete values in $[k]$. 
To solve the kernel regression task, we employed the closed form solution:
\begin{equation}
g(x;w) := (\mathcal{G}_T(x,x_1),\dots,\mathcal{G}_T(x,x_m)) \cdot  H_T^{-1} \cdot Y
\end{equation}
where $\mathcal{G}_T(x,x') = \frac{1}{T}\sum^{T}_{i=1} \mathcal{G}(x,x';w_i)$, $H_k = (\mathcal{G}_T(x_i,x_j))_{i,j \in [m]} \in \mathbb{R}^{m \times m}$ and $Y \in \mathbb{R}^{m \times k}$ is a matrix whose $i$'th row is $y_i$. 


\paragraph{Experiments on MNIST} 

In this set of experiments, training was done over $2000$ MNIST training samples, where each train/test sample is normalized to have norm 1. The reported results are the average accuracy rates over 20 samples of $w$ and the error bars are the corresponding standard deviations.  
In Fig.~\ref{fig:var_depth}(a-c) we report the expected accuracy rates of $g(x;w)$ on the test set, when varying the number of layers of $f(x;w)$, while fixing the width to be $n \in \{50,100,500\}$ and $T=1$. The performances of the infinite width limit kernels of vanilla networks, ResNets and DenseNets are plotted as well, under the names, `vanilla kernel', `resnet kernel' and `densenet kernel' respectively. In Fig.~\ref{fig:var_depth}(d-f) we report the same results, when he width is $n \in \{2,50,100\}$ and $T=30$. As can be seen, when fixing the width of the network, increasing the depth of a vanilla network is adverse to the performance of the kernel regression. However, this is not the case of ResNets and DenseNets. In addition, the results of performing kernel regression with the NTKs are comparable to the results of their corresponding infinite width limit kernels.

In Fig.~\ref{fig:var_width}, we report the effect of varying the width when fixing the depth and $T=1$. As can be seen, the performance of a standard network is significantly inferior to the performances of the kernel regressions corresponding to ResNets and DenseNets when the number of layers is larger then 4. Even though, the performance of each architecture improves when increasing the width, standard neural networks are required to be much wider, in order to achieve the same degree of success as ResNets and DenseNets. 


\paragraph{Experiments on UCI} We also compared the performance of kernel regression over 43 small UCI datasets (see list in Tab.~1 in the supplementary material). We note that the performance of the various methods vary from one dataset to another as a result of dataset complexity, number of classes, etc'. Therefore, in order to average the results over the various datasets, instead of reporting the absolute accuracy rates, we report the relative accuracy rates with respect to the accuracy rate of a three layered network (i.e., the accuracy rate divided by the accuracy rate obtained with three layers). 
For each fully connected architecture, we compared the relative accuracy rates for widths 10, 100, 500, when varying the number of layers. The relative accuracy rates are averaged over the 43 datasets. In addition, the accuracy rate on each dataset is averaged for 20 samples of $w$. 
{The results in Fig.~\ref{fig:uci} show that the performance of kernel regression for ResNet and DenseNet architectures do not degrade as a result of increasing the number of layers. In fact, the results improve when increasing the number of layers for DenseNets and DenseNets of widths 100 (about 4-5\% improvement). In contrast, for vanilla networks, increasing the number of layers harms the performance. It is evident that when increasing the width of the vanilla network, the kernel regression performance becomes more stable but still degrades when increasing the number of layers.}
Since Fig.~\ref{fig:uci} does not compare the performance of the various architectures, rather it compares its stability, for completeness, in Tab.~1 in the supplementary material we report the absolute accuracy rates of the various architectures with three layers and widths 10, 100 and 500. As can be seen, the different models achieve very similar results on all dataset.

\begin{figure}[t]
    \centering
    \begin{tabular}{ccc}
\includegraphics[trim=0cm 0cm 0cm 0cm, clip, width=.25\linewidth]{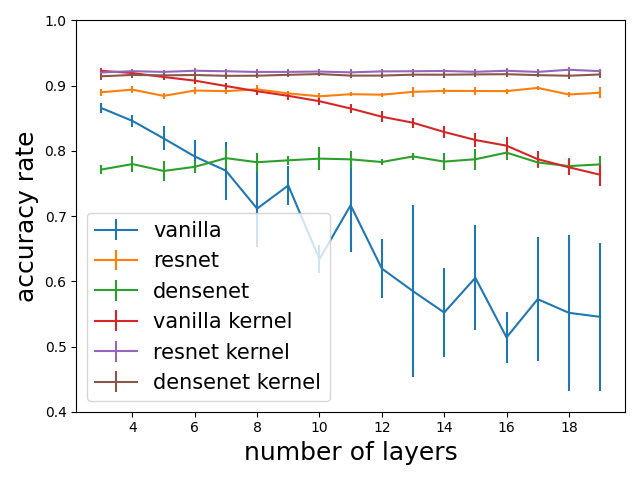}&
\includegraphics[trim=0cm 0cm 0cm 0cm, clip, width=.25\linewidth]{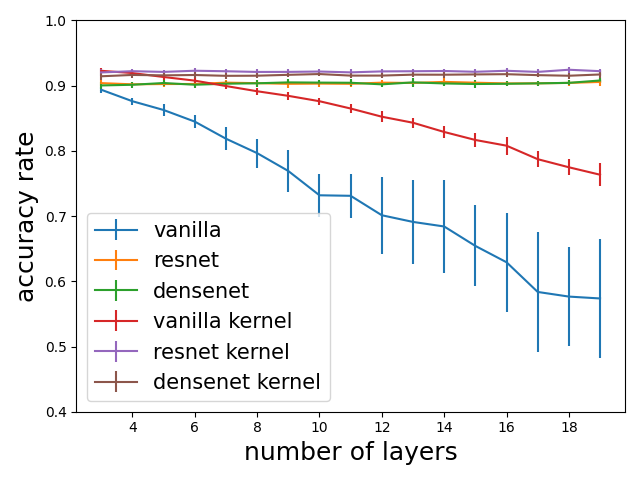}&
\includegraphics[trim=0cm 0cm 0cm 0cm, clip, width=.25\linewidth]{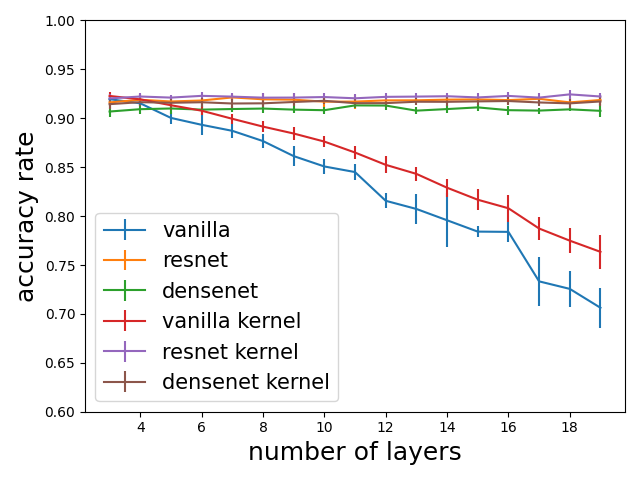}\\
(a) & (b) & (c)\\
\includegraphics[trim=0cm 0cm 0cm 0cm, clip, width=.25\linewidth]{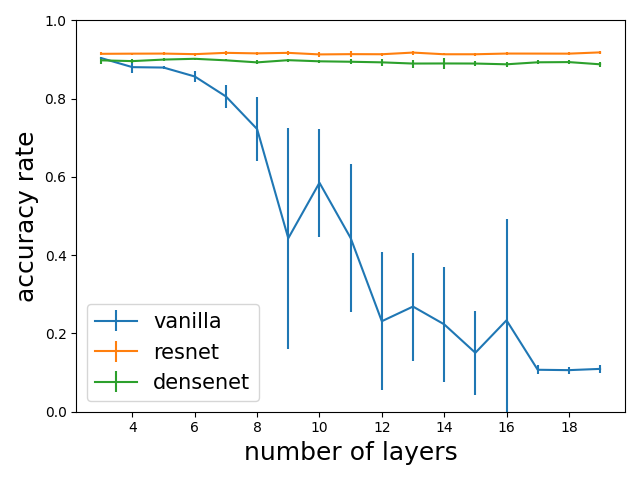}&
\includegraphics[trim=0cm 0cm 0cm 0cm, clip, width=.25\linewidth]{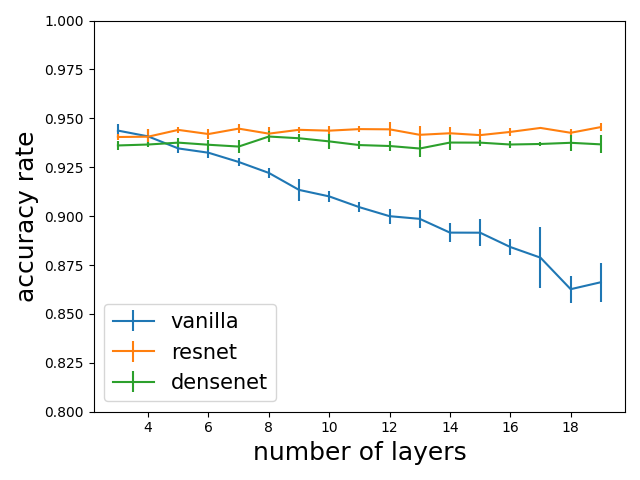}&
\includegraphics[trim=0cm 0cm 0cm 0cm, clip, width=.25\linewidth]{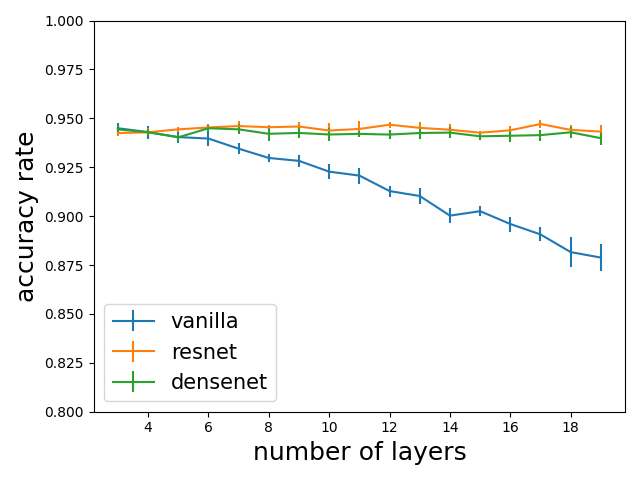}\\
(d) & (e) & (f)\\
\end{tabular}
    \caption{ {\bf Results on MNIST for kernel regression over random gradient features.}  Plotted are the averaged accuracy rates, when varying the number of layers. In the first row, $T=1$ and the width of $f(x;w)$ is either {\bf (a)} $50$ {\bf (b)} $100$ or {\bf (c)} $500$. `vanilla kernel', `Resnet kernel' and `densenet kernel' stand for the results of the infinite width limit kernels of vanilla networks, ResNets and DenseNets (resp.). In the second row, $T=30$ and the width of $f(x;w)$ is either {\bf (d)} $2$ {\bf (e)} $10$ or {\bf (f)} $100$.} 
    \label{fig:var_depth}
\end{figure}


\begin{figure}[t]
    \centering
    \begin{tabular}{@{}c@{~~}c@{~~}c@{}}
\includegraphics[trim=0cm 0cm 0cm 0cm, clip, width=.25\linewidth]{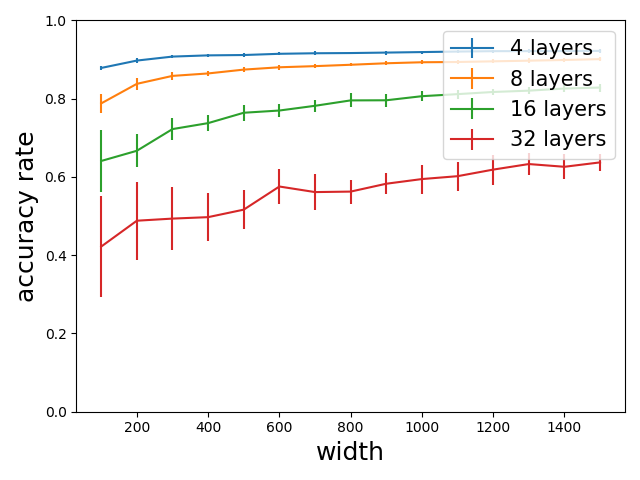}&
\includegraphics[trim=0cm 0cm 0cm 0cm, clip, width=.25\linewidth]{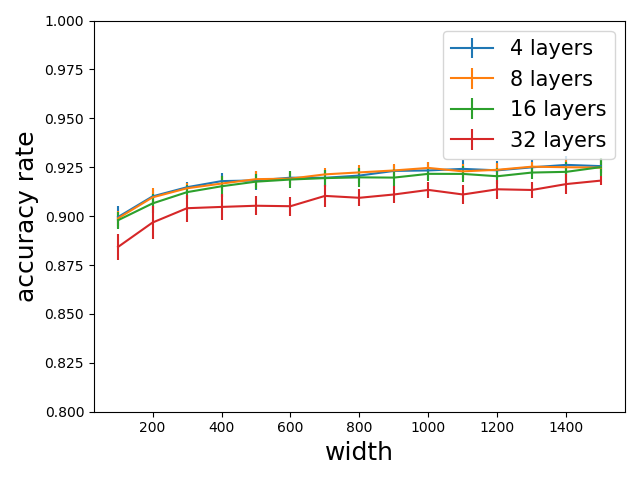}&
\includegraphics[trim=0cm 0cm 0cm 0cm, clip, width=.25\linewidth]{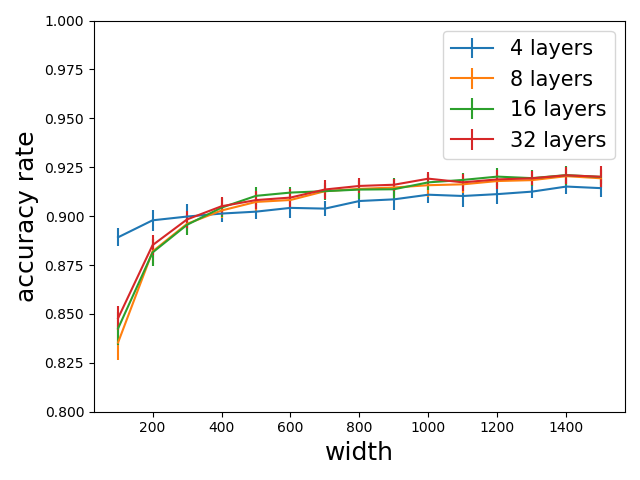}\\
(a) & (b) & (c)\\
\end{tabular}
    \caption{ {\bf Results on MNIST for kernel regression over random gradient features.} Plotted are the averaged accuracy rates, when varying the width. {\bf (a)} Results of vanilla networks, {\bf (b)} Results of ResNets, {\bf (c)} Results of DenseNets.} 
    \label{fig:var_width}
\smallskip
    \centering
    \begin{tabular}{@{}c@{~~}c@{~~}c@{}}
\includegraphics[trim=0cm 0cm 0cm 0cm, clip, width=.25\linewidth]{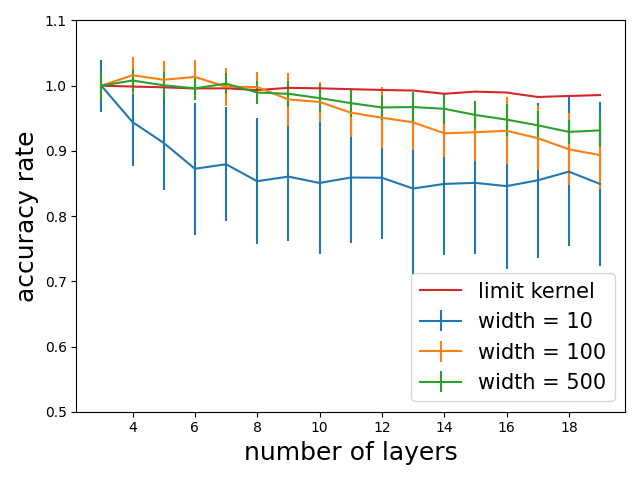}&
\includegraphics[trim=0cm 0cm 0cm 0cm, clip, width=.25\linewidth]{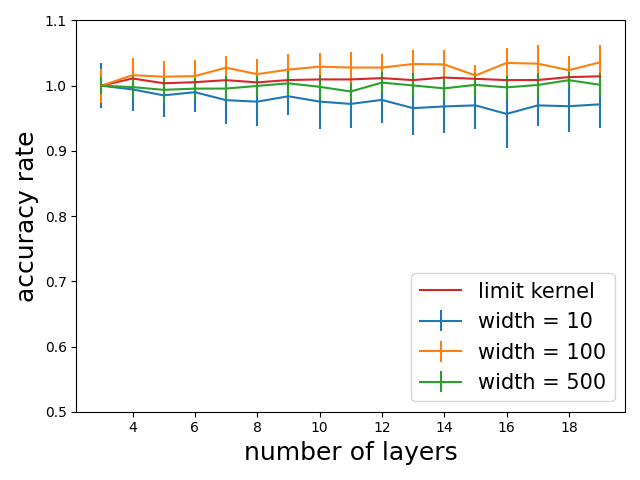}&
\includegraphics[trim=0cm 0cm 0cm 0cm, clip, width=.25\linewidth]{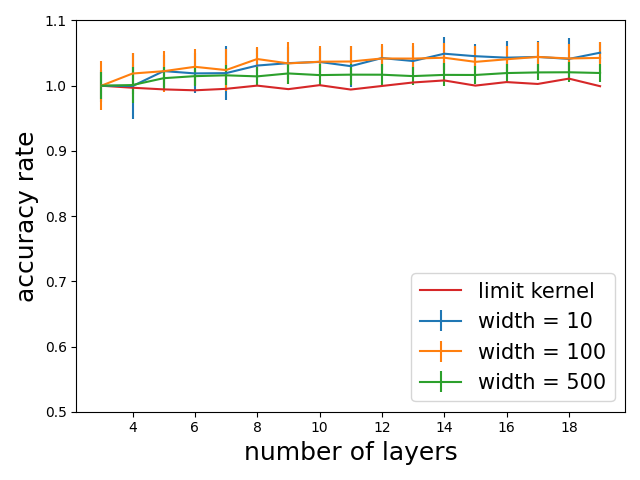}\\
(a) & (b) & (c)\\
\end{tabular}
    \caption{{ {\bf Results on UCI for kernel regression over random gradient features.} Plotted are the averaged relative accuracy rates as a function of the number of layers. {\bf (a)} Results of vanilla networks, {\bf (b)} Results of ResNets and {\bf (c)} Results of DenseNets.} }
    \label{fig:uci}
\end{figure}


\section{Related Work}

The study of infinitely wide neural networks has been in the forefront of theoretical deep learning research in the last few years. A number of papers~\cite{CNTK,NIPS2019_9063,exact} have followed up on the original NTK work~\cite{NTK}. 
An extention of the GP and NTK results is given in~\cite{tensor}, where it is shown that neural networks of any architecture (including weight-tied ResNets, DenseNets, or RNNs) converge to GPs in the infinite width limit, and prove the existence of the infinite width NTKs. In~\cite{NIPS2019_9063}, corrections to the NTK are derived to bound the change of the NTK during training, which applies for both the diagonal and off-diagonal entries of the NTK. However, depth is treated as a constant, and therefore their result only apply for shallow networks.
An interesting problem is to quantify the convergence rate of the NTK to its limit. Feynman diagrams were used to provide finite width corrections to the NTK~\cite{Dyer2020Asymptotics}.  However, the analysis relies on a conjecture, and does not hold for residual architectures. What is most related to our results are the finite width corrections to the NTK for vanilla networks, introduced in~\cite{finite_ntk}. These results depend on the depth of the network. However, their analysis does not apply to residual architectures. In contrast, in our Thms.~\ref{thm:eq} and~\ref{thm:dual}, we establish a duality that exists between forward and backward statistics, which allows considering only forward statistics, and can be readily applied for most fully connected architectures, with arbitrary topologies.
In~\cite{boris1} they tackle two failure modes that are caused in finite size networks by exponential explosion or decay of the norm of intermediate layers. It is shown that for random fully connected vanilla ReLU networks, the variance of the squared norm of the activations exponentially increases, even when initializing with the $\frac{2}{fan-in}$ initialization. For ResNets, this failure mode can be overcome by correctly rescaling the residual branches. However, it is not clear how such a rescaling affects the back propagation of gradients.


\section{Conclusions}
The Neural Tangent Kernel has provided new insights into the training dynamics of wide neural networks, as well as their generalization properties, by linking them to kernel methods. 
In this work, using a duality principle between forward and backward norm propagation, we have derived finite width and depth corrections for ResNet and DenseNet architectures, and have shown convergence properties of deep residual models that are absent in the vanilla fully connected architectures. Our results shed new light on the effect of residual connections on the training dynamics of practically sized networks, suggesting that that models incorporating residual connections operate much closer to the ``kernel regime'' approximation than vanilla architectures, even at large depths.

\clearpage


\bibliography{refs}
\bibliographystyle{plain}


\newpage
\section*{Appendix}
\section{Architectures}\label{app:arch}

\paragraph{Fully connected networks} Each fully connected architecture consist of $L$ layers, where the first layer is a standard fully connected layer from input dimension to $n$, followed by $L-2$ hidden layers of width $n$ and ends with a standard fully connected layer with a single output. Each hidden layer is a standard fully connected layer, a fully connected residual block or a fully connected dense layer, depending on the architecture at hand. Throughout the experiments, each residual block is of depth $2$. 

\paragraph{Convolutional networks} For the convolutional architectures, instead of fully connected layers, we used convolutional layers with a kernel size 3, stride 1 and padding 1. The number of channels of each layer is treated as its width. In all architectures, the first layer is a convolutional layer with three input channels and $n$ output channels. For the vanilla network, each hidden layer consists of a convolutional layer with $n$ input and output channels. For the residual network architecture, each residual block consists of two convolutional layers with $n$ input and output channels. For the DenseNet architecture, the $(i+1)$'th layer consists of a convolutional layer with $i \cdot n$ input channels and $n$ output channels. The input to this layer is the concatenation of the previous $i$ hidden layers along the channels dimension. The wide ResNet architecture was taken from the official pyTorch implementation of~\cite{BMVC2016_87}. Each residual block is of kernel size $n$ and the base width is $n$ as well. The last layer in all architectures is a fully connected layer that returns a single output.

Each convolutional layer is followed by a $\frac{1}{\sqrt{9 in\_c}}$ normalization and a ReLU activation, where, $in\_c$ is the number of input channels of the corresponding layer. For instance, in the DenseNet architecture, the $(i+1)$'th layer has $in\_c = i \cdot n$ input channels. 

Unless mentioned otherwise, for the ResNet architectures we used scaling coefficients $\alpha_1=\dots=\alpha_L = 0.1/L$ and for the DenseNet architectures we used $\alpha=0.5$.

\section{Additional Experiments}\label{app:exp}

 \begin{figure*}[t]
\centering
\begin{tabular}{cc}

\includegraphics[trim=0.5cm 0.05cm -0.5cm 0.0cm,width=.47 \linewidth]{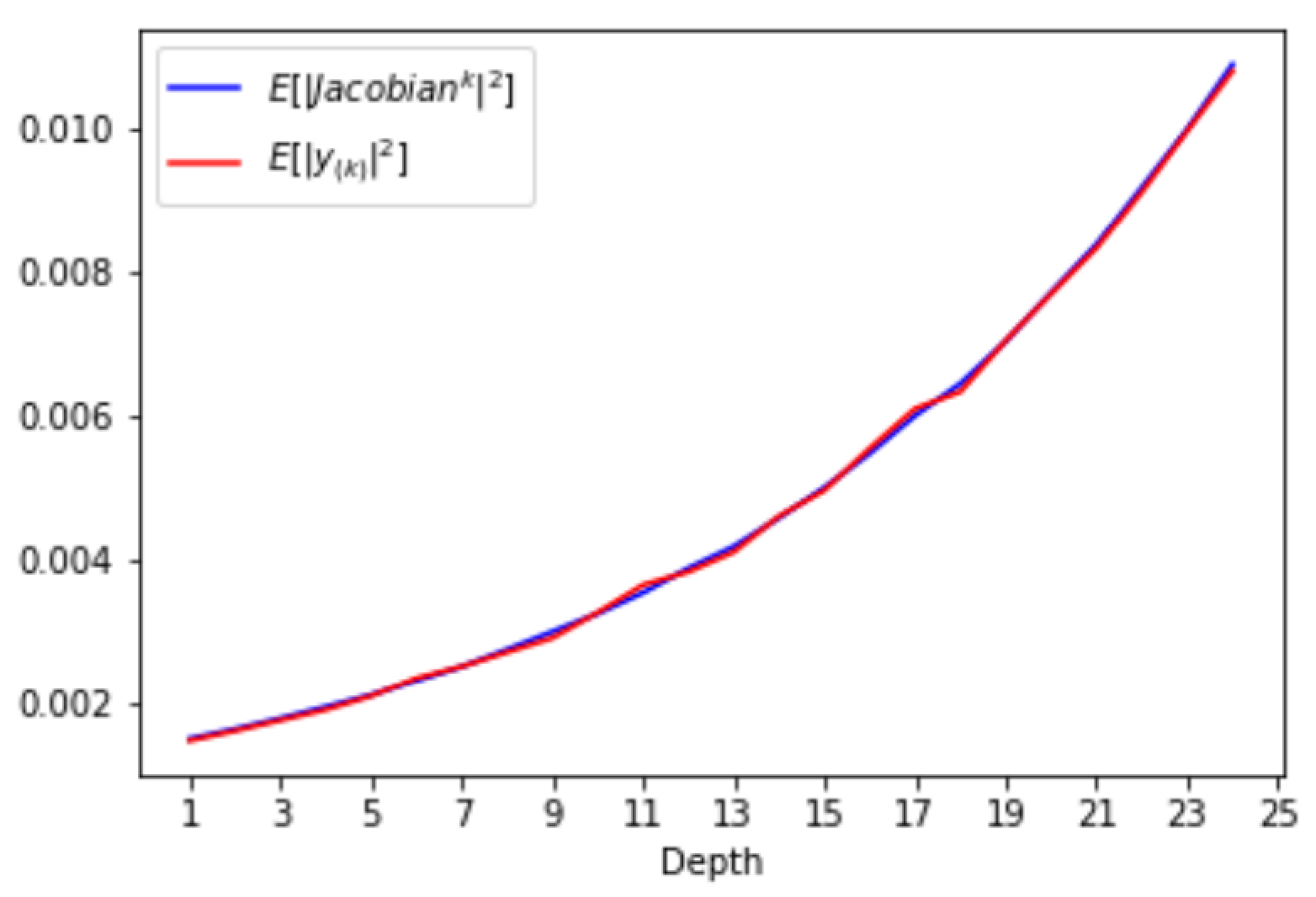}&
\includegraphics[trim=0.0cm 0.3cm -0.22cm 0.0cm,width=.47 \linewidth]{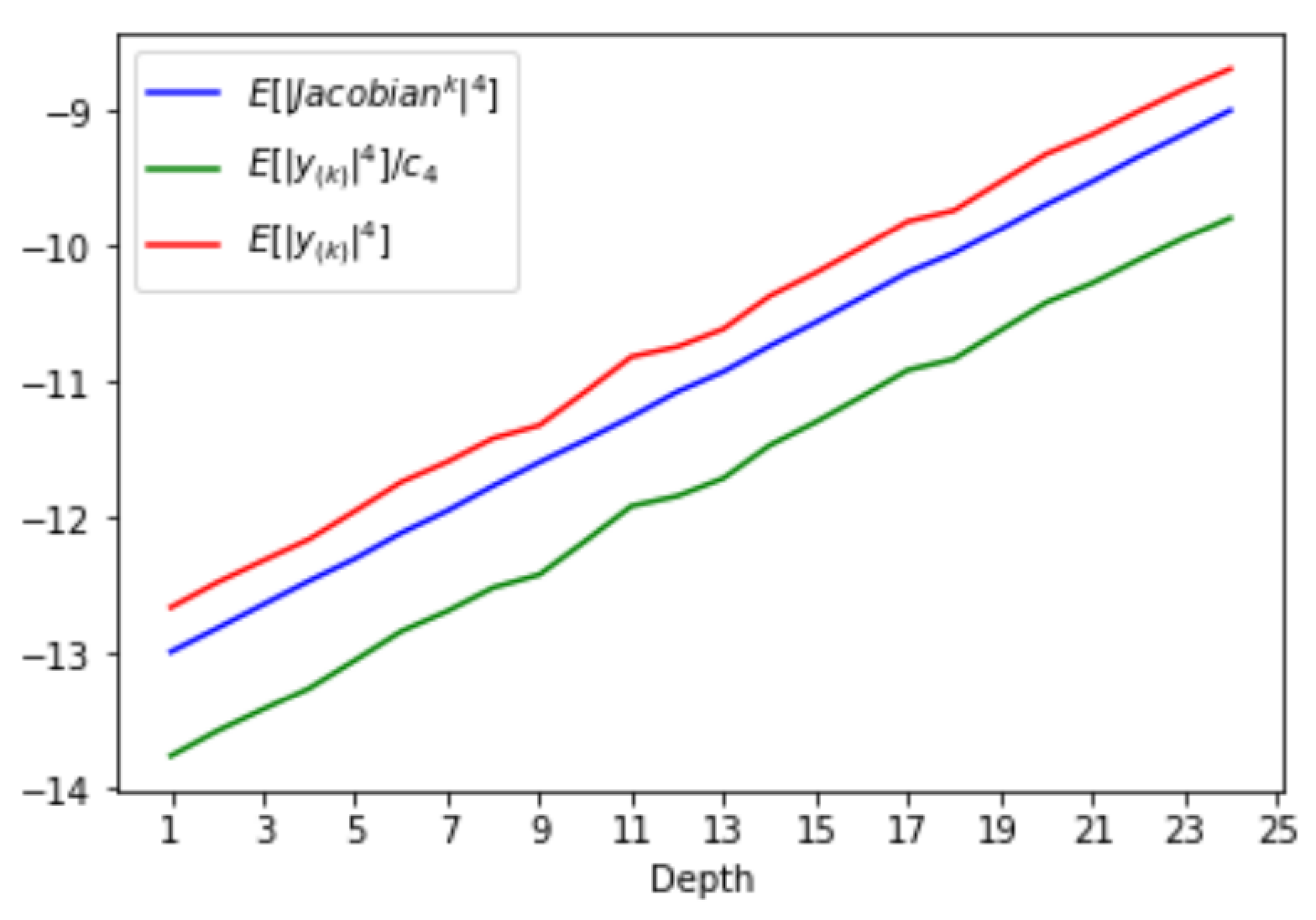}\\
(a) & (b)\\
\end{tabular}
  \caption{ The second (a) and fourth (b) moments, in log scale, of the per layer Jacobian norm $\|J^{\bk}\|_2$ and the squared norm of the output of the corresponding reduced architecture $\|f_{(\bk)}(x;w)\|_2$.   } 
  \label{fig:blah}
\end{figure*}

\noindent{\bf Validating Thm.~\ref{thm:dual}\quad} We conducted an experiment for validating Thm.~\ref{thm:dual}. For this purpose, we estimated the second and fourth moments of the per-layer Jacobian $\|J^{\bk}\|_2$ and the squared norm of the output of the corresponding reduced architecture $\|f_{(\bk)}(x;w)\|_2$ for ResNet architectures (with $m=2$, $\alpha_l=0.3$) with varying number of layers. The results were obtained from the simulated results of 200 independent runs per depth, where the value for $k$ is random for each depth. All networks were initialized using normal distributions. As can be seen in Fig.~\ref{fig:blah}, the mean of both $\|J^{\bk}\|^2_2$ and $\|f_{(\bk)}(x;w)\|^2_2$ closely match, while the fourth moment $\mathbb{E}[\|J^{\bk}\|^4_2]$ is upper and lower bounded by the corresponding moments of the output, as predicted in Thm.~\ref{thm:dual}.

\noindent{\bf Absolute accuracy rates on UCI datasets\quad} In Tab.~\ref{tab:uci}, we report the absolute accuracy rates of kernel regression over random gradient features extracted from the fully connected architectures (e.g., vanilla ReLU networks, ResNets and DenseNets) with three layers and widths 10, 100 and 500 and of kernel regression over the width limit kernel. As can be seen, the various models achieve comparable results on all dataset.

\begin{table}[t]%
\label{tab:uci}
\begin{center}
\begin{tabular}{@{}c@{~~}c@{~~}c@{~~}c@{~~}c@{~~}c@{~~}c@{~~}c@{~~}c@{~~}c@{~~}c@{~~}c@{~~}c@{}}
\hline
\multicolumn{1}{l}{\multirow{2}{*}{}} & \multicolumn{4}{c}{Vanilla network} & \multicolumn{4}{c}{ResNet} & \multicolumn{4}{c}{DenseNet} \\ 
\cmidrule(lr){2-5} 
\cmidrule(lr){6-9} 
\cmidrule(lr){10-13} 
\multicolumn{1}{c}{Dataset} 
& 10  & 100 & 500 & limit & 10 & 100 & 500 & limit & 10 & 100 & 500 & limit    \\ 
\cline{1-13} 
\multicolumn{1}{c}{Abalone} 
& 0.51 & 0.54 & 0.53 &0.54& 0.60 & 0.55 & 0.55 &0.56& 0.51 & 0.52 & 0.53 & 0.54\\
\multicolumn{1}{c} {Adult}  
& 0.74 & 0.73 & 0.76 &0.78& 0.73 & 0.72 & 0.76 &0.76& 0.76 & 0.75 & 0.76 & 0.74 \\
\multicolumn{1}{c}{Bank} 
& 0.86 & 0.85 & 0.87 &0.88& 0.87 & 0.84 & 0.87 &0.88& 0.86 & 0.85 & 0.88 & 0.87  \\
\multicolumn{1}{c}{Car}  
& 0.75 & 0.81 & 0.88 &0.90& 0.76 & 0.81 & 0.87 &0.89& 0.74 & 0.83 & 0.88 & 0.89 \\
\multicolumn{1}{c}{Cardiotocography\_10clases} 
& 0.67 & 0.73 & 0.77 &0.80& 0.70 & 0.73 & 0.78 &0.78& 0.68 & 0.72 & 0.77 & 0.78 \\
\multicolumn{1}{c} {Chess\_krvk} 
& 0.25 & 0.28 & 0.36 &0.39& 0.28 & 0.29 & 0.35 &0.38& 0.28 & 0.31 & 0.36 & 0.38 \\
\multicolumn{1}{c}{Chess\_krvkp} 
& 0.85 & 0.96 & 0.97 &0.97& 0.89 & 0.96 & 0.97 &0.97& 0.87 & 0.95 & 0.97 & 0.97 \\
\multicolumn{1}{c}{Connect 4} 
& 0.66 & 0.68 & 0.73 &0.74& 0.67 & 0.68 & 0.72 &0.74& 0.68 & 0.68 & 0.73 & 0.74 \\
\multicolumn{1}{c}{Contrac} 
& 0.46 & 0.44 & 0.48 &0.48& 0.47 & 0.44 & 0.49 &0.50& 0.48 & 0.43 & 0.49 & 0.50 \\
\multicolumn{1}{c} {Hill-Valley} 
& 0.52 & 0.57 & 0.59 &0.61& 0.53 & 0.57 & 0.60 &0.57& 0.47 & 0.57 & 0.63 & 0.57 \\
\multicolumn{1}{c}{Image-Segmentation} 
& 0.69 & 0.75 & 0.75 &0.75& 0.72 & 0.75 & 0.75 &0.75& 0.70 & 0.75 & 0.75 & 0.75 \\
\multicolumn{1}{c}{Led-Display} 
& 0.65 & 0.68 & 0.68 &0.67& 0.65 & 0.68 & 0.69 &0.65 & 0.65 & 0.68 & 0.68 & 0.60 \\
\multicolumn{1}{c}{Letter} 
& 0.56 & 0.74 & 0.79 &0.80& 0.61 & 0.74 & 0.80 & 0.81& 0.56 & 0.73 & 0.79 & 0.81 \\
\multicolumn{1}{c}{Magic} 
& 0.78 & 0.72 & 0.80 &0.81& 0.81 & 0.73 & 0.78 &0.82& 0.80 & 0.75 & 0.79 & 0.82 \\
\multicolumn{1}{c}{Molec-biol-splice} 
& 0.56 & 0.74 & 0.79 &0.80& 0.62 & 0.73 & 0.78 &0.77& 0.79& 0.61 & 0.68 &  0.78 \\
\multicolumn{1}{c}{Mushroom} 
& 0.98 & 0.99 & 0.99 &0.99& 0.99 & 0.99 & 0.99 &1.0& 0.98 & 1.0 & 1.0 & 1.0 \\
\multicolumn{1}{c}{Nursery} 
& 0.79 & 0.87 & 0.92 &0.92& 0.84 & 0.86 & 0.91 &0.93& 0.80 & 0.88 & 0.92 & 0.93 \\
\multicolumn{1}{c}{Oocytes\_merluccius\_nucleus\_4d} 
& 0.75 & 0.74 & 0.77 &0.79& 0.78 & 0.75 & 0.77 &0.77& 0.75 & 0.72 & 0.76 & 0.77 \\
\multicolumn{1}{c}{Oocytes\_merluccius\_states\_2f}
& 0.87 & 0.90 & 0.92 &0.92& 0.89 & 0.90 & 0.92 &0.91& 0.88 & 0.90 & 0.91 &  0.91 \\
\multicolumn{1}{c}{Optical} 
& 0.84 & 0.97 & 0.98 &0.98& 0.91 & 0.97 & 0.98 &0.98& 0.87 & 0.98 & 0.97 & 0.98\\
\multicolumn{1}{c}{Ozone} 
& 0.94 & 0.96 & 0.96 &0.96& 0.96 & 0.96 & 0.97 &0.97& 0.96 & 0.96 & 0.96 & 0.97 \\
\multicolumn{1}{c}{Page\_blocks} 
& 0.90 & 0.95 & 0.95 &0.96& 0.95 & 0.95 & 0.96 &0.96& 0.95 & 0.94 & 0.95 & 0.96 \\
\multicolumn{1}{c}{Pendigits} 
& 0.92 & 0.98 & 0.98 &0.99& 0.95 & 0.98 & 0.99 &0.99& 0.93 & 0.98 & 0.99 & 0.99 \\
\multicolumn{1}{c}{Plants\_margin} 
& 0.46 & 0.68 & 0.76 &0.77& 0.54 & 0.70 & 0.80 &0.77& 0.48 & 0.63 & 0.74  & 0.77 \\
\multicolumn{1}{c}{Plants\_texture} 
& 0.57 & 0.75 & 0.79 &0.80& 0.63 & 0.76 & 0.80 &0.79& 0.56 & 0.73 & 0.77& 0.80 \\
\multicolumn{1}{c}{Plants\_shape} 
& 0.42 & 0.47 & 0.52 &0.55& 0.42 & 0.50 & 0.54 &0.53& 0.38 & 0.44 & 0.49 & 0.53 \\
\multicolumn{1}{c}{Ringnorm} 
& 0.66 & 0.66 & 0.71 &0.72& 0.69 & 0.64 & 0.70 &0.73& 0.67 & 0.67 & 0.71 & 0.72 \\
\multicolumn{1}{c}{Semeion} 
& 0.70 & 0.90 & 0.93 &0.93& 0.80 & 0.92 & 0.93 &0.93& 0.72 & 0.81 & 0.92& 0.92 \\
\multicolumn{1}{c}{Spambase} 
& 0.82 & 0.89 & 0.91 &0.92& 0.85 & 0.90 & 0.91 &0.88& 0.86 & 0.89 & 0.90& 0.88 \\
\multicolumn{1}{c}{Statlog\_german\_credit} 
& 0.61 & 0.71 & 0.73 &0.74& 0.65 & 0.70 & 0.72 &0.74& 0.64 & 0.71 & 0.73 & 0.74 \\
\multicolumn{1}{c}{Statlog\_image} 
& 0.90 & 0.93 & 0.95 &0.95& 0.91 & 0.93 & 0.95 &0.95& 0.90 & 0.93 & 0.95 & 0.95 \\
\multicolumn{1}{c}{Statlog\_landsat} 
& 0.82 & 0.84 & 0.86 &0.87& 0.83 & 0.85 & 0.86 &0.87& 0.83 & 0.84 & 0.86 & 0.86 \\
\multicolumn{1}{c}{Statlog\_shuttle} 
& 0.97 & 0.98 & 0.98 &0.98& 0.97 & 0.98 & 0.98 &0.99& 0.98 & 0.98 & 0.98 & 0.98 \\
\multicolumn{1}{c}{Steel\_plates}  
& 0.66 & 0.70 & 0.74 &0.75& 0.70 & 0.70 & 0.73 &0.74& 0.66 & 0.70 & 0.73 & 0.75 \\
\multicolumn{1}{c}{Thyroid}  
& 0.92 & 0.93 & 0.95 &0.95& 0.94 & 0.94 & 0.95 &0.95& 0.94 & 0.94 & 0.95 & 0.95 \\
\multicolumn{1}{c}{Titanic}  
& 0.54 & 0.57 & 0.58 &0.70& 0.60 & 0.55 & 0.61 &0.54& 0.71 & 0.67 & 0.53 & 0.47 \\
\multicolumn{1}{c}{Twonorm}  
& 0.90 & 0.95 & 0.96 &0.96& 0.93 & 0.93 & 0.96 &0.96& 0.90 & 0.95 & 0.96 &  0.97 \\
\multicolumn{1}{c}{Waveform}  
& 0.75 & 0.74 & 0.80 &0.81& 0.77 & 0.73 & 0.80 & 0.81 & 0.74 & 0.75 & 0.81 & 0.81 \\
\multicolumn{1}{c}{Wall\_following}  
& 0.71 & 0.79 & 0.83 &0.84& 0.73 & 0.80 & 0.83 &0.83& 0.71 & 0.78 & 0.82 &  0.80 \\
\multicolumn{1}{c}{Waveform\_Noise}  
& 0.62 & 0.77 & 0.81 &0.81& 0.70 & 0.75 & 0.80 &0.82& 0.67 & 0.74 & 0.81  & 0.83 \\
\multicolumn{1}{c}{Wine\_quality\_red}  
& 0.53 & 0.55 & 0.58 &0.59& 0.54 & 0.54 & 0.58 &0.60& 0.54 & 0.56 & 0.59 & 0.60 \\
\multicolumn{1}{c}{Wine\_quality\_white}  
& 0.48 & 0.47 & 0.51 &0.52& 0.48 & 0.46 & 0.50 &0.52& 0.48 & 0.48 & 0.51 &  0.52 \\
\multicolumn{1}{c}{Yeast}  
& 0.49 & 0.45 & 0.50 &0.50& 0.51 & 0.45 & 0.49 &0.50& 0.50 & 0.46 & 0.50  & 0.51 \\
\hline
\end{tabular}
\end{center}
\caption{Results of kernel regression over random gradient features on UCI for architectures with 3 layers and widths 10, 100 and 500. The results are compared with the performance of the width limit kernels associated with each architecture.}
\end{table}%

\section{Useful Lemmas}\label{app:useful}

\begin{lemma}\label{lem:zeromeasure}
Let $f(x;w)$ be a neural network (e.g., vanilla ReLU, ResNet, DenseNet) with $N$ parameters. Let $g(x;w)$ be a pre-activation neuron within $f(x;w)$. Let $x\neq 0$ be an arbitrary input. Then, the set $\{w \mid g(x;w) = 0\}$ is of measure zero. 
\end{lemma}

\begin{proof}
We prove the claim by induction on the depth of $g(x;w)$. We denote by $v\in \mathbb{R}^{N_1}$ the subset of $w$ of weights involved in the computation of $g(x;w)$ and by $u\in \mathbb{R}^{N_2}$ the rest of the weights. For simplicity, we will denote $g(x;v) := g(x;w)$.

{\bf Base case:} Assume $g(x;w)$ is a neuron in the first hidden layer of $f(x;w)$. Then, $g(x;w) = \langle v,x \rangle$, where $v$ is a vector of weights, subset to $w$. We notice that since $x\neq 0$, the zero set $\{w \mid g(x;w) = 0\} = \{u \mid \langle u,x\rangle = 0\} \times \mathbb{R}^{N_2}$ is of dimension $N-1$. Therefore, $\{w \mid g(x;w) = 0\}$ is of measure zero. 

{\bf Induction hypothesis:} Assume that for any neuron $g(x;w)$ in the $k$'th layer, the set $\{w \mid g(x;w) = 0\}$ is of measure 0.

{\bf Induction step:} Let neuron $g(x;w)$ in the $(k+1)$'th layer. Then, we have:
\begin{equation}
g(x;w) = \langle \hat{v}, \hat{g}(x;v \setminus \hat{v})\rangle 
\end{equation}
where $\hat{v}$ are the weights of the specific neuron $g(x;w)$, $\hat{g}(x;v \setminus \hat{v})$ is a concatenation of the neurons that serve as inputs to $g(x;w)$ in the network $f(x;w)$ and $v \setminus \hat{v}$ denotes the set of weights involved in the computation of these neurons. 

Let $\hat{g}_1(x;v \setminus \hat{v})$ be the first coordinate of $\hat{g}(x;v\setminus \hat{v})$.
\begin{equation}
\begin{aligned}
\{v \mid g(x;v) = 0\} \subset  &
\{v \mid \hat{g}_1(x;v \setminus \hat{v}) \neq 0, g(x;v) = 0\} \cup \{v \mid \hat{g}_1(x;v \setminus \hat{v}) = 0, g(x;v) = 0 \} \\
\subset  &
\{v \mid \hat{g}_1(x;v \setminus \hat{v}) \neq 0, g(x;v) = 0\} \cup \mathbb{R} \times \{v\setminus \hat{v}_1 \mid \hat{g}_1(x;v \setminus \hat{v}) = 0 \}
\end{aligned}
\end{equation}
We would like to prove that each set in this union is of measure zero. This will conclude the proof, since a union of measure zero sets is measure zero as well. We note that by the induction hypothesis, the set  $\{v\setminus \hat{v}_1 \mid \hat{g}_1(x;v \setminus \hat{v}) = 0\}$ is of measure zero. In particular, 
$\mathbb{R} \times \{v\setminus \hat{v}_1 \mid \hat{g}_1(x;v \setminus \hat{v}) = 0 \}$ is of measure zero. On the other hand, for any $v \setminus \hat{v}$, such that, $\hat{g}_1(x;v \setminus \hat{v}) \neq 0$, we have:
\begin{equation}\label{eq:v}
\hat{v}_1 = -\frac{\sum^{k}_{i=2} \hat{v}_i \cdot \hat{g}_i(x;v\setminus \hat{v})}{\hat{g}_1(x;v \setminus \hat{v})}  
\end{equation}
where $k$ is the dimension of $\hat{g}(x;v \setminus \hat{v})$. We notice that since the left hand side of Eq.~\ref{eq:v} is a continuous function, the set $\{v \mid \hat{g}_1(x;v \setminus \hat{v}) \neq 0, g(x;v) = 0\}$ can be represented as a graph of a continuous function, where $\hat{v}_1$ satisfies Eq.~\ref{eq:v}. Therefore, it is of measure zero. Hence, $\{w \mid g(x;w) = 0\}$ is of measure zero as well.
\end{proof}

\begin{lemma}\label{lem:measureone} Let $f(x;w)$ be a neural network (e.g., vanilla ReLU network, ResNet or DenseNet). Let $x$ be a non-zero vector. Then, the set $\left\{w \mid J^{\bk} = \frac{\partial f_{\bk}(x;w)}{\partial W^{\bk}}\right\}$ is of measure 1.
\end{lemma}

\begin{proof}
It holds that:
\begin{equation}
J^{\bk} = \frac{\partial f_{\bk}(x;w)}{\partial W^{\bk}} + \frac{\partial f^c_{\bk}(x;w) }{\partial W^{\bk}}
\end{equation}
We would like to prove that the set of $w$, such that, $\frac{\partial f^c_{\bk}(x;w) }{\partial W^{\bk}} = 1$ is of measure 1. 

First, we consider that the set of weights $w_{\gamma,l}$ within the expression $f^c_{\bk}(x;w) = \sum_{\gamma \in S \setminus S_{\bk} }c_\gamma z_{\gamma}\prod_{l=1}^{|\gamma|} w_{\gamma,l} $ is disjoint to the set of weights $w^k_{i,j}$ in $W^{\bk}$, since the complement $f^c_{\bk}(x;w)$ sums over the paths $\gamma$ that skip $W^{\bk}$. We note that $z_{\gamma}$ is a binary function that indicates whether the neurons along the path $\gamma$ are activated or not. Therefore, for any $\gamma\in S\setminus S_{\bk}$, we have: $\frac{\partial z_{\gamma}}{\partial W^{\bk}} = 0$ for every $w$, such that, the pre-activations of each neuron along the path $\gamma$ are non-zero (otherwise, the gradient is undefined). By Lem.~\ref{lem:zeromeasure}, the complement of this set (i.e., all $w$, such that, the pre-activation of at least one neuron along the path $\gamma$ is zero) is of measure zero. Therefore, we conclude that $\frac{\partial z_{\gamma}}{\partial W^{\bk}} = 0$ holds almost surely. Since this is true for all $\gamma \in S\setminus S_{\bk}$, we conclude that $\frac{\partial f^c_{\bk}(x;w)}{\partial W^{\bk}} = 0$ almost surely.  
\end{proof}

\begin{lemma}\label{lem:Jp}
Let $f(x;w)$ be a neural network (e.g., vanilla ReLU network, ResNet or DenseNet). Let $x$ be a non-zero vector. Then, 
\begin{equation}
\mathbb{E}[\|J^{\bk}\|^p_2] = \mathbb{E}\left[\left\|\frac{\partial f_{\bk}(x;w)}{\partial W^{\bk}}\right\|^p_2\right] 
\end{equation}
\end{lemma}

\begin{proof}
By Lem.~\ref{lem:measureone}, the set $\left\{w \mid J^{\bk} = \frac{\partial f_{\bk}(x;w)}{\partial W^{\bk}} \right\}$ is of measure $1$. Therefore, since $w$ is distributed according to a continuous distribution, we have the desired equation: $\mathbb{E}[\|J_k\|^p_2] = \mathbb{E}\left[\left\|\partial f_{\bk}(x;w)/\partial W^{\bk} \right\|^p_2 \right]$.
\end{proof}

\section{Proofs of the Main Results}

We make use of the following propositions and definitions to aid in the proofs of Thms.~\ref{thm:eq} and~\ref{thm:dual}.

\begin{proposition}\label{p1}
Given a random vector $w = [w_1...w_n]$ such that each component is identically and symmetrically distributed i.i.d random variable with moments $\mathbb{E}[w_1^m] =c_m$ (e.g., $c_0 = 1,c_1 = 0$), a set of non negative integers $m_1,...,m_l$, such that, $\sum_{i=1}^lm_i$ is even, and a random binary variable $z\in \{0,1\}$, such that, $p(z\mid w) = 1-p(z\mid -w)$, then it holds that:
\begin{equation}
\mathbb{E}\left[\prod_{i=1}^{l} w_{i}^{m_i} z \right] =\frac{\prod_{i=1}^lc_{m_i}}{2}
\end{equation}
\end{proposition}

\begin{proof}
We have:
\begin{equation}
\begin{aligned}
\prod_{i=1}^lc_{m_i} 
&= \int_w \prod_{i=1}^lw_i^{m_i}p(w)~\textnormal{d}w  \\
&=  \int_{w|z=1} \prod_{i=1}^lw_i^{m_i}p(w)~\textnormal{d}w + \int_{w|z=0}\prod_{i=1}^lw_i^{m_i}p(w)~\textnormal{d}w\\
&=\int_{w|z=1} \prod_{i=1}^lw_i^{m_i}p(w)~\textnormal{d}w + \int_{w|z=1}\prod_{i=1}^l(-w_i)^{m_i}p(w)~\textnormal{d}w\\
&=\int_{w} \prod_{i=1}^lw_i^{m_i}z\cdot p(w)~\textnormal{d}w + \int_{w}\prod_{i=1}^l(-w_i)^{m_i}z\cdot p(w)~\textnormal{d}w
\end{aligned}
\end{equation}
Since $\sum_{i=1}^l m_i$ is even, it follows that:
\begin{equation}
\int_{w}\prod_{i=1}^l(-w_i)^{m_i}z\cdot p(w)~\textnormal{d}w = \int_{w}\prod_{i=1}^lw_i^{m_i}z\cdot p(w)~\textnormal{d}w\\
\end{equation}
Therefore,
\begin{equation}
\prod_{i=1}^lc_{m_i} = 2\int_{w} \prod_{i=1}^lw_i^{m_i}z\cdot p(w)~\textnormal{d}w
\end{equation}
Put differently,
\begin{equation}
\frac{\prod_{i=1}^lc_{m_i}}{2} = \int_{w} \prod_{i=1}^lw_i^{m_i}z\cdot p(w)~\textnormal{d}w = \mathbb{E}\left[\prod_{i=1}^{l}w_{i}^{m_i} z \right]
\end{equation}
\end{proof}

\begin{proposition}\label{p2}
Given a random vector $w = [w_1,...,w_n]$, such that,its components are i.i.d symmetrically distributed random variable with moments $\mathbb{E}[w_i^m] = c_m$ ($c_0 = 1,c_1 = 0$), two sets of non negative integers $m_1,...,m_l$, $n_1,...,n_l$, such that, $\sum_{i=1}^l m_i$ , $\sum_{i=1}^l n_i$ are even, $\forall~i\in [l]: m_i\geq n_i$, and a random binary variable $z\in \{0,1\}$, such that $p(z\mid w) = 1-p(z\mid -w)$, then it holds that:
\begin{equation}
\mathbb{E}\left[\frac{1}{w_{i}^{n_i}}\prod_{i=1}^{l}w_{i}^{m_i} z\right] =\frac{\prod_{i=1}^lc_{m_i - n_i}}{2}
\end{equation}
\end{proposition}

\begin{proof}
Follows immediately from Prop.~\ref{p1} since $\sum_i (m_i - n_i)$ is even.
\end{proof}

\begin{definition}[ResNet path parametrization] Let $f(x;w)$ be a ResNet with two layer residual branches ($m=2$). A path from input to output $\gamma$ in $f$, defines a product of weights along the path denoted by: 
\begin{equation}
P_{\gamma} = \prod_{l=0}^{L+1} p_{\gamma,l}
\end{equation}
where:
\begin{equation}
p_{\gamma,l} = 
\begin{cases} 
      1 & l \notin \gamma\\
      w_{\gamma,l}^1z_{\gamma,l}w_{\gamma,l}^2 & l \in \gamma, 0<l\leq L\\
      w_{\gamma,l} & l=\{0,L+1\} 
\end{cases}
\end{equation}

Here, $w_{\gamma,l}^1,w_{\gamma,l}^2$ are weights associated with residual branch $l$, $w_{\gamma,0},w_{\gamma,L+1}$ belong to the first and last linear projection matrices $W^0,W^{L+1}$, and $z_{\gamma,l}$ is the binary activation variable relevant for weight $w_{\gamma,l}^1$. (Note that $z_{\gamma,l}$ depends on $w_{\gamma,l}^1$, but not on $w_{\gamma,l}^2$ ).
$l \notin \gamma$ indicates if layer $l$ is skipped. 
\end{definition}

\begin{definition}[DenseNet path parametrization] Let $f(x;w)$ be a DenseNet. A path $\gamma$ from input in to output in $f$, defines a product of weights along the path denoted by:
\begin{equation}
P_{\gamma} = \prod_{l=0}^{L+1} p_{\gamma,l}
\end{equation}
where:
\begin{equation}
p_{\gamma,l} = \begin{cases} 
      1 & l \notin \gamma\\
      w_{\gamma,l}z_{\gamma,l} & l \in \gamma, 0<l\leq L\\
      w_{\gamma,l} &  l=\{0,L+1\} 
   \end{cases}
\end{equation}
Here, $w_{\gamma,l}$ is a weight associated with layer $l$, $w_{\gamma,0},w_{\gamma,L+1}$ belong to the first and last linear projection matrices $W^0,W^{L+1}$, and $z_{\gamma,l}$ is the binary activation variable relevant for weight $w_{\gamma,l}$. The notation $l \notin \gamma$ indicates that the layer $l$ is skipped. 
\end{definition}

Similarly, we denote $z^{(\bk)}_{\gamma,l}$, $p^{(\bk)}_{\gamma,l}$ and $P^{(\bk)}_{\gamma}$ to be the same quantities as $z_{\gamma,l}$, $p_{\gamma,l}$ and $P_{\gamma}$ for the network $f_{(\bk)}$ instead of $f$.

\begin{proposition}\label{p3}
Let $f(x;w)$ be a ResNet/DenseNet/ANN. For any set of even $m$ paths from input to output $\{\gamma^i\}_{i=1}^m$, it holds that:
\begin{equation}
\E\left[\prod_{i=1}^m P_{\gamma^i} \right] = \begin{cases}
\prod_{l=0}^{L+1}\left(\E\left[\prod_{i=1}^m p_{\gamma^i,l}\mid \sum_{h=0}^{l-1}\|q^h\|_2>0\right]\right) & f(x;w) ~\textnormal{is DenseNet  }\\
\prod_{l=0}^{L+1}\left(\E\left[\prod_{i=1}^m p_{\gamma^i,l}\mid \|y^{l-1}\|_2>0\right]\right) & f(x;w) ~\textnormal{is ResNet or ANN}
\end{cases}
\end{equation}
\end{proposition}
\begin{proof}
We prove the claim for DenseNets. the extension to ANNs and ReseNets is trivial, and requires no further arguments.
We have that:
\begin{equation}
\mathbb{E}\left[\prod_{i=1}^m P_{\gamma^i} \right] = \mathbb{E}\left[\prod_{l=0}^{L+1}\left(\prod_{i=1}^mp_{\gamma^i,l}\right)\right]
\end{equation}
From the linearity of the last layer, it follows that:
\begin{equation}
\begin{aligned}
\mathbb{E}\left[\prod_{i=1}^m P_{\gamma^i}\right] &= \mathbb{E}\left[\prod_{l=0}^{L}\left(\prod_{i=1}^mp_{\gamma^i,l}\right)\right] \cdot \mathbb{E}\left[\prod_{i=1}^m p_{\gamma^i,{L+1}}\right] \\
&= \mathbb{E}\left[\prod_{l=0}^{L}\left(\prod_{i=1}^mp_{\gamma^i,l}\right)\right] \cdot \mathbb{E}\left[\prod_{i=1}^mw_{\gamma^i,L+1}\right]
\end{aligned}
\end{equation}
We denote by $\{w_u^{L+1}\}_{u=1}^s$ the set of $s\leq m$ unique weights in $\{w_{\gamma^i,L+1}\}_{i=1}^m$, with corresponding multiplicities $\{m_u^{L+1}\}^{s}_{u=1}$, such that, $\sum^{s}_{u=1} m_{u}^{L+1} = m$. It follows that:
\begin{equation}
\begin{aligned}
\mathbb{E}\left[\prod_{i=1}^mP_{\gamma^i}\right] 
&= \mathbb{E}\left[\prod_{l=0}^{L}\left(\prod_{i=1}^mp_{\gamma^i,l}\right)\right] \cdot \mathbb{E}\left[\prod_{u} \left(w_u^{L+1} \right)^{m_u^{L+1}}\right] \\
&= \mathbb{E}\left[\prod_{l=0}^{L}\left(\prod_{i=1}^mp_{\gamma^i,l}\right)\right] \cdot \prod_u c_{m_u^{L+1}}
\end{aligned}
\end{equation}
where $c_{m_u^{L+1}}$ is the $m_u^{L+1}$'th moment of a normal distribution.

Since the computations done by all considered architectures form a Markov chain, such that, the output of any layer depends only on the set $R^{l-1}$ of weights in the previous layers, we have that:
\begin{equation}
\begin{aligned}
\mathbb{E}\left[\prod_{l=0}^{L}\left(\prod_{i=1}^mp_{\gamma^i,l}\right)\right] = \mathbb{E}\left[\prod_{l=0}^{L-1}\left(\prod_{i=1}^mp_{\gamma^i,l}\right)\mathbb{E}\left[\prod_{i=1}^mp_{\gamma^i,L}\Bigg| R^{L-1}\right]\right]
\end{aligned}
\end{equation}
And also,
\begin{equation}
\begin{aligned}
\mathbb{E}\left[\prod_{i=1}^mp_{\gamma^i,L}\Bigg|R^{L-1}\right] 
= \mathbb{E}\left[\prod_{i=1}^mp_{\gamma^i,L} \Bigg|  R^{L-1}\right] = \mathbb{E}\left[\prod_{i=1}^mp_{\gamma^i,L} \Bigg|  q^0,...,q^{L-1}\right]
\end{aligned}
\end{equation}
We note that the pre-activations $y^L$ conditioned on $q^0,...,q^{L-1}$ are distributed according to zero mean i.i.d Gaussian variables. In addition, the coordinates of $q^L = 2\phi(y^L)$ are i.i.d distributed. We denote by $\{z_u\}_{u=1}^s$ the set of unique activation variables in the set $\{z_{\gamma^i,L}\}_{i=1}^m$. For each $z_u$, we denote by $\{w_{u,v}^L\}$ the set of unique weights in $\{w_{\gamma^i,L}\}$ multiplying $z_u$, with corresponding multiplicities $m_{u,v}^L$, such that, $\sum_{u,v}m_{u,v}^L = m$, and $\sum_{v}m_{u,v}^L = m_u^{L+1}$. Note that, from the symmetry of the normal distribution, it holds that odd moments vanish, and so we only need to consider even $m_u^{L+1}$ for all $u$.  From the independence of the set $\{z_u\}$, the expectation takes a factorized form:
\begin{equation}
\begin{aligned}
\mathbb{E}\left[\prod_{i=1}^mp_{\gamma^i,L}\Bigg| q^0...q^{L-1}\right] &= \mathbbm{1}\left[\sum_{l=0}^{L-1}\|q^l\|_2>0\right] \cdot \mathbb{E}\left[\prod_{i=1}^mp_{\gamma^i,L}\mid q^0,...,q^{L-1}\right] \\
&= \mathbbm{1}\left[\sum_{l=0}^{L-1}\|q^l\|_2>0\right] \cdot \prod_{u=1}^s\mathbb{E}\left[ z_u\prod_{v}(w_{u,v}^L)^{m_{u,v}^L}\Bigg| q^0,...,q^{L-1}\right]
\end{aligned}
\end{equation}
Using Prop.~\ref{p1}:
\begin{equation}\label{trans}
\begin{aligned}
&\prod_{u=1}^s\mathbb{E}\left[ z_u\prod_{v}(w_{u,v}^L)^{m_{u,v}^L}\Bigg| q^0...q^{L-1}\right] \\
=& \mathbbm{1}\left[\sum_{l=0}^{L-1}\|q^l\|_2>0\right] \cdot \prod_{u=1}^s\left(\frac{\prod_{v}c_{m_{u,v}^L}}{2} \right)\\
=& \mathbbm{1}\left[\sum_{l=0}^{L-1}\|q^l\|_2>0\right] \cdot \mathbb{E}\left[ \prod_{i=1}^mp_{\gamma^i,L}\Bigg| \sum_{l=0}^{L-1}\|q^l\|_2>0\right]
\end{aligned}
\end{equation}
It then follows:
\begin{equation}\label{factorized}
\begin{aligned}
\mathbb{E}\left[\prod_{l=0}^{L}(\prod_{i=1}^mp_{\gamma^i,l})\right] &= \mathbb{E}\left[\mathbbm{1}\left[\sum_{l=0}^{L-1}\|q^l\|_2>0\right]\cdot \prod_{l=0}^{L-1}\left(\prod_{i=1}^mp_{\gamma^i,l}\right)\right]\cdot \prod_{u=1}^s\left(\frac{\prod_{v}c_{m_{u,v}^L}}{2}\right)\\
&= \mathbb{E}\left[\prod_{l=0}^{L-1}\left(\prod_{i=1}^mp_{\gamma^i,l}\right)\right]\cdot \mathbb{E}\left[ \prod_{i=1}^mp_{\gamma^i,L}\Bigg| \sum_{l=0}^{L-1}\|q^l\|_2>0\right]
\end{aligned}
\end{equation}
Recursively applying the above completes the proof.
\end{proof}

\begin{restatable}{theorem}{eq}\label{thm:eq}
Let $f(x;w)$ be a ResNet/DenseNet. Then, for any non-negative even integer $m$, we have:
\begin{equation}\label{eq:equi}
\forall~\bk :~ \mathbb{E}\left[(f_{(\bk)}(x;w))^{m} \right] =  \mathbb{E}\left[(f_{\bk}(x;w))^{m} \right]
\end{equation}
\end{restatable}

\begin{proof}
We present the proof using the DenseNet path parameterization. Extending to ResNet parameterization is trivial and requires no additional arguments.
We aim to show that for any even integer $m>0$, and $\forall~\bk = \{l_k,h_k\}$:
\begin{equation}
\mathbb{E}\left[(f_{(\bk)}(x;w))^{m} \right] =  \mathbb{E}\left[(f_{\bk}(x;w))^{m} \right]
\end{equation}
The output $f_{\bk}(x;w)$ can be expressed in the following manner:
\begin{equation}
f_{\bk}(x;w) = \sum_{\gamma \in S_{\bk}} c_\gamma \prod_{l=0}^{L+1} p_{\gamma,l}
\end{equation}
Since the output $f_{(\bk)}(x;w)$ is composed of products of weights and activations along the same paths $\gamma \in S_{\bk}$ as $f_{\bk}$ (with different activation variables), we only need to prove the following: for any weight matrix $W^{\bk}$, and a set of $m$ paths $\gamma^1,...,\gamma^m \in S_{\bk}$, it holds that:
\begin{equation}
\mathbb{E}\left[\prod_{i=1}^m P_{\gamma^i}\right] = \mathbb{E}\left[\prod_{i=1}^m P^{(\bk)}_{\gamma^i}\right]
\end{equation}
Using Prop.~\ref{p3}:
\begin{equation}
\prod_{l=0}^{L+1}\left(\E\left[\prod_{i=1}^m p_{\gamma^i,l}\Bigg|\sum_{h=1}^{l-1}\|q^h\|_2>0\right]\right) = \prod_{l=0}^{L+1}\left(\E\left[\prod_{i=1}^m p^\bk_{\gamma^i,l}\Bigg|\sum_{h=1}^{l-1}\|q_{(\bk)}^h\|_2>0\right]\right)
\end{equation}

Note that for both the full and reduced architectures, flipping the sign of all the weights in layer $l$ will flip the ensuing activation variables (except for a set of measure zero defined by $\sum_{l=0}^{l_k-1}W^{l_k,l}q^l = 0$, which does not affect the expectation. And so, using Prop.~\ref{p1} along with Eq.~\ref{trans}: 
\begin{equation}
\E\left[\prod_{i=1}^m p_{\gamma^i,l}\Bigg|\sum_{h=1}^{l-1}\|q^h\|_2>0\right] = \E\left[\prod_{i=1}^m p^\bk_{\gamma^i,l}\Bigg|\sum_{h=1}^{l-1}\|q_{(\bk)}^h\|_2>0\right]
\end{equation}
Completing the proof.
\end{proof}

\begin{restatable}{theorem}{dual}\label{thm:dual}
Let $f(x;w)$ be a ResNet/DenseNet. Then, we have:
\begin{enumerate}\label{eq:du}
  \item $\forall~\bk: ~\mathbb{E}\big[\|J^{\bk}\|^2_2 \big] = \mathbb{E}\big[(f_{(\bk)}(x;w))^2 \big]$.
  \item $\forall~\bk: ~\frac{\mathbb{E}\big[(f_{(\bk)}(x;w))^4 \big]}{3} \leq \mathbb{E}\big[\|J^{\bk}\|^4_2 \big] \leq \mathbb{E}\big[(f_{(\bk)}(x;w))^4 \big]$.
\end{enumerate}
\end{restatable}

\begin{proof}
We present the proof using the DenseNet path parameterization. Extending to ResNet parameterization is trivial and requires no additional arguments. Neglecting scaling coefficients for notational simplicity, let $\bk = (l_k,h_k)$ be an index of a weight matrix $W^{\bk}$ in $f(x;w)$, by Lem.~\ref{lem:Jp}, we have:
\begin{equation}
\begin{aligned}
\mathbb{E}\left[\|J^{\bk}\|^2_2\right] 
= \mathbb{E}\left[\left\|\frac{\partial f_{\bk}(x;w)}{\partial W^{\bk}}\right\|^2_2\right] 
= \sum_{i,j}\mathbb{E}\left[\left(\sum_{\gamma \in S \textnormal{ s.t: } w_{i,j}^{\bk}\in \gamma}\frac{1}{w_{i,j}^{\bk}}P_{\gamma}\right)^2\right]
\end{aligned}
\end{equation}
where $\gamma \textnormal{ s.t: } w_{i,j}^{\bk}\in \gamma$ denotes a path that includes the weight $w_{i,j}^\bk$. From Prop.~\ref{p3}, the expectation is factorized as follows:
\begin{equation}\label{second}
\begin{aligned}
&\mathbb{E}\left[\left\|\frac{\partial f_{\bk}(x;w)}{\partial W^{\bk}}\right\|^2_2\right] \\
=&\sum_{i,j}\sum_{\gamma \in S \textnormal{ s.t: } w_{i,j}^{\bk}\in \gamma}\mathbb{E} \left[\left(\frac{1}{w_{i,j}^{\bk}}p_{\gamma,l_k} \right)^2 \Bigg| \sum_{h=0}^{l_k-1}\|q^h\|_2>0\right] \cdot \prod_{l\neq l_k} \mathbb{E} \left[\left(p_{\gamma,l_k}\right)^2\Bigg| \sum_{h=0}^{l-1}\|q^h\|_2>0\right]
\end{aligned}
\end{equation}
Using Props.~\ref{p1} and~\ref{p2}, for all $\gamma \in S$, such that, $w_{i,j}^{\bk}\in \gamma$, we have:
\begin{equation}
\begin{aligned}
&\mathbb{E}\left[\left(\frac{1}{w_{i,j}^{\bk}}p_{\gamma,l_k}\right)^2\Bigg| \sum_{h=0}^{l_k-1}\|q^h\|_2>0\right]\\ 
=& \mathbb{E} \left[\left(\frac{w_{i,j}^{\bk} z_{\gamma,l_k}}{w_{i,j}^{\bk}}\right)^2\Bigg| \sum_{h=0}^{l_k-1}\|q^h\|_2>0\right] = 1/2 = \mathbb{E} \left[(p_{\gamma,l_k})^2\Bigg| \sum_{h=0}^{l_k-1}\|q^h\|_2>0\right]
\end{aligned}
\end{equation}
Inserting into Eq.~\ref{second}, and using Thm.~\ref{thm:eq} proves the first claim. 

Next we would like to prove the second claim. By Lem.~\ref{lem:zeromeasure}, we have:
\begin{equation}
\begin{aligned}
\mathbb{E}\left[\|J^{\bk}\|^4_2\right] =& 
\mathbb{E}\left[\left\|\frac{\partial f_{\bk}(x;w)}{\partial W^{\bk}}\right\|^2_2 \cdot \left\|\frac{\partial f_{\bk}(x;w)}{\partial W^{\bk}}\right\|^2_2\right]\\
=& \sum_{i,j}\sum_{i',j'}\mathbb{E} \left[\left(\sum_{\gamma \textnormal{ s.t } w_{i,j}^{\bk}\in \gamma}\frac{1}{w_{i,j}^{\bk}}P_{\gamma}\right)^2\left(\sum_{\gamma \textnormal{ s.t } w_{i',j'}^{\bk}\in \gamma}\frac{1}{w_{i',j'}^{\bk}}P_{\gamma}\right)^2\right]\\
=& \sum_{i,i',j,j'}\mathbb{E} \left[\frac{1}{(w_{i,j}^{\bk})^2(w_{i',j'}^\bk)^2}\sum_{\gamma^1,\gamma^2 \textnormal{ s.t } w_{i,j}^{\bk}\in \gamma^1,\gamma^2}~~\sum_{\gamma^3,\gamma^4 \textnormal{ s.t } w_{i',j'}^{\bk}\in \gamma^3,\gamma^4}P_{\gamma^1}P_{\gamma^2}P_{\gamma^3}P_{\gamma^4}\right]
\end{aligned}
\end{equation}
By applying Prop.~\ref{p3}, the expectation is factorized as follows:
\begin{equation}\label{fourth}
\begin{aligned}
&\mathbb{E}\left[\|J^{\bk}\|^4_2\right] \\
=& \sum_{\substack{i,i',j,j' \\ \gamma^1,\gamma^2 \textnormal{ s.t } w_{i,j}^{\bk}\in \gamma^1,\gamma^2 \\
\gamma^3,\gamma^4 \textnormal{ s.t } w_{i',j'}^{\bk}\in \gamma^3,\gamma^4
}}
\left[\mathbb{E} \left[\frac{\prod_{h=1}^4p_{\gamma^h,l_k}}{(w_{i,j}^{\bk})^2(w_{i',j'}^\bk)^2}\Bigg|\sum_{h=0}^{l_k-1}\|q^h\|>0\right] \cdot \prod_{l\neq l_k} \mathbb{E} \left[\prod_{h=1}^4p_{\gamma^h,l}\Bigg|\sum_{h=0}^{l-1}\|q^h\|>0\right]\right]
\end{aligned}
\end{equation}
Using Props.~\ref{p1} and~\ref{p2}, for all $\gamma^1,\gamma^2$, such that, $w_{i,j}^k\in \gamma^1 $ and $w_{i',j'}^\bk\in \gamma^2$, we have:
\begin{equation}
\begin{aligned}
&\mathbb{E} \left[\frac{\prod_{h=1}^4p_{\gamma^h,k}}{(w_{i,j}^{\bk})^2(w_{i',j'}^\bk)^2}\Bigg|\sum_{h=0}^{l_k-1}\|q^h\|_2>0\right]\\
=& \mathbb{E} \left[\frac{(w_{i,j}^{\bk})^2(w_{i',j'}^\bk)^2 z_{\gamma^1,k}z_{\gamma^2,k}}{(w_{i,j}^{\bk})^2(w_{i',j'}^\bk)^2}\Bigg|\sum_{h=0}^{k-1}\|q^h\|_2>0\right]\\
=& \begin{cases} 
      1/2 & w_{i,j}^\bk \equiv w_{i',j'}^\bk\\
      1/4 & otherwise
   \end{cases}\\
=& \begin{cases} 
      \frac{1}{3}\mathbb{E} \left[\prod_{h=1}^4p_{\gamma^h,k}\Bigg|\sum_{h=0}^{l_k-1}\|q^h\|_2>0\right]\ & w_{i,j}^\bk \equiv w_{i',j'}^\bk\\
      \mathbb{E} \left[\prod_{h=1}^4p_{\gamma^h,k}\Bigg|\sum_{h=0}^{l_k-1}\|q^h\|_2>0\right]\ & otherwise 
   \end{cases}
\end{aligned}
\end{equation}
Inserting into Eq.~\ref{fourth} proves the second claim. 
\end{proof}

We use the following proposition to aid in the proofs of Thms.~\ref{thm:res_ntk_var} and~\ref{thm:dense_ntk_var}.

\begin{proposition}\label{p4}
Let $f(x;w)$ be a vanilla fully connected ReLU network, with intermediate outputs given by:
\begin{equation}
\forall~0\leq l \leq L:~y^{l} = \sqrt{2}\phi\left(\frac{1}{\sqrt{n_{l-1}}}W^{l} y^{l-1}\right)
\end{equation}
where the weight matrices $W^l \in \mathbb{R}^{n_l \times n_{l-1}}$ are normally distributed. Then, the following holds at initialization:
\begin{equation}
\begin{aligned}
\mathbb{E}\left[\|y^l\|^2_2\right] &= \frac{n_l}{n_{l-1}}\mathbb{E}\left[\|y^{l-1}\|^2_2\right]\\
\mathbb{E}\left[\|y^l\|^4_2\right] &=  \frac{n_l(n_l+5)}{n_{l-1}^2}\mathbb{E}\left[\|y^{l-1}\|^4_2\right]
\end{aligned}
\end{equation}
\begin{proof}
Absorbing the scale $\sqrt{\frac{2}{n_{l-1}}}$ into the weights, we denote by $Z^l$ the diagonal matrix holding in its diagonal the activation variables $z_j^l$ for unit $j$ in layer $l$, and so we have:
\begin{equation}
y^{l} = Z^l W^{l} y^{l-1}
\end{equation}
Conditioning on $R^{l-1} = \{W^1,...,W^{l-1}\}$ and taking expectation:
\begin{equation}
\begin{aligned}
\mathbb{E}\left[\|y^l\|^2_2\mid R^{l-1} \right] &= {y^{l-1}}^{\top}\mathbb{E} \left[{W^l}^{\top}Z^lW^{l}\right]y^{l-1}\\
&=\sum_{j=1}^{n_l}\sum_{i_1,i_2=1}^{n_{l-1}}y_{i_1}^{l-1}y_{i_2}^{l-1}\mathbb{E}\left[w_{i_1,j}^lw_{i_2,j}^lz_j^l\mid R^{l-1}\right]
\end{aligned}
\end{equation}
From Prop.~\ref{p1}, it follows that:
\begin{equation}
\mathbb{E}\left[\|y^l\|^2_2\right] = \mathbb{E}\left[\mathbb{E}\left[\|y^l\|^2_2\mid R^{l-1}\right]\right] = \frac{n_L}{n_{L-1}}\mathbb{E}\left[\|y^{l-1}\|^2_2\right]
\end{equation}
Similarly:
\begin{equation}
\begin{aligned}
\mathbb{E}\left[\|y^l\|^4_2 \mid R^{l-1}\right] &= \mathbb{E}\left[\left({y^{L-1}}^{\top}{W^L}^{\top}Z^L W^{L} y^{L-1}\right)^2\Big| R^{l-1}\right]\\
&=\sum_{j_1,j_2,i_1,i_2,i_3,i_4} \prod^{4}_{t=1} y_{i_t}^{l-1} \cdot \mathbb{E}\left[w_{i_1,j_1}^lw_{i_2,j_1}^lw_{i_3,j_2}^lw_{i_4,j_2}^lz_{j_1}^lz_{j_2}^l\mid R^{l-1}\right]
\end{aligned}
\end{equation}
From Prop.~\ref{p1}, and the independence of the activation variables conditioned on $R^{l-1}$: 
\begin{equation}
\begin{aligned}
&\sum_{j_1,j_2,i_1,i_2,i_3,i_4} \prod^{4}_{t=1} y_{i_t}^{l-1} \cdot \mathbb{E}\left[w_{i_1,j_1}^lw_{i_2,j_1}^lw_{i_3,j_2}^lw_{i_4,j_2}^lz_{j_1}^lz_{j_2}^l|R^{l-1}\right] \\
=& \sum_{j_1,j_2,i_1,i_2,i_3,i_4} \prod^{4}_{t=1} y_{i_t}^{l-1} \cdot \mathbb{E}\left[w_{i_1,j_1}^lw_{i_2,j_1}^lw_{i_3,j_2}^lw_{i_4,j_2}^lz_{j_1}^lz_{j_2}^l|R^{l-1}\right]\\
&\cdot \Big(\mathbbm{1}_{j_1=j_2,i_1=i_2=i_3=i_4} + \mathbbm{1}_{j_1=j_2,i_1=i_2,i_3=i_4,i_1\neq i_3}\\
&\;\;\;\;\;\;\;+\mathbbm{1}_{j_1=j_2,i_1=i_3,i_2=i_1,i_2\neq i_3} + \mathbbm{1}_{j_1=j_2,i_1=i_4,i_2=i_3,i_1\neq i_2} +\mathbbm{1}_{j_1\neq j_2,i_1=i_2,i_3=i_4}\Big)
\end{aligned}
\end{equation}
and so: 
\begin{equation}
\begin{aligned}
\mathbb{E}\left[\|y^l\|^4_2\right] &= \frac{n_l}{2}\sum_{i}\mathbb{E}\left[(y_i^{l-1})^4\right] + \frac{6n_l}{n_{l-1}^2}\sum_{i_1\neq i_2}\mathbb{E}\left[(y_{i_1}^{l-1})^2(y_{i_2}^{l-1})^2\right]\\
&+ \frac{n_l(n_l-1)}{n_{l-1}^2}\sum_{i_1, i_2}\mathbb{E}\left[(y_{i_1}^{l-1})^2(y_{i_2}^{l-1})^2\right]\\
&=\frac{n_l(n_l+5)}{n_{l-1}^2}\mathbb{E}\left[\|y^{l-1}\|^4_2\right]
\end{aligned}
\end{equation}
proving the claim.

\end{proof}

\end{proposition}

\begin{proposition}\label{p5}
For a vanilla fully connected linear network, with intermediate outputs given by:
\begin{equation}
\forall~0\leq l \leq L:~y^{l} = \frac{1}{\sqrt{n_{l-1}}}W^{l}y^{l-1}
\end{equation}
where the weight matrices $W^l \in \mathbb{R}^{n_l \times n_{l-1}}$ are normally distributed, the following holds at initialization:
\begin{equation}
\begin{aligned}
\mathbb{E}\left[\|y^l\|^2_2\right] &= \frac{n_l}{n_{l-1}}\mathbb{E}\left[\|y^{l-1}\|^2_2\right]\\
\mathbb{E}\left[\|y^l\|^4_2\right] &=  \frac{n_l(n_l+2)}{n_{l-1}^2}\mathbb{E}\left[\|y^{l-1}\|^4_2\right]
\end{aligned}
\end{equation}
\end{proposition}

\proof{
The proof follows immediately from the derivation of Prop.~\ref{p4}, and will be omitted for brevity.
}

\begin{restatable}{theorem}{resNTKvar}\label{thm:res_ntk_var}
Let $f(x;w)$ be a depth $L$, constant width = $n$ ResNet with residual branches of depth $m$ and positive initialization constants $\{\alpha_l\}_{l=1}^L$. Then, there exists a constant $C>0$ such that:
\begin{equation}
\max\left[1,\frac{\sum_{\bu} \alpha_{l_u}^2}{\sum_{\bu,\bv} \alpha_{l_u}\alpha_{l_v}} \cdot \xi \right] \leq \eta(n,L) \leq \xi~\textnormal{ where: }~ \xi = \exp\left[\frac{5m}{n}+\frac{C}{n}\sum_{l=1}^L \frac{\alpha_l}{1+\alpha_l} \right] \cdot \left(1 + \mathcal{O}(1/n)\right)
\end{equation}
\end{restatable}

\begin{proof}
Using the result of Thm.~\ref{thm:dual}, and using Cauchy–Schwartz inequality, an upper bound to $\eta$ can be derived:
\begin{equation}
\begin{aligned}
\eta 
&= \frac{\mathbb{E}[\mathcal{G}(x,x)^2]}{\mathcal{K}^R_L(x,x)^2} \\  
&= \frac{\sum_{\bu,\bv}\mathbb{E}[\|J^{\bu}\|^2_2\cdot \|J^{\bv}\|^2_2]}{\mathcal{K}^R_L(x,x)^2} \\
&\leq \frac{\sum_{\bu,\bv}\sqrt{\mathbb{E}[\|J^{\bu}\|^4_2] \cdot \mathbb{E}[\|J^{\bv}\|^2_2]}}{\mathcal{K}^R_L(x,x)^2} \\
&\leq \frac{\sum_{\bu,\bv}\sqrt{\mathbb{E}[\|f_{(\bu)}(x;w)\|^4_2] \cdot \mathbb{E}[\|f_{(\bv)}(x;w)\|^2_2]}}{\mathcal{K}^R_L(x,x)^2}
\end{aligned}
\end{equation}
The lower bound is similarly derived using Thm.~\ref{thm:dual}:
\begin{equation}
\eta \geq  \frac{\sum_{\bk}\mathbb{E}[\|J^{\bk}\|^4_2]}{\mathcal{K}^R_L(x,x)^2} \geq \frac{1}{3}\cdot \frac{\sum_{\bk}\mathbb{E}[\|f_{(\bk)}(x;w)\|^4_2]}{\mathcal{K}^R_L(x,x)^2}
\end{equation}

The asymptotic behaviour of $\eta$ is therefore governed by the propagation of the fourth moment $\mathbb{E}[\|y_{(\bk)}^l\|^4_2]$ through the model.

In the following proof, for the sake of notation simplicity, we omit the notation $\bk = (l_k,h_k)$ in $y^l_{(\bk)}$, and assume that $y^l$ stands for the reduced network  $y^l_{(\bk)}$.
The recursive formula for the intermediate outputs of the reduced network are given by:
\begin{equation}
y^l = 
\begin{cases} 
      y^{l-1} + \sqrt{\alpha_l}y^{l-1,m} & 0<l\leq L, l\neq l_k  \\
      \sqrt{\alpha_l}y^{l-1,m} & l=l_k
   \end{cases}
\end{equation}
where:
\begin{equation}
y^{l-1,h} = \begin{cases} 
      \sqrt{\frac{1}{n}}W^{l,h} q^{l-1,h-1} & 1<h\leq m  \\
      \sqrt{\frac{1}{n}}W^{l,h} y^{l-1} & h=1 
   \end{cases}
\end{equation}
with $q^{l-1,h} = \sqrt{2}\phi(y^{l-1,h})$.

Using the results of Props.~\ref{p4} and~\ref{p5}, for layer $L$, we have:
\begin{equation}\label{rm}
\begin{aligned}
\mathbb{E}\left[\|y^L\|^2_2 \right] 
=&  \mathbb{E}\left[\|y^{L-1}\|^2_2\right] +  \frac{\alpha_L}{n}\mathbb{E}\left[{y^{L-1,m-1}}^{\top}{W^{L,m}}^{\top}W^{L,m}y^{L-1,m-1}\right]\\ 
=&   \mathbb{E}\left[\|y^{L-1}\|^2_2\right] + \alpha_L\mathbb{E}\left[\|y^{L-1,m-1}\|^2_2\right] \\
=& \mathbb{E}\left[\|y^{L-1}\|^2_2\right]\cdot (1 + \alpha_L)\\
=& \mathbb{E}\left[\|y^{l_k}\|^2_2\right]\prod_{l=l_k+1}^L (1 + \alpha_l) \\
=& \mathbb{E}\left[\|y^{l_k-1}\|^2_2\right]\alpha_{l_k}\prod_{l=l_k+1}^L(1 + \alpha_l)\\
=&\alpha_{l_k}\mathbb{E}[\|y^0\|^4_2]\prod^{L}_{\substack{l=1\\ l\neq l_k}}(1 + \alpha_l ) 
\end{aligned}
\end{equation}
For the fourth moment, using the results of Props.~\ref{p4} and~\ref{p5} (taking into account that odd powers will vanish in expectation), it holds: 
\begin{equation}\label{res}
\begin{aligned}
\mathbb{E}\left[\|y^L\|^4_2 \right] 
=&  \mathbb{E}\left[\|y^{L-1}\|^4_2\right] +  \alpha_L^2\mathbb{E}\left[\|y^{L-1,m}\|^4_2\right]
\\
&+ 4\alpha_L\mathbb{E}\left[\left({y^{L-1,m}}^{\top}y^{L-1}\right)^2\right]
+2\alpha_L\mathbb{E}\left[\|y^{L-1,m}\|^2_2 \cdot \|y^{L-1}\|^2_2\right]
\end{aligned}
\end{equation}
Next, we analyze each term separately:
\begin{equation}
\mathbb{E}\left[\|y^{L-1,m}\|^4_2\right] = \mathbb{E}\left[\mathbb{E}[\|y^{L-1,m}\|_2^4
\mid R^{L-1}]\right]
\end{equation}
Using the results of Props.~\ref{p4} and~\ref{p5}:
\begin{equation}
\begin{aligned}
\mathbb{E}\left[\|y^{L-1,m}\|^4_2\mid R^{L-1}\right] 
&= \left(1+2/n\right)\cdot \left(1+5/n\right)^{m-1}\cdot \|y^{L-1}\|^4_2 \\
&\sim \left(1+5/n\right)^{m} \cdot \|y^{L-1}\|^4_2 \\
\end{aligned}
\end{equation}
In addition,
\begin{equation}
\begin{aligned}
\mathbb{E}\left[\left({y^{L-1,m-1}}^{\top}y^{L-1}\right)^2\right] &= \frac{1}{n}\sum_{j_1,j_2,i_1,i_2}\mathbb{E}\left[y^{L-1,m-1}_{i_1}y^{L-1,m-1}_{i_2}y^{L-1}_{j_1}y^{L-1}_{j_2}w^{L,m}_{i_1,j_1}w^{L,m}_{i_2,j_2} \right]\\
&= \frac{1}{n}\mathbb{E}\left[\|y^{L-1,m-1}\|^2_2\cdot \|y^{L-1}\|^2_2 \right] \\
&= \frac{1}{n}\mathbb{E}\left[\|y^{L-1}\|^4_2 \right]\\
\end{aligned}
\end{equation}
and also,
\begin{equation}
\mathbb{E}\left[\|y^{L-1,m}\|^2_2 \cdot \|y^{L-1}\|^2_2\right] = \mathbb{E}\left[\|y^{L-1}\|^4_2 \right]
\end{equation}
Plugging it all into Eq.~\ref{res}, by recursion, we have:
\begin{equation}
\mathbb{E}\left[\|y^L\|^4_2 \right] \sim \mathbb{E}\left[\|y^{l_k}\|^4_2 \right] \cdot \prod_{l=l_k+1}^L\beta_l
\end{equation}
where,
\begin{equation}
\beta_l := 1+2\alpha_l\left(1+2/n\right) + \alpha_l^{2}\left(1+5/n\right)^{m}
\end{equation}
In the reduced architecture, the transformation from layer $l_k-1$ to layer $l_k$ is given by an $m$ layer fully connected network, with a linear layer on top, we can use the results from the vanilla case, and assigning $\|y^0\|^4_2 = 1$:
\begin{equation}
\begin{aligned}
\mathbb{E}\left[\|y^L\|^4_2 \right] 
&= \alpha_{l_k}^{2}\left(1+2/n\right)\cdot \left(1+5/n\right)^{m-1}\prod_{l\neq l_k}^L\beta_l \\
&\sim \alpha_{l_k}^{2} \left(1+5/n\right)^{m}\prod_{l\neq l_k}^L\beta_l
\end{aligned}
\end{equation}
Denoting $\rho = \left(1+5/n\right)^\frac{m}{2}$, and using the following:
\begin{equation}
\beta_l \sim \left(1 + \alpha_l\rho\right)^2
\end{equation}
It follows that:
\begin{equation}
\begin{aligned}
\mathbb{E}[\mathcal{G}(x,x)^2] 
&\lessapprox \sum_{\bu,\bv}\sqrt{\mathbb{E}[\|y_{(\bu)}^L\|^4_2]\mathbb{E}[\|y_{(\bv)}^L\|^2]}\\
&\sim  \left(1+5/n\right)^{m}\sum_{\bu,\bv}\alpha_{l_u}\alpha_{l_v}\sqrt{\left(\prod_{l\neq l_u}^L\beta_l\right)\left(\prod_{l\neq l_v}^L\beta_l \right)}\\
&= \left(1+5/n\right)^{m}\sum_{\bu,\bv}\alpha_{l_u}\alpha_{l_v}\left[\prod_{l\neq l_u}(1 + \rho\alpha_l)\right] \cdot \left[\prod_{l\neq l_v}(1 + \rho\alpha_l)\right]
\end{aligned}
\end{equation}
where $\bu = (l_u,h_u)$ and $\bv = (l_v,h_v)$. 

Similarly, we have:
\begin{equation}
\begin{aligned}
\mathbb{E}[\mathcal{G}(x,x)^2] 
&\gtrapprox \sum_{\bk}\mathbb{E}[\|J^{\bk}\|^4_2] \sim  \left(1+5/n\right)^{m}\cdot \sum_{\bu}\alpha_{l_u}^2\prod_{l\neq l_u}^L\beta_l = \left(1+5/n\right)^{m}\cdot \sum_{\bu}\alpha_{l_u}^2\prod_{l\neq l_u}^L(1 + \rho\alpha_l)^2
\end{aligned}
\end{equation}
Using Eq.~\ref{rm}, we have that:
\begin{equation}
\E[\mathcal{G}(x,x)]^2 = \sum_{\bu,\bv} \alpha_{l_u} \alpha_{l_v}\left(\prod_{l\neq l_u}(1 + \alpha_l)\right) \cdot \left(\prod_{l\neq l_v}(1 + \alpha_l)\right) 
\end{equation}
This yields that:
\begin{equation}
\begin{aligned}
\frac{\mathbb{E}[\mathcal{G}(x,x)^2]}{\E[\mathcal{G}(x,x)]^2} 
&\lessapprox \left(1+5/n\right)^{m} \cdot \frac{\sum_{\bu,\bv}\alpha_{l_u}\alpha_{l_v}\left(\prod_{l\neq l_u}(1 + \rho\alpha_l)\right)\left(\prod_{l\neq l_v}(1 + \rho\alpha_l)\right)}{\sum_{\bu,\bv} \alpha_{l_u} \alpha_{l_v}\left(\prod_{l\neq l_u}(1 + \alpha_l)\right)\left(\prod_{l\neq l_v}(1 + \alpha_l)\right)} \\
&\sim \left(1+5/n\right)^{m}\cdot \frac{\sum_{\bu,\bv}\alpha_{l_u}\alpha_{l_v}\left(\prod_{l=1}^L(1 + \rho\alpha_l)\right)\left(\prod_{l=1}^L(1 + \rho\alpha_l)\right)}{\sum_{\bu,\bv} \alpha_{l_u} \alpha_{l_v}\left(\prod_{l=1}^L(1 + \alpha_l)\right)\left(\prod_{l=1}^L(1 + \alpha_l)\right)}\\
&= \left(1+5/n\right)^{m}\cdot \frac{\left(\prod_{l=1}^L(1 + \rho\alpha_l)\right)^2}{ \left(\prod_{l=1}^L(1 + \alpha_l)\right)^2} \\
&= \left(1+5/n\right)^{m}\cdot \left(\prod_{l=1}^L\left(1 + \frac{\alpha_l(\rho - 1)}{1 + \alpha_l}\right)\right)^2\\
&\sim \exp \left[\frac{5m}{n} + \frac{C}{n}\sum^{L}_{l=1} \frac{\alpha_l}{1+\alpha_l} \right]\left(1 + \mathcal{O}(1/n)\right)
\end{aligned}
\end{equation}
For the lower bound, we have:
\begin{equation}
\begin{aligned}
\frac{\mathbb{E}[\mathcal{G}(x,x)^2]}{\E[\mathcal{G}(x,x)]^2} 
&\gtrapprox \left(1+5/n\right)^{m}\cdot \frac{\sum_{\bu}\alpha_{l_u}^2\left(\prod_{l\neq l_u}^L(1 + \rho\alpha_l)\right)^2}{\sum_{\bu,\bv} \alpha_{l_u} \alpha_{l_v} \left(\prod_{l\neq l_u}(1 + \alpha_l)\right)\left(\prod_{l\neq l_v}(1 + \alpha_l)\right) } \\
&\sim \frac{\sum_{\bu} \alpha_{l_u}^2}{\sum_{\bu,\bv} \alpha_{l_u}\alpha_{l_v}}\exp \left[\frac{5m}{n} + \frac{C}{n}\sum_{l=1}^L \frac{\alpha_l}{1+\alpha_l} \right]\left(1 + \mathcal{O}(1/n)\right)
\end{aligned}
\end{equation}
Since $\mathbb{E}[\mathcal{G}(x,x)^2]>\E[\mathcal{G}(x,x)]^2$, the lower bound is given by:
\begin{equation}
\begin{aligned}
\frac{\mathbb{E}[\mathcal{G}(x,x)^2]}{\E[\mathcal{G}(x,x)]^2} 
\gtrapprox& \max \left[1,\frac{\sum_{\bu} \alpha_{l_u}^2}{\sum_{\bu,\bv} \alpha_{l_u}\alpha_{l_v}}\exp \left[\frac{5m}{n} + \frac{C}{n}\sum_{l=1}^L \frac{\alpha_l}{1+\alpha_l} \right] \left(1 + \mathcal{O}(1/n)\right)\right]
\end{aligned}
\end{equation}
\end{proof}


\begin{restatable}{theorem}{denseNTKvar}\label{thm:dense_ntk_var}
Let $f(x;w)$ be a constant width = $n$ DenseNet with initialization constant $\alpha>0$. Then, there exist constants $C_1,C_2>0$, such that: 
\begin{equation}\label{surprise}
\max\left[1,\frac{C_1}{L\log(L)^2}\cdot \xi\right] \leq \eta(n,L) \leq \xi~\textnormal{ where: }~ \xi = \exp\left[C_2/n\right]\cdot \left(1 + \mathcal{O}(1/n)\right)
\end{equation}
\end{restatable}

\begin{proof}
In the following proof, for the sake of notation simplicity, we omit the notation $\bk = (l_k,h_k)$ in $y^l_{(\bk)}$, and assume that $y^l$ stands for the reduced network  $y^l_{(\bk)}$.
The recursive formula for the intermediate outputs of the reduced network are given by:
\begin{equation}
y^l = 
\begin{cases} 
\sqrt{\frac{\alpha}{nl}}\sum_{h=k}^{l-1}W^{l,h}q^h
& l_k<l\leq L \\
\sqrt{\frac{\alpha}{nl}}\sum_{h=0}^{l-1}W^{l,h}q^h & 1\leq l <l_k\\
\sqrt{\frac{\alpha}{nl_k}}W^{l_k,h_k-1}q^{l_k-1} & l = l_k
\end{cases}
\end{equation}
with $q^h = \sqrt{2}\phi(y^h)$. We define, $\mu_l := \mathbb{E}\left[\|q^l\|_2^2 \right]$. It follows that: 
\begin{equation}
\begin{aligned}
\mu_L &= \mathbb{E}\left[\|q^L\|_2^2 \right] = \frac{2\alpha}{Ln}\mathbb{E}\left[\left(\sum_{l=l_k}^{L-1}{q^{l}}^{\top}W^{L,l}\right)Z^L\left(\sum_{l=l_k}^{L-1} {q^{l}}^{\top} W^{L,l}\right) \right] = \frac{\alpha}{L}\sum_{l=l_k}^{L-1}\mu_l
\end{aligned}
\end{equation}
where $Z^l$ is a diagonal matrix holding in its diagonal the activation variables $z^l_j$ for unit $j$ in layer $l$.

Next, by telescoping the mean:
\begin{equation}
\begin{aligned}
\mu_L &= \frac{\alpha}{L}\sum_{l=l_k}^{L-1}\mu_{l} = \frac{\alpha\mu_{L-1}}{L} + \frac{L-1}{L}\mu_{L-1} =  \mu_{L-1}\left(1 + \frac{\alpha-1}{L}\right) \\
&= \mu_{l_k+1}\prod_{l=l_k+2}^L \left(1 + \frac{\alpha-1}{l} \right) = \frac{\alpha}{l_k+1}\mu_{l_k}\prod_{l=l_k+2}^L \left(1 + \frac{\alpha-1}{l}\right) \\
&=  \frac{\alpha}{l_k+1}\mu_{0}\prod_{\substack{l= 1 \\ l\neq l_k +1}}^L \left(1 + \frac{\alpha-1}{l} \right) \sim \frac{\alpha}{l_k+1}\prod_{l=1}^L \left(1 + \frac{\alpha-1}{l} \right)
\end{aligned}
\end{equation}
and so:
\begin{equation}
\begin{aligned}
\E[\mathcal{G}(x,x)]^2 &= \left(\sum_{l_k=1}^L \mu_L\right)^2 = \left(\sum_{l_k=1}^L\frac{\alpha}{l_k+1}\right)^2\prod_{l=1}^L \left(1 + \frac{\alpha-1}{l}\right)^2 
\sim \alpha^2\log(L)^2 \prod_{l=1}^L \left(1 + \frac{\alpha-1}{l}\right)^2
\end{aligned}
\end{equation}
For the fourth moment:
\begin{equation}
\begin{aligned}
&\mathbb{E}\left[\|q^L\|_2^4 \right] \\
=& \frac{4\alpha^2}{n^2L^2}\mathbb{E}\left[\left(\sum_{l=l_k}^{L-1}\left({q^{l}}^{\top} W^{L,l}\right)Z^L\sum_{l=l_k}^{L-1}\left({q^{l}}^{\top}W^{L,l}\right)\right)^2 \right] \\
=& \frac{4\alpha^2}{n^2L^2}\mathbb{E}\left[\left(\sum_{l_1=l_k}^{L-1}\left({q^{l_1}}^{\top}W^{L,l_1}\right)Z^L\sum_{l_2=l_k}^{L-1}\left(q^{l_2\top}W^{L,l_2}\right)\sum_{l_3=l_k}^{L-1} \left({q^{l_3}}^{\top}W^{L,l_3}\right)Z^L\sum_{l_4=l_k}^{L-1}\left({q^{l_4}}^{\top}W^{L,l_4}\right)\right)\right]
\end{aligned}
\end{equation}
We denote:
\begin{equation}
C_{l,l'} = \mathbb{E}\left[\|q^{l}\|_2^2 \cdot \|q^{l'}\|_2^2 \right]
\end{equation}
Using the results from the vanilla architecture, we have:
\begin{equation}\label{c}
C_{L,L} = \frac{\alpha^2(n+5)}{nL^2}\sum_{l_1,l_2=l_k}^{L-1}C_{l_1,l_2}
\end{equation}
From Eq.~\ref{c}, it also holds that:
\begin{equation}\label{c2}
   \sum_{l_1,l_2=l_k}^{L-2}C_{l_1,l_2} = \frac{n(L-1)^2}{\alpha^2(n+5)} \cdot C_{L-1,L-1}
\end{equation}
It then follows:
\begin{equation}
\begin{aligned}
\mathbb{E}\left[\|q^L\|_2^4\right] =& C_{L,L} \\
=& \frac{\alpha^2(1+5/n)}{L^2}\sum_{l_1,l_2=l_k}^{L-1}C_{l_1l_2}\\
=&\frac{\alpha^2(1+5/n)}{L^2}\left(C_{L-1,L-1} + \sum_{l_1,l_2=l_k}^{L-2}C_{l_1l_2} + 2\sum_{l=l_k}^{L-2}C_{L-1,l} \right)\\
=&\frac{\alpha^2(1+5/n)}{L^2}\left(C_{L-1,L-1} +  \frac{(L-1)^2n}{\alpha^2(n+5)}C_{L-1,L-1} + 2\sum_{l=l_k}^{L-2}C_{L-1,l} \right)
\end{aligned}
\end{equation}
The following also holds for all $l_1 > l_2 \geq l_k$:
\begin{equation}
C_{l_1,l_2} = \frac{\alpha}{nl_1}\mathbb{E}\left[(\sum_{l=l_k}^{l_1-1}q^{l\top}W^{l_1,l}Z^{l_1})^2\|q^{l_2}\|_2^2\right]
 = \frac{\alpha}{l_1}\sum_{l=l_k}^{l_1-1}C_{l,l_2}
\end{equation}
and so:
\begin{equation}
\begin{aligned}
C_{L,L} &= \frac{\alpha^2(n+5)}{nL^2}\left(C_{L-1,L-1} +  \frac{(L-1)^2n}{\alpha^2(n+5)}C_{L-1,L-1} + \frac{2\alpha}{L-1}\sum_{l_1=l_k}^{L-2}\sum_{l_2=l_k}^{L-2}C_{l_1,l_2} \right)\\
&=\frac{\alpha^2(n+5)}{nL^2}\left(C_{L-1,L-1} +  \frac{(L-1)^2n}{\alpha^2(n+5)}C_{L-1,L-1}
+\frac{2n(L-1)}{\alpha(n+5)}C_{L-1,L-1}  \right)\\
&=\frac{\alpha^2(n+5)}{nL^2}C_{L-1,L-1}\left(1
+\frac{(L-1)^2n}{\alpha^2(n+5)} +\frac{2n(L-1)}{\alpha(n+5)}\right)\\
&= C_{L-1,L-1}\left(\left({1 + \frac{\alpha - 1}{L}}\right)^2 + \frac{5\alpha^2}{nL^2}\right)
\end{aligned}
\end{equation}
Recursively, we have:
\begin{equation}
\begin{aligned}
C_{L,L} = C_{l_k+1,l_k+1}\prod_{l=l_k+2}^L \left(\left({1 + \frac{\alpha - 1}{l}}\right)^2 + \frac{5\alpha^2}{nl^2}\right)
\end{aligned}
\end{equation}
For the reduced architecture, the transition from $q^{l_k}$ to $q^{l_k+1}$ is a vanilla ReLU block, and so using the result from the vanilla architecture:
\begin{equation}
\begin{aligned}
C_{L,L} &= C_{l_k,l_k}\frac{\alpha^2(n+5)}{n(l_k+1)^2}\prod_{l=l_k+2}^L \left(\left({1 + \frac{\alpha - 1}{l}}\right)^2 + \frac{5\alpha^2}{nl^2}\right)\\
&= \frac{\alpha^2(n+5)}{n(l_k+1)^2}\prod_{l\neq l_k+1}\left(\left({1 + \frac{\alpha - 1}{l}}\right)^2 + \frac{5\alpha^2}{nl^2}\right)\\
&\sim \frac{\alpha^2(n+5)}{n(l_k+1)^2}\prod_{l=1}^L\left(\left({1 + \frac{\alpha - 1}{l}}\right)^2 + \frac{5\alpha^2}{nl^2}\right) 
\end{aligned}
\end{equation}
where we assigned $C_{0,0} = 1$. It follows:
\begin{equation}
\begin{aligned}
\mathbb{E}[\mathcal{G}(x,x)^2] &\lessapprox \sum_{\bu,\bv}\sqrt{\mathbb{E}\left[\|y_{(\bu)}^L\|^4_2\right]\cdot\mathbb{E}\left[\|y_{(\bv)}^L\|^2_2\right]}\\
&\sim \left(\sum_{l_k=1}^L\frac{1}{l_k+1}\right)^2\cdot \frac{\alpha^2(n+5)}{n}\cdot \prod_{l=1}^L\left(\left({1 + \frac{\alpha - 1}{l}}\right)^2 + \frac{5\alpha^2}{nl^2}\right)\\
&\sim \log(L)^2\cdot \frac{\alpha^2(n+5)}{n}\cdot \prod_{l=1}^L\left(\left({1 + \frac{\alpha - 1}{l}}\right)^2 + \frac{5\alpha^2}{nl^2}\right)
\end{aligned}
\end{equation}
Similarly, we have:
\begin{equation}
\begin{aligned}
\mathbb{E}[\mathcal{G}(x,x)^2] 
&\gtrapprox \sum_{l_k}\mathbb{E}[\|J^{\bk}\|^4_2] \\
&= \sum_{l_k=1}^L
\frac{\alpha^2(n+5)}{n(l_k+1)^2}\cdot \prod_{l=1}^L\left(\left({1 + \frac{\alpha - 1}{l}}\right)^2 + \frac{5\alpha^2}{nl^2}\right) \\
&\sim \frac{\alpha^2(n+5)}{nL}\cdot \prod_{l=1}^L\left(\left({1 + \frac{\alpha - 1}{l}}\right)^2 + \frac{5\alpha^2}{nl^2}\right)
\end{aligned}
\end{equation}
This yields that:
\begin{equation}
\begin{aligned}
\frac{\mathbb{E}[\mathcal{G}(x,x)^2]}{\E[\mathcal{G}(x,x)]^2} &\lessapprox \frac{\frac{n+5}{n}\cdot \prod_{l=1}^L\left(\left({1 + \frac{\alpha - 1}{l}}\right)^2 + \frac{5\alpha^2}{nl^2}\right)}{\prod_{l=1}^L \left(1 + \frac{\alpha-1}{l}\right)^2}\\
&= \frac{n+5}{n}\cdot \prod_{l=1}^L\left(1 + \frac{5\alpha^2}{n(l+\alpha-1)^2}\right)\\
&\sim \exp\left[\sum_{l=1}^L\frac{5\alpha^2}{n(l+\alpha-1)^2}\right]\cdot \left(1 + \mathcal{O}(1/n)\right)\\
&\sim \exp\left[C/n\right]\cdot \left(1 + \mathcal{O}(1/n)\right)
\end{aligned}
\end{equation}
For the lower bound, we have:
\begin{equation}
\begin{aligned}
\frac{\mathbb{E}[\mathcal{G}(x,x)^2]}{\E[\mathcal{G}(x,x)]^2} &\gtrapprox \frac{\frac{n+5}{n} \cdot \prod_{l=1}^L\left(\left({1 + \frac{\alpha - 1}{l}}\right)^2 + \frac{5\alpha^2}{nl^2}\right)}{L\log(L)^2\prod_{l=1}^L \left(1 + \frac{\alpha-1}{l}\right)^2}\\
&\sim \frac{1}{L\log(L)^2}\cdot \exp\left[C/n\right]\cdot \left(1 + \mathcal{O}(1/n)\right)
\end{aligned}
\end{equation}
Since $\mathbb{E}[\mathcal{G}(x,x)^2]>\E[\mathcal{G}(x,x)]^2$, the lower bound is given by:
\begin{equation}
\frac{\mathbb{E}[\mathcal{G}(x,x)^2]}{\E[\mathcal{G}(x,x)]^2} \gtrapprox \max \left[1,\frac{1}{L\log(L)^2}\cdot \exp\left[C/n\right]\cdot\left(1 + \mathcal{O}(1/n)\right)\right]
\end{equation}
\end{proof}

\end{document}


\appendix

\onecolumn
\section{Appendix}
We dedicate the appendix to proving all results presented in the paper. The proofs of Theorems 1 and 2 (the form of the NTK for ResNet and DenseNet models) are given for sequentially pushing all width parameters to infinity, as it is much easier to prove. The results, however, hold regardless of the order in which the limits are considered. 
\begin{theorem}\label{res_ntk}
Given a depth $L$ ResNet, with positive initialization constants $\{\alpha_l\}_{l=1}^L$, it holds at initialization that $\mathcal{G}(x,x')$ converges (in law) to $K_L^R(x,x')$ as $n_0,n_{0,1}...n_{1,1}...n_L \rightarrow \infty$, such that:
\[
K_L^R(x,x') 
=
K_{L-1}^R(x,x')\Big(\alpha_L\prod_{h=1}^{m-1}\dot{\Sigma}^{L-1,h}(x,x') + 1 \Big) + \alpha_L\sum_{h=1}^m\Big( \Sigma^{L,h-1}(x,x')\prod_{h'=h}^{m-1}\dot{\Sigma}^{L,h'}(x,x')\Big)
\]
where $K_1^R(x,x') = \alpha_1\sum_{h=1}^m\Big( \Sigma^{0,h-1}(x,x')\prod_{h'=h}^{m-1}\dot{\Sigma}^{0,h'}(x,x')\Big)$, and:
\[
\dot{\Sigma}^{l,h}(x,x') = 
\begin{cases} 
  2\E_{u,v \sim \mathcal{N}(0,\Sigma^{l,h-1}(x,x'))}[\dot{\phi}(u)\dot{\phi}(v)]  & 1< h \\
  2\E_{u,v \sim \mathcal{N}(0,\Lambda^l(x,x'))}[\dot{\phi}(u)\dot{\phi}(v)]  & h=1 
  \end{cases}
\]
where $\dot{\phi(u)}$ denotes the derivative of $\phi$ by $u$.

\begin{proof}
Recall that we need to compute $K_L(x,x')^R = \lim_{n_0...n_L\rightarrow \infty}\sum_{l=1}^L \sum_{h=1}^m\langle \frac{\partial f(x,W)}{\partial W^{l,h}},\frac{\partial f(x,W)}{\partial W^{l,h}}\rangle$.\\
The derivative $\frac{\partial f(x,W)}{\partial W^{l,h}}$ can be written in a compact matrix form:
\[
\forall_{0<l\leq L},~\frac{\partial f(x,W)}{\partial W^{l,h}} =
\begin{cases} 
      \sqrt{\frac{a_l2^{m-h}}{\prod_{h'=h-1}^{m-1}n_{l-1,h'}}}q^{l-1,h-1}(x)\Big(C_h^l(x)A^{l+1}(x)\Big)^\top & 1<h\leq m  \\
      \sqrt{\frac{a_l2^{m-h}}{\prod_{h'=h-1}^{m-1}n_{l-1,h'}}}y^{l-1}(x)\Big(C_1^l(x)A^{l+1}(x)\Big)^\top & h=1 \end{cases}
\]
where:
\[
C_h^l(x) = \begin{cases}
\prod_{h'=h}^{m-1}\Big(\Dot{Z}^{l,h'}(x)W^{l,h'+1}\Big)& 1<h<m\\
I & else
\end{cases}
\]
and:
\[
A^{l}(x) = 
\begin{cases} 
      \Big(\sqrt{\frac{\alpha_l2^{m-1}}{\prod_{h'=0}^{m-1}n_{l-1,h'}}}W^{l,1}C_1^l(x) + I^l\Big)A^{l+1}(x) & 0<l\leq L  \\
      \frac{1}{\sqrt{n_L}}W^f & l = L  \end{cases}
\]

$Z^{l,h}$ is a diagonal matrix holding the binary activation variables of layer $h$ in residual branch $l$ in its diagonal,  
and $I^l =  [I_{n_l}, \textbf{0}_{n_{l-1} - n_l}]^\top \in \mathbb{R}^{n_{l-1} \times n_l}$ is a concatenation of an $n_l$ dimensional identity matrix, and a zero matrix $\textbf{0}_{n_{l-1} - n_l}\in \mathbb{R}^{n_l\times n_{l-1} - n_l}$.\\
It follows:
\[
&&\forall_{1<h\leq L}~\Big\langle\frac{\partial f(x,W)}{\partial W^{l,h}},\frac{\partial f(x',W)}{\partial W^{l,h}}\Big\rangle \\
&&= \frac{a_l2^{m-h}}{\prod_{h'=h-1}^{m-1}n_{l-1,h'}}\Big\langle q^{l-1,h-1}(x)\Big(C_h^l(x)A^{l+1}(x)\Big)^\top,q^{l-1,h-1}(x')\Big(C_h^l(x')A^{l+1}(x')\Big)^\top\Big)  \Big\rangle \\
&&=\frac{a_l2^{m-h}}{\prod_{h'=h-1}^{m-1}n_{l-1,h'}} \Big\langle q^{l-1,h-1}(x),q^{l-1,h-1}(x')\Big\rangle A^{l+1}(x)^\top C_h^l(x)^\top C_h^l(x') A^{l+1}(x')
\]

In the following, we take the limits $n_0,n_{0,1}...n_{1,1}...n_L \rightarrow \infty$ consecutively.
After taking the limits $n_0...n_{l-1,h-1} \rightarrow \infty$, it holds that $\frac{1}{n_{l-1,h-1}}\Big\langle q^{l-1,h-1}(x),q^{l-1,h-1}(x')\Big\rangle \rightarrow \Sigma^{l-1,h-1}(x,x')$.\\
We are left with:
\[
&&\lim_{n_0...n_{l-1,h-1}}\Big\langle\frac{\partial f(x,W)}{\partial W^{l,h}},\frac{\partial f(x',W)}{\partial W^{l,h}}\Big\rangle\\
&&= \frac{a_l2^{m-h}}{\prod_{h'=h}^{m-1}n_{l-1,h'}} \Sigma^{l-1,h-1}(x,x') A^{l+1}(x)^\top C_h^l(x)^\top C_h^l(x') A^{l+1}(x')\\
&&=\frac{a_l2^{m-h}}{\prod_{h'=h}^{m-1}n_{l-1,h'}} \Sigma^{l-1,h-1}(x,x') A^{l+1}(x)^\top C_{h+1}^l(x)W^{l,h+1\top}\Dot{Z}^{l,h}(x)\Dot{Z}^{l,h}(x')W^{l,h+1} C_{h+1}^l(x') A^{l+1}(x')
\]
Taking $n_{l-1,h} \rightarrow \infty$, it holds that $\frac{2}{n_{l-1,h}}W^{l,h+1\top}\Dot{Z}^{l,h}(x)\Dot{Z}^{l,h}(x')W^{l,h+1} \rightarrow \dot{\Sigma}^{l-1,h}(x,x')I$.\\
And so:
\[
&&\lim_{n_0...n_{l-1,h} \rightarrow \infty}\Big\langle\frac{\partial f(x,W)}{\partial W^{l,h}},\frac{\partial f(x',W)}{\partial W^{l,h}}\Big\rangle \\
&&= \frac{a_l2^{m-h-1}}{\prod_{h'=h+1}^{m-1}n_{l-1,h'}}\Sigma^{l-1,h-1}(x,x')\dot{\Sigma}^{l-1,h}(x,x') A^{l+1}(x)^\top C_{h+1}^l(x)C_{h+1}^l(x') A^{l+1}(x')\\
&&= a_l\Sigma^{l-1,h-1}(x,x')\prod_{h'=h}^{m-1}\dot{\Sigma}^{l-1,h'}(x,x')A^{l+1}(x)^\top A^{l+1}(x')
\]

Expanding $A^{l+1}(x)^\top A^{l+1}(x')$:
\[
&&\Big\langle 
A^{l+1}(x), A^{l+1}(x')\Big\rangle 
= A^{l+2}(x)^\top\Big(\sqrt{\frac{a_{l+1}2^{m-1}}{\prod_{h'=0}^{m-1}n_{l,h'}}}W^{l+1,1}C_1^{l+1}(x) + I^{l+1}\Big)^\top...\\
&&\Big(\sqrt{\frac{a_{l+1}2^{m-1}}{\prod_{h'=0}^{m-1}n_{l,h'}}}W^{l+1,1}C_1^{l+1}(x') + I^{l+1}\Big)A^{l+2}(x')\\
&&= \sqrt{\frac{a_{l+1}2^{m-1}}{\prod_{h'=0}^{m-1}n_{l,h'}}}A^{l+2}(x)^\top C_1^{l+1}(x)^\top W^{l+1,1\top}I^{l+1}A^{l+2}(x') \\
&&+\sqrt{\frac{a_{l+1}2^{m-1}}{\prod_{h'=0}^{m-1}n_{l,h'}}}A^{l+2}(x')^\top C_1^{l+1}(x')^\top W^{l+1,1\top}I^{l+1}A^{l+2}(x)\\
&&+\frac{a_{l+1}2^{m-1}}{\prod_{h'=0}^{m-1}n_{l,h'}}A^{l+2}(x)^\top C_1^{l+1}(x)^\top W^{l+1,1\top} W^{l+1,1}C_1^{l+1}(x') A^{l+2}(x')\\
&&+ \Big\langle 
A^{l+2}(x), A^{l+2}(x')\Big\rangle \\
&&=T_1+T_2+T_3+T_4
\]

Note that $\sqrt{\frac{1}{n_l}}\|W^{l+1,1\top}I^{l+1\top}\| \rightarrow 0$ as $n_l$ tends to infinity, and so $T_1,T_2 \rightarrow 0$. Taking $n_{l+1} \rightarrow \infty$, we have $\frac{1}{n_l}W^{l+1,1\top} W^{l+1,1} \rightarrow I$. Taking $n_0...n_{l,m} \rightarrow \infty$, we have that $T_3 \rightarrow \prod_{h'=1}^{m-1}\dot{\Sigma}^{l,h'}(x,x')\Big\langle 
A^{l+2}(x), A^{l+2}(x')\Big\rangle$. \\
And so it follows:
\[
\lim_{n_0...n_L \rightarrow \infty}\Big\langle 
A^{l+1}(x), A^{l+1}(x')\Big\rangle  = \lim_{n_0...n_L \rightarrow \infty}\Big\langle 
A^{l+2}(x), A^{l+2}(x')\Big\rangle \Big(\alpha_{l+1}\prod_{h'=1}^{m-1}\dot{\Sigma}^{l,h'}(x,x') + I \Big)\\
=\prod_{l'=l}^{L-1}\Big(a_{l'+1}\prod_{h'=1}^{m-1}\dot{\Sigma}^{l',h'}(x,x') + I \Big)\lim_{n_L\rightarrow \infty}\frac{1}{n_L}W^{f\top}W^f = \prod_{l'=l}^{L-1}\Big(\alpha_{l'+1}\prod_{h'=1}^{m-1}\dot{\Sigma}^{l',h'}(x,x') + I \Big)
\]
And finally:
\[
&&\lim_{n_0...n_L \rightarrow \infty}\Big\langle\frac{\partial f(x,W)}{\partial W^{l,h}},\frac{\partial f(x',W)}{\partial W^{l,h}}\Big\rangle \\
&&= \alpha_l \Sigma^{l-1,h-1}(x,x')\prod_{h'=h}^{m-1}\dot{\Sigma}^{l-1,h'}(x,x')\prod_{l'=l}^{L-1}\Big(\alpha_{l'+1}\prod_{h'=1}^{m-1}\dot{\Sigma}^{l',h'}(x,x') + I \Big)
\]
The derivation for the case of $h=1$ gives an identical result, and so summing over all the weights:
\[
&&K_L^R(x,x') = \lim_{n_0...n_L \rightarrow \infty}~~\sum_{l=1}^L \sum_{h=1}^m\langle \frac{\partial f(x,W)}{\partial W^{l,h}},\frac{\partial f(x,W)}{\partial W^{l,h}}\rangle\\
 &&= \sum_{l=1}^L\sum_{h=1}^m\alpha_l\Big( \Sigma^{l-1,h-1}(x,x')\prod_{h'=h}^{m-1}\dot{\Sigma}^{l-1,h'}(x,x')\Big)\prod_{l'=l}^{L-1}\Big(\alpha_{l'+1}\prod_{h'=1}^{m-1}\dot{\Sigma}^{l',h'}(x,x') + 1 \Big)\\
  &&= K_{L-1}^R(x,x')\Big(\alpha_L\prod_{h=1}^{m-1}\dot{\Sigma}^{L-1,h}(x,x') + 1 \Big) + \alpha_L\sum_{h=1}^m\Big( \Sigma^{L,h-1}(x,x')\prod_{h'=h}^{m-1}\dot{\Sigma}^{L,h'}(x,x')\Big)
\]
\end{proof}
\end{theorem}

\begin{theorem}\label{dense_ntk}
Given a depth $L$ DenseNet, with positive initialization constant $\alpha$, it holds at initialization that $\mathcal{G}(x,x')$ converges (in law) to $K_L^D(x,x')$ as $n_0, n_0'...n_L \rightarrow \infty$, such that:

\[
K_L^D(x,x') = K_{L-1}^D(x,x')\Big(\frac{\alpha\dot{\Sigma}^{L-1}(x,x')+L-1}{L} + \Big) +\frac{\alpha\Sigma^{L-1}(x,x')}{L}  
\]
where:
\[
\dot{\Sigma}^{l}(x,x') = 2\E_{u,v \sim \mathcal{N}(0,\Lambda^l(x,x')}[\dot{\phi}(u)\dot{\phi}(v)]
\]

\begin{proof}
Similarly to the ResNet case, the derivative $\frac{\partial f(x,W)}{\partial W^{l,h}}$ can be written in a compact matrix form:
\[
\forall_{0\leq h<l\leq L},~\frac{\partial f(x,W)}{\partial W^{l,h}} =
\sqrt{\frac{\alpha}{ln_{l-1}'}}[q^h(x)]_{n_{l-1}'}\Big(A^{l+1,l}(x)\Big)^\top
\]
where:
\[
A^{l,h}(x) = 
\begin{cases} 
     \sqrt{\frac{2\alpha}{ln_{l-1}'}}Z^h(x)I^{h,l-1}W^{l,h}A^{l+1,l}(x) + A^{l+1,h}(x) & 0\leq h < l < L  \\
     \sqrt{\frac{2\alpha}{ln_{l-1}'}}Z^h(x)I^{h,l-1}W^{l,h}A^{l+1,l}(x) & l = L\\
       \frac{1}{\sqrt{n_L}}W^f & l = L+1 \end{cases}
\]
where $Z^{l}$ is a diagonal matrix holding the binary activation variables of layer $l$ in its diagonal,
and $I^{h,l} =  [I_{n_l'}, \textbf{0}_{n_h - n_l'}]^\top \in \mathbb{R}^{n_h \times n_l'}$ is a concatenation of an $n_h$ dimensional identity matrix, and a zero matrix $\textbf{0}_{n_h - n_l'}\in \mathbb{R}^{n_l'\times n_h - n_l'}$.
We then have:
\begin{equation}\label{eq:dense_eq}
\forall_{0\leq h<l\leq L},~\Big\langle \frac{\partial f(x,W)}{\partial W^{l,h}},\frac{\partial f(x',W)}{\partial W^{l,h}}\Big\rangle = \frac{\alpha}{ln_{l-1}'}\Big\langle [q^h(x)]_{n_`{l-1}'},[q^h(x')]_{n_{l-1}'} \Big\rangle \Big\langle A^{l+1,l}(x),A^{l+1,l}(x') \Big\rangle
\end{equation}
In the following we take the limits $n_0, n_0'...n_L \rightarrow \infty$ consecutively.
After taking the limit $n_0,n_0'...n_{l-1}' \rightarrow \infty$, it holds that $\frac{\alpha}{ln_{l-1}'}\Big\langle [q^h(x)]_{n_{l-1}'},[q^h(x')]_{n_{l-1}'} \Big\rangle \rightarrow \frac{\alpha}{l}\Sigma^h(x,x')$.
We are left with computing the limit of $\Big\langle A^{l+1,l}(x),A^{l+1,l}(x') \Big\rangle$. It follows that:

\[\label{eq:rec}
&&\forall_{0\leq h<l<L},~~\Big\langle A^{l,h}(x),A^{l,h}(x') \Big\rangle = \sqrt{\frac{2\alpha}{ln_{l-1}'}}A^{l+1,h\top}(x)Z^h(x')I^{h,l-1}W^{l,h}A^{l+1,l}(x')\\
&&+\Big\langle A^{l,h}(x),A^{l,h}(x') \Big\rangle = \sqrt{\frac{2\alpha}{ln_{l-1}'}}A^{l+1,h\top}(x')Z^h(x)I^{h,l-1}W^{l,h}A^{l+1,l}(x)\\
&&+\frac{2\alpha}{ln_{l-1}'}A^{l+1,l\top}(x)W^{l,h\top}I^{h,l-1\top}Z^h(x)Z^h(x')I^{h,l-1}W^{l,h}A^{l+1,l}(x')
+A^{l+1,h\top}(x)A^{l+1,h}(x')\\
&&=T_1+T_2+T_3+T_4
\]

Expanding $T_1$:
\[
&&T_1 = \sqrt{\frac{2\alpha}{ln_{l-1}'}}A^{l+1,h\top}(x)Z^h(x')I^{h,l-1}W^{l,h}A^{l+1,l}(x')\\
&&=\Big(\sqrt{\frac{2\alpha}{(l+1)n_l'}}Z^h(x)I^{h,l}W^{l+1,h}A^{l+2,l+1}(x) + A^{l+2,h}(x)\Big)^\top\sqrt{\frac{2\alpha}{ln_{l-1}'}}Z^h(x')I^{h,l-1}W^{l,h}A^{l+1,l}(x')\\
&&=\frac{2\alpha}{\sqrt{l(l+1)n_{l-1}'n_l'}}A^{l+2,l+1\top}(x)W^{l+1,h\top}I^{h,l\top}Z^h(x)Z^h(x')I^{h,l-1}W^{l,h}A^{l+1,l}(x')\\
&&+ A^{l+2,h}(x)^\top \sqrt{\frac{2\alpha}{ln_{l-1}'}}Z^h(x')I^{h,l-1}W^{l,h}A^{l+1,l}(x')
\]

Looking at the first term of the expansion, notice that after taking the limit $n_0...n_{l-1}' \rightarrow \infty$, it holds that $\frac{1}{\sqrt{n_{l-1}'}}\|I^{h,l\top}Z^h(x)Z^h(x')I^{h,l-1}W^{l,h}\| \rightarrow 0$, and so we are left with the second term:
\[
\lim_{n_0...n_{l-1}' \rightarrow \infty} T_1 = \lim_{n_0...n_{l-1}' \rightarrow \infty}A^{l+2,h}(x)^\top \sqrt{\frac{2\alpha}{ln_{l-1}'}}Z^h(x')I^{h,l-1}W^{l,h}A^{l+1,l}(x')
\]
Recursively expanding $A^{l+2,h}(x)$, we get similar terms that vanish in the limit, therefore:
\[
\lim_{n_0...n_{l-1}' \rightarrow \infty} T_1 = \lim_{n_0...n_{l-1}' \rightarrow \infty}A^{L,h}(x)^\top \sqrt{\frac{2\alpha}{ln_{l-1}'}}Z^h(x')I^{h,l-1}W^{l,h}A^{l+1,l}(x')
 \rightarrow 0
\]
And so $\lim_{n_0...n_{l-1}' \rightarrow \infty} T_1,T_2 \rightarrow 0$. We are left with evaluating the limit of $T_3,T_4$. Expanding $T_3$:
\[
\lim_{n_0...n_{l-1}' \rightarrow \infty} T_3 = \lim_{n_0...n_{l-1}' \rightarrow \infty}\frac{2\alpha}{ln_{l-1}'}A^{l+1,l\top}(x)W^{l,h\top}I^{h,l-1\top}Z^h(x)Z^h(x')I^{h,l-1}W^{l,h}A^{l+1,l}(x')
\]
Note that it holds that $\lim_{n_0...n_{l-1}' \rightarrow \infty}\frac{2\alpha}{n_{l-1}'}W^{l,h\top}I^{h,l-1\top}Z^h(x)Z^h(x')I^{h,l-1}W^{l,h} \rightarrow \alpha\dot{\Sigma}^h(x,x')I$, and so:
\[
\lim_{n_0...n_{l-1}' \rightarrow \infty} T_3 = \frac{\alpha\dot{\Sigma}^h(x,x')}{l}A^{l+1,l\top}(x)A^{l+1,l}(x')
\]
Plugging back into Eq.~\ref{eq:rec}:
\[
&&\lim_{n_0...n_{l-1}' \rightarrow \infty}\Big\langle A^{l,h}(x),A^{l,h}(x') \Big\rangle = \lim_{n_0...n_{l-1}' \rightarrow \infty} T_3+T_4 \\
&&=\frac{\alpha\dot{\Sigma}^h(x,x')}{l}\Big\langle A^{l+1,l}(x),A^{l+1,l}(x')\Big\rangle + \Big\langle A^{l+1,h}(x),A^{l+1,h}(x')\Big\rangle 
\]
and:
\[
&&\lim_{n_0...n_{L} \rightarrow \infty}\Big\langle A^{L,h}(x),A^{L,h}(x') \Big\rangle = \lim_{n_0...n_L \rightarrow \infty} T_3\\
&&=\lim_{n_0...n_{L} \rightarrow \infty}\frac{\alpha\dot{\Sigma}^h(x,x')}{L}\Big\langle A^{L+1,l}(x),A^{L+1,l}(x')\Big\rangle = \frac{\alpha\dot{\Sigma}^h(x,x')}{L}
\]
Denoting by $m^{l,h}(x,x') = \lim_{n_0...n_{L} \rightarrow \infty}\Big\langle A^{l,h}(x),A^{l,h}(x') \Big\rangle$, it holds that:
\[
\lim_{n_0...n_{L} \rightarrow \infty}m_{L+1,h} = \frac{1}{n_L}W^{f\top}W^f = 1
\]
Plugging into Eq.~\ref{eq:dense_eq}, we arrive at:
\begin{equation}\label{dense_final}
\lim_{n_0...n_{L} \rightarrow \infty}\Big\langle,\frac{\partial f(x,W)}{\partial W^{l,h}},\frac{\partial f(x',W)}{\partial W^{l,h}}\Big\rangle = \frac{\alpha}{l}\Sigma^h(x,x')m^{l+1,l}(x,x')
\end{equation}
with the following recursion:
\begin{equation}
m_{l,h}(x,x') = \begin{cases} 
     \frac{\alpha\dot{\Sigma}^h(x,x')}{l}m^{l+1,l}(x,x') + m^{l+1,h}(x,x') & 0\leq h < l < L  \\
     \frac{\alpha\dot{\Sigma}^h(x,x')}{L} & 0\leq h < L,l=L\\
     1 & else\end{cases}
\end{equation}
Summing over all the weights, we have:
\begin{equation}\label{eq:dense_rec}
K_L^D(x,x') = \sum_{l=1}^L\frac{\sum_{h=0}^{l-1}\alpha\Sigma^h(x,x')}{l}m^{l+1,l}(x,x')\\
= K_{L-1}^D(x,x')\Big(\frac{\alpha\dot{\Sigma}^{L-1}(x,x')}{L} + \frac{L-1}{L}\Big) +\frac{\alpha\Sigma^{L-1}(x,x')}{L}  
\end{equation}
\end{proof}
\end{theorem}

\begin{lemma}\label{eq:finite}
Given a $\beta$-smooth function $f(u) \in [0,1]$, a learning rate $0 \leq \alpha < \frac{1}{\beta}$ and a constant $0<c<1$. Let $\{f(u_t)\}_{t=0}^\infty$ be the series of outputs given by applying gradient descent, such that $f(u_{t+1}) = f(u_t) - \alpha \nabla f(u_t)$.  Assuming the following holds for any point $u$:
\begin{equation}\label{grad_cond}
\|\nabla f(u)\|\geq \sqrt{2c\beta f(u)}
\end{equation}
it then holds that $f(u_\infty) = 0$.
\begin{proof}
Smoothness of $f$ implies that for any $u,v$ it holds $\|\nabla f(u) - \nabla f(v)\|\leq \frac{\beta}{2}\|u - v\|$, and the following hold:
\begin{equation}\label{lr}
f(v) \leq f(u)  + \langle \nabla f(u),v - u \rangle + \frac{\beta}{2}\|v - u\|^2
\end{equation}
Setting $v = u -\frac{1}{\beta}\nabla f(u)$, we have (using the fact that $f\geq 0$):
\[
\|\nabla f(u)\| \leq \sqrt{2\beta f(u)}
\]
together with the assumption that $\|\nabla f(u)\| \geq \sqrt{2c\beta f(u)}$, we have:
\[
\sqrt{2c\beta f(u)} \leq \|\nabla f(u)\| \leq \sqrt{2\beta f(u)}
\]
Setting $v = u -\alpha\nabla f(u)$ in Eq.~\ref{lr}:
\[
f(v)\leq f(u) - \alpha\|\nabla f(u)\|^2(1 - \frac{\alpha\beta}{2} ) \leq f(u) - 2c\alpha\beta f(u)(1 - \frac{\alpha\beta}{2})\\
=f(u)\Big(1 - 2c\alpha\beta (1 - \frac{\alpha \beta}{2})\Big)
\]
For $0\leq \alpha < \frac{1}{\beta}$, it holds that $0 \leq \Big(1 - 2c\alpha\beta (1 - \frac{\alpha \beta}{2})\Big) <1$, and so
It follows:
\[
f(u_\infty) \leq f(u_0)\Big(1 - 2c \alpha \beta(1 - \frac{\alpha \beta}{2})\Big)^\infty = 0
\]
\end{proof}
\end{lemma}

We make use of the following propositions and definitions to aid in the proofs of Theorems 3 and 4.
\begin{proposition}\label{p1}
Given a random vector $w = [w_1...w_n]$ such that each component is identically and symmetrically distributed i.i.d random variable with moments $\E(w_i^m) =c_m$ ($c_0 = 1,c_1 = 0$), a set of non negative integers $m_1...m_l$ such that $\sum_{i=1}^lm_i$ is even, and a random binary variable $z\in \{0,1\}$ such that $(z|w) = 1-(z|-w)$, then it holds that:
\[ \E[\prod_{i=1}^{l}w_{i}^{m_i} z] =\frac{\prod_{i=1}^lc_{m_i}}{2}
\]
\begin{proof}
We have:
\[
\prod_{i=1}^lc_{m_i} &=& \int_w \prod_{i=1}^lw_i^{m_i}p(w)dw  =  \int_{w|z=1} \prod_{i=1}^lw_i^{m_i}p(w)dw + \int_{w|z=0}\prod_{i=1}^lw_i^{m_i}p(w)dw\\
&=&\int_{w|z=1} \prod_{i=1}^lw_i^{m_i}p(w)dw + \int_{w|z=1}\prod_{i=1}^l(-w_i)^{m_i}p(w)dw\\
&=&\int_{w} \prod_{i=1}^lw_i^{m_i}zp(w)dw + \int_{w}\prod_{i=1}^l(-w_i)^{m_i}zp(w)dw
\]
For even $\sum_{i=1}^l m_i$, it follows that:
\[
&& \int_{w}\prod_{i=1}^l(-w_i)^{m_i}zp(w)dw = \int_{w}\prod_{i=1}^lw_i^{m_i}zp(w)dw\\
&& \prod_{i=1}^lc_{m_i} = 2\int_{w} \prod_{i=1}^lw_i^{m_i}zp(w)dw
\]
yielding: 
\[
\frac{\prod_{i=1}^lc_{m_i}}{2} = \int_{w} \prod_{i=1}^lw_i^{m_i}zp(w)dw = \E[\prod_{i=1}^{l}w_{i}^{m_i} z]
\]
\end{proof}
\end{proposition}

\begin{proposition}\label{p2}
Given a random vector $w = [w_1...w_n]$ such that each component is identically and symmetrically distributed i.i.d random variable with moments $\E(w_i^m) =c_m$ ($c_0 = 1,c_1 = 0$), two sets of non negative integers $m_1...m_l$, $n_1...n_l$, such that $\sum_{i=1}^lm_i$ , $\sum_{i=1}^ln_i$ are even , $\forall_i,m_i\geq n_i$, and a random binary variable $z\in \{0,1\}$, such that $(z|w) = 1-(z|-w)$, then it holds that:
\[ \E[\frac{1}{w_{i}^{n_i}}\prod_{i=1}^{l}w_{i}^{m_i} z] =\frac{\prod_{i=1}^lc_{m_i - n_i}}{2}
\]
\begin{proof}
This is trivially true from proposition 1 since $\sum_i (m_i - n_i)$ is even.
\end{proof}
\end{proposition}

\begin{definition}
\textbf{ResNet path parametrization}: A path from input to output $\gamma$ in a ResNet, with two layer residual branches (m=2), defines a product of weights along the path denoted by: 
\[
P_{\gamma} = \prod_{l=0}^{L+1} p_{\gamma,l}
\]
where:
\[
p_{\gamma,l} = \begin{cases} 
      1 & l \notin \gamma\\
      w_{\gamma,l}^1z_{\gamma,l}w_{\gamma,l}^2 & l \in \gamma, 0<l\leq L\\
      w_{\gamma,l} & l=\{0,L+1\} 
   \end{cases}
\]

Here, $w_{\gamma,l}^1,w_{\gamma,l}^2$ are weights associated with residual branch $l$, $w_{\gamma,0},w_{\gamma,L+1}$ belong to the first and last linear projection matrices $W^s,W^f$, and $z_{\gamma,l}$ is the binary activation variable relevant for weight $w_{\gamma,l}^1$. (Note that $z_{\gamma,l}$ depends on $w_{\gamma,l}^1$, but not on $w_{\gamma,l}^2$ ).
$l \notin \gamma$ indicates if layer $l$ is skipped. 
\end{definition}

\begin{definition}
\textbf{DenseNet path parametrization}: Similarly to the ResNet case, a path $\gamma$ from input in to output in a DensdNet, defines a product of weights along the path denoted by:
\[
P_{\gamma} = \prod_{l=0}^{L+1} p_{\gamma,l}
\]
where:
\[
p_{\gamma,l} = \begin{cases} 
      1 & l \notin \gamma\\
      w_{\gamma,l}z_{\gamma,l} & l \in \gamma, 0<l\leq L\\
      w_{\gamma,l} &  l=\{0,L+1\} 
   \end{cases}
\]

Here, $w_{\gamma,l}$ is a weight associated with layer $l$, $w_{\gamma,0},w_{\gamma,L+1}$ belong to the first and last linear projection matrices $W^s,W^f$, and $z_{\gamma,l}$ is the binary activation variable relevant for weight $w_{\gamma,l}$.
$l \notin \gamma$ indicates if layer $l$ is skipped. 
\end{definition}

\begin{theorem}\label{the:eq}
For any architecture $\mathcal{N}$ described in Sec 4, the following holds at initialization for any non-negative even integer $m$:
\begin{equation}\label{eq:equi}
   \forall_{W^k \in \mathcal{N}},~ \E\big[(f_{(k)})^{m} \big] =  \E\big[(f^{k})^{m} \big]
\end{equation}
\begin{proof}
We present the proof using the DenseNet path parameterization. Extending to ResNet parameterization is trivial and requires no additional arguments.
We aim to show that for an even $m$, and $\forall_{W^k \in \mathcal{N}}$:
\[
\E\big[(f_{(k)})^{m} \big] =  \E\big[(f^{k})^{m} \big]
\]
The output $f^{k}$ can be expressed as follows:
\[
f^k = \sum_{\gamma \in \{\gamma_{k}\}}c_\gamma \prod_{l=0}^{L+1} p_{\gamma,l}
\]
Since the output $f_{(k)}$ is composed of products of weights and activations along the same paths $\gamma \in \{\gamma_{k}\}$ as $f^{k}$ (with different activation variables), we only need to prove the following:\\
For any weight matrix $W^k \in \mathcal{N}$, and a set of $m$ paths $\gamma^1...\gamma^m \in \{\gamma_{k}\}$, it holds that:
\[
&&\E_{\mathcal{N}}[\prod_{i=1}^mP_{\gamma^i}] = \E_{\mathcal{N}_{(k)}}[\prod_{i=1}^mP_{\gamma^i}]
\]
where $\E_{\mathcal{N}}$ stands for expectation using the full architecture, and $\E_{\mathcal{N}_{(k)}}$ stands for expectation using the reduced architecture.

We have that:
\[
&&\E_{\mathcal{N}}[\prod_{i=1}^mP_{\gamma^i}] = \E_{\mathcal{N}}[\prod_{l=0}^{L+1}(\prod_{i=1}^mp_{\gamma^i,l})]
\]
From the linearity of the last layer, it follows that:
\[
\E_{\mathcal{N}}[\prod_{i=1}^mP_{\gamma^i}] = \E_{\mathcal{N}}[\prod_{l=0}^{L}(\prod_{i=1}^mp_{\gamma^i,l})] \E_{\mathcal{N}}[(\prod_{i=1}^mp_{\gamma^i,{L+1}})] = \E_{\mathcal{N}}[\prod_{l=0}^{L}(\prod_{i=1}^mp_{\gamma^i,l})] \E_{\mathcal{N}}[\prod_{i=1}^mw_{\gamma^i,L+1}]
\]
We denote by $\{w_u^{L+1}\}_{u=1}^s,s\leq m$ the set of $s$ unique weights in $\{w_{\gamma^i,L+1}\}_{i=1}^m$, with corresponding multiplicity $\{m_u^{L+1}\}$, such that $\sum_u m_u^{L+1} = m$. It follows that:
\[
\E_{\mathcal{N}}[\prod_{i=1}^mP_{\gamma^i}] = \E_{\mathcal{N}}[\prod_{l=0}^{L}(\prod_{i=1}^mp_{\gamma^i,l})] \E_{\mathcal{N}}[\prod_{u} (w_u^{L+1})^{m_u^{L+1}}] = \E_{\mathcal{N}}[\prod_{l=0}^{L}(\prod_{i=1}^mp_{\gamma^i,l})] \prod_u c_{m_u^{L+1}}
\]
where $c_{m_u^{L+1}}$ is the $m_u^{L+1}$'th moment of a normal distribution.

Since the computations done by all considered architectures form a markov chain, such that the output of any layer depends only on up-stream weights, denoted by $R^{l-1}$, we have that:
\[
\E_{\mathcal{N}}[\prod_{i=1}^mp_{\gamma^i,L}] = \E_\mathcal{N}\Big[\E_{\mathcal{N}}[\prod_{i=1}^mp_{\gamma^i,L}\Big|R^{L-1}]\Big] = \E_\mathcal{N}\Big[\E_{\mathcal{N}}[\prod_{i=1}^mp_{\gamma^i,L}\Big|q^0...q^{L-1}]\Big]
\]
Conditioned on $q^0...q^{L-1}$, the pre-activations $y^L$ are zero mean iid Gaussian variables. In addition, the activations $q^L = 2\phi(y^L)$ are iid distributed. We denote by $\{z_u\}_{u=1}^s$ the set of unique activation variables in the set $\{z_{\gamma^i,L}\}_{i=1}^m$. For each $z_u$, we denote by $\{w_{u,v}^L\}$ the set of unique weights in $\{w_{\gamma^i,L}\}$ multiplying $z_u$, with corresponding multiplicity $m_{u,v}^L$, such that $\sum_{u,v}m_{u,v}^L = m$, and $\sum_{v}m_{u,v}^L = m_u^{L+1}$. Note that, from the symmetry of the normal distribution, it holds that odd moments vanish, and so we only need to consider even $m_u^{L=1}$ for all $u$.  From the independence of the set $\{z_u\}$, the expectation takes a factorized form:
\[
\E_\mathcal{N}\Big[\E_{\mathcal{N}}[\prod_{i=1}^mp_{\gamma^i,L}\Big|q^0...q^{L-1}]\Big] &=& \E_\mathcal{N}\Big[\prod_{u=1}^s\E_{\mathcal{N}}[ z_u\prod_{v}(w_{u,v}^L)^{m_{u,v}^L}\Big|q^0...q^{L-1}]\Big]\\
&=&\E_\mathcal{N}\Big[\prod_{u=1}^s\E_{\mathcal{N}}[ z_u\prod_{v}(w_{u,v}^L)^{m_{u,v}^L}\Big|\sum_{h=0}^{L-1}\|q^h\|>0]\Big]
\]
Note that for both the full and reduced architecture, flipping the sign of all weights in layer $k$, will flip the activation variables (except for a set of zero measure defined by $\sum_{h=0}^{k-1}W^{k,h\top}q^h = 0$, which does not affect the expectation). And so, using Proposition \ref{p1}:
\[
\E_\mathcal{N}\Big[\prod_{u=1}^s\E_{\mathcal{N}}[ z_u\prod_{v}(w_{u,v}^L)^{m_{u,v}^L}\Big|\sum_{h=0}^{l-1}\|q^h\|>0]\Big] = \E[\mathbbm{1}_{\sum_{h=0}^{l-1}\|q^h\|>0}]\prod_{u=1}^s(\frac{\prod_{v}c_{m_{u,v}^L}}{2})
\]
It then follows:
\[\label{factorized}
\E_{\mathcal{N}}[\prod_{l=0}^{L}(\prod_{i=1}^mp_{\gamma^i,l})] &=& \E_{\mathcal{N}}[\mathbbm{1}_{\sum_{h=0}^{L-1}\|q^h\|>0}\prod_{l=0}^{L-1}(\prod_{i=1}^mp_{\gamma^i,l})]\prod_{u=1}^s(\frac{\prod_{v}c_{m_{u,v}^L}}{2})\\
&=& \E_{\mathcal{N}}[\prod_{l=0}^{L-1}(\prod_{i=1}^mp_{\gamma^i,l})]\prod_{u=1}^s(\frac{\prod_{v}c_{m_{u,v}^L}}{2})
\]
Similarly, it holds for the reduced architecture:
\[
\E_{\mathcal{N}_{(k)}}[\prod_{l=0}^{L}(\prod_{i=1}^mp_{\gamma^i,l})] &=&\E_{\mathcal{N}_{(k)}}[\mathbbm{1}_{\sum_{h=0}^{L-1}\|q_{(k)}^h\|>0}\prod_{l=0}^{L-1}(\prod_{i=1}^mp_{\gamma^i,l})]\prod_{u=1}^s(\frac{\prod_{v}c_{m_{u,v}^L}}{2})\\
&=& \E_{\mathcal{N}_{(k)}}[\prod_{l=0}^{L-1}(\prod_{i=1}^mp_{\gamma^i,l})]\prod_{u=1}^s(\frac{\prod_{v}c_{m_{u,v}^L}}{2})
\]
Recursively going through $l = L-1...0$ completes the proof for DenseNet.
\end{proof}
\end{theorem}

\begin{theorem}\label{the:dual}
For any architecture $\mathcal{N}$ described in Sec 4, the following holds at initialization:
\begin{enumerate}\label{eq:du}
  \item $\forall_{W^k \in \mathcal{N}}~~\E\big[\|J^k\|^2 \big] = E\big[(f_{(k)})^2 \big]$
  \item $\forall_{W^k \in \mathcal{N}}~~\frac{E\big[(f_{(k)})^4 \big]}{c_4} \leq \E\big[\|J^k\|^4 \big] \leq E\big[(f_{(k)})^4 \big]$
\end{enumerate}
\begin{proof}
We present the proof using the DenseNet path parameterization. Extending to ResNet parameterization is trivial and requires no additional arguments.
Neglecting scaling coefficients for notational simplicity, for any weight matrix $W^k \in \mathcal{N}$,
the Jacobian $J^{k}$ can be expressed as follows:
\[
\E\Big[\|J^k\|^2\Big] = \E\Big[\|\frac{\partial f^k}{\partial W^k}\|^2\Big] = \sum_{i,j}\E\Big[(\sum_{\gamma|w_{i,j}^k\in \gamma}\frac{1}{w_{i,j}^{k}}P_{\gamma})^2\Big]
\]
where $\sum_{\gamma|w_{i,j}^k\in \gamma}$ represents sum over paths that include weight $w_{i,j}^k$. From applying Eq.~\ref{factorized} recursively, the expectation is factorized as follows:
\[\label{second}
\E\Big[\|\frac{\partial f^k}{\partial W^{k}}\|^2\Big]&=&\sum_{i,j}\sum_{\gamma|w_{i,j}^k\in \gamma}\E \Big[(\frac{1}{w_{i,j}^k}p_{\gamma,k})^2\Big|\sum_{h=0}^{k-1}\|q^h\|>0\Big]\prod_{l\neq k} \E \Big[(p_{\gamma,l})^2\Big|\sum_{h=0}^{l-1}\|q^h\|>0\Big]
\]
Using Propositions \ref{p1} and \ref{p2} , we have that
\[
\forall_{\gamma|w_{i,j}^k\in \gamma},~ \Big[(\frac{1}{w_{i,j}^k}p_{\gamma,k})^2\Big|\sum_{h=0}^{k-1}\|q^h\|>0\Big] &=& \E \Big[(\frac{w_{i,j}^kz_{\gamma,k}}{w_{i,j}^k})^2\Big|\sum_{h=0}^{k-1}\|q^h\|>0\Big] = \frac{\mathbbm{1}_{\sum_{h=0}^{k-1}\|q^h\|>0}}{2}\\
&=& \E \Big[(p_{\gamma,k})^2\Big|\sum_{h=0}^{k-1}\|q^h\|>0\Big]
\]
Inserting into Eq.~\ref{second}, and using Theorem 3 proves the first claim.\\
For the second claim, it follows:
\[
\E\Big[\|J^k\|^4\Big] &=& 
\E\Big[\|\frac{\partial f^{k}}{\partial W^k}\|^2\|\frac{\partial f^{k}}{\partial W^k}\|^2\Big]\\
&=& \sum_{i,j}\sum_{i',j'}\E \Big[(\sum_{\gamma|w_{i,j}^{k}\in \gamma}\frac{1}{w_{i,j}^{k}}P_{\gamma})^2(\sum_{\gamma|w_{i',j'}^{k}\in \gamma}\frac{1}{w_{i',j'}^{k}}P_{\gamma})^2\Big]\\
&=& \sum_{i,i',j,j'}\E \Big[\frac{1}{(w_{i,j}^{k})^2(w_{i',j'}^k)^2}\sum_{\gamma^1,\gamma^2|w_{i,j}^{k}\in \gamma^1,\gamma^2}~~\sum_{\gamma^3,\gamma^4|w_{i',j'}^{k}\in \gamma^3,\gamma^4}P_{\gamma^1}P_{\gamma^2}P_{\gamma^3}P_{\gamma^4}\Big]
\]
From applying Eq.~\ref{factorized} recursively, the expectation is factorized as follows:
\[\label{fourth}
\E\Big[\|J^k\|^4\Big] = \sum_{i,i',j,j'}\sum_{\gamma^1,\gamma^2|w_{i,j}^{k}\in \gamma^1,\gamma^2}~~\sum_{\gamma^3,\gamma^4|w_{i',j'}^{k}\in \gamma^3,\gamma^4}\bigg[\E \Big[\frac{\prod_{h=1}^4p_{\gamma^h,k}}{(w_{i,j}^{k})^2(w_{i',j'}^k)^2}\Big|\sum_{h=0}^{k-1}\|q^h\|>0\Big]...\\
\prod_{l\neq k} \E \Big[\prod_{h=1}^4p_{\gamma^h,l}\Big|\sum_{h=0}^{l-1}\|q^h\|>0\Big]\bigg]
\]
Using Propositions \ref{p1} and \ref{p2}, we have that
\[
&&\forall_{\gamma^1,\gamma^2|w_{i,j}^k\in \gamma^1,w_{i',j'}^k\in \gamma^2},\E \Big[\frac{\prod_{h=1}^4p_{\gamma^h,k}}{(w_{i,j}^{k})^2(w_{i',j'}^k)^2}\Big|\sum_{h=0}^{k-1}\|q^h\|>0\Big]\\
&&= \E \Big[\frac{(w_{i,j}^{k})^2(w_{i',j'}^k)^2 z_{\gamma^1,k}z_{\gamma^2,k}}{(w_{i,j}^{k})^2(w_{i',j'}^k)^2}\Big|\sum_{h=0}^{k-1}\|q^h\|>0\Big]\\
&&= \begin{cases} 
      \frac{\mathbbm{1}_{\sum_{h=0}^{k-1}\|q^h\|>0}}{2} & w_{i,j}^k \equiv w_{i',j'}^k\\
      \frac{\mathbbm{1}_{\sum_{h=0}^{k-1}\|q^h\|>0}}{4} & else 
   \end{cases}\\
&&= \begin{cases} 
      \frac{1}{c_4}\E \Big[\prod_{h=1}^4p_{\gamma^h,k}\Big|\sum_{h=0}^{k-1}\|q^h\|>0\Big]\ & w_{i,j}^k \equiv w_{i',j'}^k\\
      \E \Big[\prod_{h=1}^4p_{\gamma^h,k}\Big|\sum_{h=0}^{k-1}\|q^h\|>0\Big]\ & else 
   \end{cases}
\]

Inserting into Eq.~\ref{fourth}, using Theorem 3, and the fact that $c_4>1$, the second claim is proven. 
\end{proof}
\end{theorem}

We use the following proposition to aid in the proofs of Theorems 5 and 6.
\begin{proposition}\label{p4}
For vanilla fully connected network, with intermediate outputs given by:
\[
\forall_{0\leq l \leq L},~y^{l} = \sqrt{2}\phi(\frac{1}{\sqrt{n_{l-1}}}W^{l\top}y^{l-1})
\]
, where the weight matrices $W^l \in \mathbb{R}^{n_{l-1}\times n_l}$ are normally distributed, the following holds at initialization:
\[
\E[\|y^l\|^2] &=& \frac{n_l}{n_{l-1}}\E[\|y^{l-1}\|^2]\\
\E[\|y^l\|^4] &=&  \frac{n_l(n_l+5)}{n_{l-1}^2}\E[\|y^{l-1}\|^4]
\]
\begin{proof}
Absorbing the scale $\sqrt{\frac{2}{n_{l-1}}}$ into the weights, we
denote by $Z^l$ the diagonal matrix holding in its diagonal the activation variables $z_j^l$ for unit $j$ in layer $l$, and so we have:
\[
y^{l} = Z^lW^{l\top}y^{l-1}
\]
Conditioning on $R^{l-1} = \{W^1...W^{l-1}\}$ and taking expectation:
\[
\mathbb{E}\Big[\|y^l\|^2|R^{l-1} \Big] = y^{l-1\top}\E \Big[W^lZ^lW^{l\top}\Big]y^{l-1}\\
=\sum_{j=1}^{n_l}\sum_{i_1,i_2=1}^{n_{l-1}}y_{i_1}^{l-1}y_{i_2}^{l-1}\E\Big[w_{i_1,j}^lw_{i_2,j}^lz_j^l|R^{l-1}\Big]
\]
From Proposition \ref{p1}, it follows that:
\[
\E[\|y^l\|^2] = \E\Big[\E[\|y^l\|^2|R^{l-1}]\Big] = \frac{n_L}{n_{L-1}}\E[\|y^{l-1}\|^2]
\]

Similarly:
\[
&&\E\Big[\|y^l\|^4 \Big|R^{l-1}\Big] = \E\Big[(y^{L-1\top}W^LZ^LW^{L\top}y^{L-1})^2\Big|R^{l-1}\Big]\\
&&=\sum_{j_1,j_2,i_1,i_2,i_3,i_4} y_{i_1}^{l-1}  y_{i_2}^{l-1}  y_{i_3}^{L-1} y_{i_4}^{l-1}\E\Big[w_{i_1,j_1}^lw_{i_2,j_1}^lw_{i_3,j_2}^lw_{i_4,j_2}^lz_{j_1}^lz_{j_2}^l|R^{l-1}\Big]
\]
From Proposition 1, and the independence of the activation variables conditioned on $R^{l-1}$:
\[
&&\E\Big[w_{i_1,j_1}^lw_{i_2,j_1}^lw_{i_3,j_2}^lw_{i_4,j_2}^lz_{j_1}^lz_{j_2}^l|R^{l-1}\Big] \\
&&=\E\Big[w_{i_1,j_1}^lw_{i_2,j_1}^lw_{i_3,j_2}^lw_{i_4,j_2}^lz_{j_1}^lz_{j_2}^l|R^{l-1}\Big]\Big(\mathbbm{1}_{j_1=j_2,i_1=i_2=i_3=i_4} + 3\mathbbm{1}_{j_1=j_2,i_1=i_2,i_3=i_4,i_1\neq i_3}  ...\\
&&+\mathbbm{1}_{j_1\neq j_2,i_1=i_2,i_3=i_4}\Big)
\]
and so:
\[
\E\Big[\|y^l\|^4\Big] &=& \frac{n_lc_4^l}{2}\sum_{i=1}\E\Big[(y_i^{l-1})^4\Big] + \frac{6n_l}{n_{l-1}^2}\sum_{i_1\neq i_2}\E\Big[(y_{i_1}^{l-1})^2(y_{i_2}^{l-1})^2\Big]\\
&+& \frac{n_l(n_l-1)}{n_{l-1}^2}\sum_{i_1, i_2}\E\Big[(y_{i_1}^{l-1})^2(y_{i_2}^{l-1})^2\Big]\\
&=&\frac{n_l(c_4^l - 3(c_2^l)^2)}{2}\sum_{i}\E\Big((y_i^{l-1})^4\Big) + \frac{n_l(n_l+5)}{n_{l-1}^2}\E\Big[\|y^{l-1}\|^4\Big]\\
 &=& \frac{n_l\Delta}{n_{l-1}^2}\sum_{i}\E\Big((y_i^{l-1})^4\Big) + \frac{n_l(n_l+5)}{n_{l-1}^2}\E\Big[\|y^{l-1}\|^4\Big]
\]
For Gaussian distributions, $\Delta = 0$, proving the claim.

\end{proof}

\end{proposition}

\begin{proposition}\label{p5}
For a vanilla fully connected linear network, with intermediate outputs given by:
\[
\forall_{0\leq l \leq L},~y^{l} = \frac{1}{\sqrt{n_{l-1}}}W^{l\top}y^{l-1}
\]
, where the weight matrices $W^l \in \mathbb{R}^{n_{l-1}\times n_l}$ are normally distributed, the following holds at initialization:
\[
\E[\|y^l\|^2] &=& \frac{n_l}{n_{l-1}}\E[\|y^{l-1}\|^2]\\
\E[\|y^l\|^4] &=&  \frac{n_l(n_l+2)}{n_{l-1}^2}\E[\|y^{l-1}\|^4]
\]
\begin{proof}
The proof follows the derivation of Proposition \ref{p4} exactly, and will be omitted for brevity.
\end{proof}
\end{proposition}

\begin{theorem}\label{res_ntk_var}
For a constant width ResNet  ($n_0,n_{0,1}...n_{1,1}...n_L = n$), with positive initialization constants $\{\alpha_l\}_{l=1}^L$, there exists a constant $C>0$ such that:
\[
\max\Big[1,\frac{\sum_{u} \alpha_u^2}{\sum_{u,v} \alpha_u\alpha_v}\xi\Big] \leq \eta(n,L) \leq \xi
\]
where:
\[
 \xi = \exp\Big[\frac{C}{n}\sum_{l=1}^L \frac{\alpha_l}{1+\alpha_l} \Big]\Big(1 + \mathcal{O}(\frac{1}{n})\Big)
\]

\begin{proof}
Using the result of Theorem 4, and using Cauchy–Schwarz inequality, an upper bound to $\eta$ can be derived:
\[
\eta = \frac{\E[\mathcal{G}(x,x)^2]}{K_L(x,x)^2}  = \frac{\sum_{u,v}\E[\|J^u\|^2\|J^v\|^2]}{K_L(x,x)^2} \leq \frac{\sum_{u,v}\sqrt{\E[\|J^u\|^4]\E[\|J^v\|^2]}}{K_L(x,x)^2} \leq \frac{\sum_{u,v}\sqrt{\E[\|f_{(u)}\|^4]\E[\|f_{(v)}\|^2]}}{K_L(x,x)^2}
\]
The lower bound is similarly derived using Theorem 4:
\[
\eta \geq  \frac{\sum_{k}\E[\|J^k\|^4]}{K_L(x,x)^2} \geq \frac{1}{c_4}\frac{\sum_{k}\E[\|f_{(k)}\|^4]}{K_L(x,x)^2}
\]

The asymptotic behaviour of $\eta$ is therefore governed by the propagation of the fourth moment $\E[\|y_{(k)}^l\|^4]$ through the model.

In the following proof, for the sake of notation simplicity, we omit the notation $k$ in $y^l_{(k)}$, and assume that $y^l$ stands for the reduced network  $y^l_{(k)}$.
The recursive formula for the intermediate outputs of the reduced network are given by:
\[
y^l = 
\begin{cases} 
      y^{l-1} + \sqrt{\alpha_l}y^{l-1,m} & 0<l\leq L, l\neq k  \\
      \sqrt{\alpha_l}y^{l-1,m} & l=k
   \end{cases}
\]
where:
\[  y^{l-1,h} = \begin{cases} 
      \sqrt{\frac{1}{n}}W^{l,h\top}q^{l-1,h-1} & 1<h\leq m  \\
      \sqrt{\frac{1h}{n}}W^{l,h\top}y^{l-1} & h=1 
   \end{cases}
\]
with $q^{l-1,h} = \sqrt{2}\phi(y^{l-1,h})$.

We have for layer $L$, using the results of Propositions \ref{p4} and \ref{p5}:
\[\label{rm}
    \E\Big[\|y^L\|^2 \Big] &=&  \E\Big[\|y^{L-1}\|^2\Big] +  \frac{\alpha_L}{n}\E\Big[y^{L-1,m-1\top}W^{L,m}W^{L,m\top}y^{L-1,m-1}\Big]\\ &=&   \E\Big[\|y^{L-1}\|^2\Big] + \alpha_L\E\Big[\|y^{L-1,m-1}\|^2\Big] = \E\Big[\|y^{L-1}\|^2\Big]\Big(1 + \alpha_L\Big)\\
    &=& \E\Big[\|y^{k}\|^2\Big]\prod_{l=k+1}^L\Big(1 + \alpha_l\Big) = \E\Big[\|y^{k-1}\|^2\Big]\alpha_k\prod_{l=k+1}^L\Big(1 + \alpha_l\Big)\\
     &=&\alpha_k\E[\|y^0\|^4]\prod_{l\neq k}\Big(1 + \alpha_l\Big) 
\]
For the fourth moment, using the results of proposition \ref{p4} and \ref{p5} (taking into account that odd powers will vanish in expectation), it holds: 
\[\label{res}
    \E\Big[\|y^L\|^4 \Big] =  \E\Big[\|y^{L-1}\|^4\Big] +  \alpha_L^2\E\Big[\|y^{L-1,m}\|^4\Big]+ 4\alpha_L\E\Big[(y^{L-1,m\top}y^{L-1})^2\Big]
    +2\alpha_L\E\Big[\|y^{L-1,m}\|^2\|y^{L-1}\|^2\Big]
\]
We now handle each term separately:
\[
    \E\Big[\|y^{L-1,m}\|^4\Big] = \E\Big[\E[\|y^{L-1,m}\|^4|R^{L-1}]\Big]
\]
Using the results of Propositions \ref{p4} and \ref{p5}:
\[ 
 &&\E\Big[\|y^{L-1,m}\|^4|R^{L-1}\Big] = (1+\frac{2}{n})(1+\frac{5}{n})^{m-1}\|y^{L-1}\|^4 \sim (1+\frac{5}{n})^{m}\|y^{L-1}\|^4\\
    &&\E\Big[(y^{L-1,m-1\top}y^{L-1})^2\Big] = \frac{1}{n}\sum_{j_1j_2i_1i_2}\E\Big[y^{L-1,m-1}_{i_1}y^{L-1,m-1}_{i_2}y^{L-1}_{j_1}y^{L-1}_{j_2}w^{L,m}_{i_1,j_1}w^{L,m}_{i_2,j_2} \Big]\\
     &&= \frac{1}{n}\E\Big[\|y^{L-1,m-1}\|^2\|y^{L-1}\|^2 \Big] = \frac{1}{n}\E\Big[\|y^{L-1}\|^4 \Big]\\
    &&\E\Big[\|y^{L-1,m}\|^2\|y^{L-1}\|^2\Big] = \E\Big[\|y^{L-1}\|^4 \Big]
\]
Plugging it all into Eq.~\ref{res}, and denoting:
\[
\beta_L = 1+2\alpha_L(1+\frac{2}{n}) + \alpha_L^{2}(1+\frac{5}{n})^{m}\\
\]
we have after recursing through $l=L-1...k+1$:
\[
 \E\Big[\|y^L\|^4 \Big] \sim \E\Big[\|y^k\|^4 \Big]\prod_{l=k+1}^L\beta_l
\]

In the reduced architecture, the transformation from layer $k-1$ to layer $k$ is given by an $m$ layer fully connected network, with a linear layer on top, we can use the results from the vanilla case, and assigning $\|y^0\|^4 = 1$:
\begin{equation}
    \E\Big[\|y^L\|^4 \Big] = \alpha_k^{2}(1+\frac{2}{n})(1+\frac{5}{n})^{m-1}\prod_{l\neq k}^L\beta_l \sim \alpha_k^{2}(1+\frac{5}{n})^{m}\prod_{l\neq k}^L\beta_l
\end{equation}
Denoting $\rho = (1+\frac{5}{n})^\frac{m}{2}$, and using the following inequality:
\[
\beta_l \sim \Big(1 + \alpha_l\rho\Big)^2
\]
it follows that:
\[
\E[\mathcal{G}(x,x)^2] &\lessapprox& \sum_{u,v}\sqrt{\E[\|y_{(u)}^L\|^4]\E[\|y_{(v)}^L\|^2]}\\
&\sim&  (1+\frac{5}{n})^{m}\sum_{u,v}\alpha_u\alpha_v\sqrt{(\prod_{l\neq u}^L\beta_l)(\prod_{l\neq v}^L\beta_l)}\\
&=&(1+\frac{5}{n})^{m}\sum_{u,v}\alpha_u\alpha_v\Big(\prod_{l\neq u}(1 + \rho\alpha_l)\Big)\Big(\prod_{l\neq v}(1 + \rho\alpha_l)\Big)
\]
Similarly, we have:
\[
\E[\mathcal{G}(x,x)^2] \gtrapprox \sum_{k}\E[\|J^k\|^4] \sim (1+\frac{5}{n})^{m}\sum_{u}\alpha_u^2\prod_{l\neq u}^L\beta_l = (1+\frac{5}{n})^{m}\sum_{u}\alpha_u^2\prod_{l\neq u}^L(1 + \rho\alpha_l)^2
\]
Using Eq.~\ref{rm}, we have that:
\[
K_L^D(x,x)^2 = \sum_{u,v} \alpha_u \alpha_v\Big(\prod_{l\neq u}(1 + \alpha_l)\Big)\Big(\prod_{l\neq v}(1 + \alpha_l)\Big) 
\]
yielding:
\[
\frac{\E[\mathcal{G}(x,x)^2]}{K_L^R(x,x)^2} &\lessapprox& (1+\frac{5}{n})^{m}\frac{\sum_{u,v}\alpha_u\alpha_v\Big(\prod_{l\neq u}(1 + \rho\alpha_l)\Big)\Big(\prod_{l\neq v}(1 + \rho\alpha_l)\Big)}{\sum_{u,v} \alpha_u \alpha_v\Big(\prod_{l\neq u}(1 + \alpha_l)\Big)\Big(\prod_{l\neq v}(1 + \alpha_l)\Big)} \\
&\sim& (1+\frac{5}{n})^{m}\frac{\sum_{u,v}\alpha_u\alpha_v\Big(\prod_{l=1}^L(1 + \rho\alpha_l)\Big)\Big(\prod_{l=1}^L(1 + \rho\alpha_l)\Big)}{\sum_{u,v} \alpha_u \alpha_v\Big(\prod_{l=1}^L(1 + \alpha_l)\Big)\Big(\prod_{l=1}^L(1 + \alpha_l)\Big)}\\
&=&(1+\frac{5}{n})^{m}\frac{\Big(\prod_{l=1}^L(1 + \rho\alpha_l)\Big)^2}{ \Big(\prod_{l=1}^L(1 + \alpha_l)\Big)^2} = (1+\frac{5}{n})^{m}\Big(\prod_{l=1}^L(1 + \frac{\alpha_l(\rho - 1)}{1 + \alpha_l})\Big)^2\\
&\sim& \exp \bigg[\frac{C}{n}\sum_l \frac{\alpha_l}{1+\alpha_l} \bigg]\Big(1 + \mathcal{O}(\frac{1}{n})\Big)
\]
for the lower bound, we have:
\[
\frac{\E[\mathcal{G}(x,x)^2]}{K_L^R(x,x)^2} &\gtrapprox& (1+\frac{5}{n})^{m}\frac{\sum_{u}\alpha_u^2\Big(\prod_{l\neq u}^L(1 + \rho\alpha_l)\Big)^2}{\sum_{u,v} \alpha_u \alpha_v\Big(\prod_{l\neq u}(1 + \alpha_l)\Big)\Big(\prod_{l\neq v}(1 + \alpha_l)\Big)} \\
&\sim& \frac{\sum_{u} \alpha_u^2}{\sum_{u,v} \alpha_u\alpha_v}\exp \bigg[\frac{C}{n}\sum_{l=1}^L \frac{\alpha_l}{1+\alpha_l} \bigg]\Big(1 + \mathcal{O}(\frac{1}{n})\Big)
\]

Since $\E[\mathcal{G}(x,x)^2]>K_L^R(x,x)^2$, the lower bound is given by:
\[
\frac{\E[\mathcal{G}(x,x)^2]}{K_L^R(x,x)^2} &\gtrapprox& \max \Bigg[1,\frac{\sum_{u} \alpha_u^2}{\sum_{u,v} \alpha_u\alpha_v}\exp \bigg[\frac{C}{n}\sum_{l=1}^L \frac{\alpha_l}{1+\alpha_l} \bigg]\Big(1 + \mathcal{O}(\frac{1}{n})\Big)\Bigg]
\]
\end{proof}
\end{theorem}

\begin{theorem}\label{dense_ntk_var}
For a constant width DenseNet ($n_0, n_0'...n_L = n$), with initialization constant $\alpha>0$, there exists constants $C_1,C_2>0$ such that:
\[\label{surprise}
\max\Big[1,\frac{C_1}{L\log(L)^2}\xi\Big] \leq \eta(n,L) \leq \xi
\]
where:
\[
 \xi = \exp\Big[\frac{C_2}{n} \Big]\Big(1 + \mathcal{O}(\frac{1}{n})\Big)
\]
\begin{proof}
In the following proof, for the sake of notation simplicity, we omit the notation $k$ in $y^l_{(k)}$, and assume that $y^l$ stands for the reduced network  $y^l_{(k)}$.
The recursive formula for the intermediate outputs of the reduced network are given by:
\[
y^l = 
\begin{cases} 
     
 \sqrt{\frac{\alpha}{nl}}\sum_{h=k}^{l-1}W^{l,h\top}q^h
   & k<l\leq L \\
      \sqrt{\frac{\alpha}{nl}}\sum_{h=0}^{l-1}W^{l,h\top}q^h & 1\leq l <k\\
      \sqrt{\frac{\alpha}{nk}}W^{k,k-1\top}q^{k-1} & l = k
   \end{cases}
\]

with $q^h = \sqrt{2}\phi(y^h)$.
We have:
\[
    \mu_L = \mathbb{E}\Big[\|q^L\|^2 \Big] = \frac{2\alpha}{Ln}\mathbb{E}\Big[(\sum_{l=k}^{L-1}q^{l\top}W^{Ll})Z^L(\sum_{l=k}^{L-1}q^{l\top}W^{Ll}) \Big] = \frac{\alpha}{L}\sum_{l=k}^{L-1}\mu_l 
\]

Telescoping the mean:
\[
\mu_L = \frac{\alpha}{L}\sum_{l=k}^{L-1}\mu_{l} = \frac{\alpha\mu_{L-1}}{L} + \frac{L-1}{L}\mu_{L-1} =  \mu_{L-1}(1 + \frac{\alpha-1}{L}) =\mu_{k+1}\prod_{l=k+2}^L (1 + \frac{\alpha-1}{l})\\
= \frac{\alpha}{k+1}\mu_{k}\prod_{l=k+2}^L (1 + \frac{\alpha-1}{l}) =  \frac{\alpha}{k+1}\mu_{0}\prod_{l\neq k+1}^L (1 + \frac{\alpha-1}{l}) \sim \frac{\alpha}{k+1}\prod_{l=1}^L (1 + \frac{\alpha-1}{l})
\]
and so:
\[
K_L^D(x,x)^2 = (\sum_{k=1}^L \mu_L)^2 = (\sum_{k=1}^L\frac{\alpha}{k+1})^2\prod_{l=1}^L (1 + \frac{\alpha-1}{l})^2 \sim \alpha^2\log(L)^2 \prod_{l=1}^L (1 + \frac{\alpha-1}{l})^2
\]

For the fourth moment:

\[
    &&\mathbb{E}\Big[\|q^L\|^4 \Big] = \frac{4\alpha^2}{n^2L^2}\mathbb{E}\Big[\Big(\sum_{l=k}^{L-1}(q^{l\top}W^{Ll})Z^L\sum_{l=k}^{L-1}(q^{l\top}W^{Ll})\Big)^2 \Big] \\
  &&  = \frac{4\alpha^2}{n^2L^2}\mathbb{E}\Big[\Big(\sum_{l_1=k}^{L-1}(q^{l_1\top}W^{Ll_1})Z^L\sum_{l_2=k}^{L-1}(q^{l_2\top}W^{Ll_2})\sum_{l_3=k}^{L-1}(q^{l_3\top}W^{Ll_3})Z^L\sum_{l_4=k}^{L-1}(q^{l_4\top}W^{Ll_4})\Big]
\]
Using the results from the vanilla architecture, and denoting $C_{l,l'} = \E\Big[\|q^{l}\|^2\|q^{l'}\|^2 \Big]$
, it then follows:
\begin{equation}\label{c}
    C_{L,L} = \frac{\alpha^2(n+5)}{nL^2}\sum_{l_1,l_2=k}^{L-1}C_{l_1,l_2}
\end{equation}
From Eq.~\ref{c}, it also holds that:
\begin{equation}\label{c2}
   \sum_{l_1,l_2=k}^{L-2}C_{l_1,l_2} = \frac{n(L-1)^2}{(n+5)\alpha^2}C_{L-1,L-1}
\end{equation}
It then follows:
\[
    \mathbb{E}(\|q^L\|^4) &=& C_{L,L} \\
    &=& \frac{\alpha^2}{L^2}(1+\frac{5}{n})\sum_{l_1,l_2=k}^{L-1}C_{l_1l_2}\\
    &=&\frac{\alpha^2}{L^2}(1+\frac{5}{n})\Big(C_{L-1,L-1} + \sum_{l_1,l_2=k}^{L-2}C_{l_1l_2} + 2\sum_{l=k}^{L-2}C_{L-1,l} \Big)\\
    &=&\frac{\alpha^2}{L^2}(1+\frac{5}{n})\Big(C_{L-1,L-1} +  \frac{(L-1)^2n}{\alpha^2(n+5)}C_{L-1,L-1} + 2\sum_{l=k}^{L-2}C_{L-1,l} \Big)
\]

The following also holds:
\begin{equation}
\forall_{l_1> l_2\geq k},~C_{l_1,l_2} = \frac{\alpha}{nl_1}\E\Big[(\sum_{l=k}^{l_1-1}q^{l\top}W^{l_1,l}Z^{l_1})^2\|q^{l_2}\|^2\Big]
 = \frac{\alpha}{l_1}\sum_{l=k}^{l_1-1}C_{l,l_2}
\end{equation}
and so:
\[
C_{L,L} &=& \frac{\alpha^2(n+5)}{nL^2}\Big(C_{L-1,L-1} +  \frac{(L-1)^2n}{\alpha^2(n+5)}C_{L-1,L-1} + \frac{2\alpha}{L-1}\sum_{l_1=k}^{L-2}\sum_{l_2=k}^{L-2}C_{l_1,l_2} \Big)\\
&=&\frac{\alpha^2(n+5)}{nL^2}\Big(C_{L-1,L-1} +  \frac{(L-1)^2n}{\alpha^2(n+5)}C_{L-1,L-1} 
+ \frac{2n(L-1)}{\alpha(n+5)}C_{L-1,L-1}  \Big)\\
 &=& \frac{\alpha^2(n+5)}{nL^2}C_{L-1,L-1}\Big(1
 +  \frac{(L-1)^2n}{\alpha^2(n+5)} + \frac{2n(L-1)}{\alpha(n+5)}\Big)\\
  &=& C_{L-1,L-1}\bigg(\Big({1 + \frac{\alpha - 1}{L}}\Big)^2 + \frac{5\alpha^2}{nL^2}\bigg)
\]

Telescoping through $l=L-1...k+1$:
\[
C_{L,L} = C_{k+1,k+1}\prod_{l=k+2}^L \bigg(\Big({1 + \frac{\alpha - 1}{l}}\Big)^2 + \frac{5\alpha^2}{nl^2}\bigg)
\]
For the reduced architecture, the transition from $q^k$ to $q^{k+1}$ is a vanilla ReLU block, and so using the result from the vanilla architecture:
\[
C_{L,L} &=& C_{k,k}\frac{\alpha^2(n+5)}{n(k+1)^2}\prod_{l=k+2}^L \bigg(\Big({1 + \frac{\alpha - 1}{l}}\Big)^2 + \frac{5\alpha^2}{nl^2}\bigg)\\
&=& \frac{\alpha^2(n+5)}{n(k+1)^2}\prod_{l\neq k+1}\bigg(\Big({1 + \frac{\alpha - 1}{l}}\Big)^2 + \frac{5\alpha^2}{nl^2}\bigg)\\
&\sim& \frac{\alpha^2(n+5)}{n(k+1)^2}\prod_{l=1}^L\bigg(\Big({1 + \frac{\alpha - 1}{l}}\Big)^2 + \frac{5\alpha^2}{nl^2}\bigg) 
\]
where we assumed $C_{0,0} = 1$. It follows:
\[
\E[\mathcal{G}(x,x)^2] &\lessapprox& \sum_{u,v}\sqrt{\E[\|y_{(u)}^L\|^4]\E[\|y_{(v)}^L\|^2]}\\
&\sim& \Big(\sum_{k=1}^L\frac{1}{k+1}\Big)^2 \frac{\alpha^2(n+5)}{n}\prod_{l=1}^L\bigg(\Big({1 + \frac{\alpha - 1}{l}}\Big)^2 + \frac{5\alpha^2}{nl^2}\bigg)\\
&\sim& \log(L)^2\frac{\alpha^2(n+5)}{n}\prod_{l=1}^L\bigg(\Big({1 + \frac{\alpha - 1}{l}}\Big)^2 + \frac{5\alpha^2}{nl^2}\bigg)
\]

Similarly, we have:
\[
\E[\mathcal{G}(x,x)^2] \gtrapprox \sum_{k}\E[\|J^k\|^4] &=& \sum_{k=1}^L
\frac{\alpha^2(n+5)}{n(k+1)^2}\prod_{l=1}^L\bigg(\Big({1 + \frac{\alpha - 1}{l}}\Big)^2 + \frac{5\alpha^2}{nl^2}\bigg) \\
&\sim& \frac{\alpha^2(n+5)}{nL}\prod_{l=1}^L\bigg(\Big({1 + \frac{\alpha - 1}{l}}\Big)^2 + \frac{5\alpha^2}{nl^2}\bigg)
\]

yielding:
\[
\frac{\E[\mathcal{G}(x,x)^2]}{K_L^D(x,x)^2} &\lessapprox& \frac{\frac{(n+5)}{n}\prod_{l=1}^L\bigg(\Big({1 + \frac{\alpha - 1}{l}}\Big)^2 + \frac{5\alpha^2}{nl^2}\bigg)}{\prod_{l=1}^L \Big(1 + \frac{\alpha-1}{l}\Big)^2}\\
&=& \frac{(n+5)}{n}\prod_{l=1}^L\Big(1 + \frac{5\alpha^2}{n(l+\alpha-1)^2}\Big)\\
&\sim& \exp\Big[\sum_{l=1}^L\frac{5\alpha^2}{n(l+\alpha-1)^2}\Big]\Big(1 + \mathcal{O}(\frac{1}{n})\Big)\\
&\sim& \exp\Big[\frac{C}{n}\Big]\Big(1 + \mathcal{O}(\frac{1}{n})\Big)
\]

For the lower bound, we have:
\[
\frac{\E[\mathcal{G}(x,x)^2]}{K_L^D(x,x)^2} &\gtrapprox& \frac{\frac{(n+5)}{n}\prod_{l=1}^L\bigg(\Big({1 + \frac{\alpha - 1}{l}}\Big)^2 + \frac{5\alpha^2}{nl^2}\bigg)}{L\log(L)^2\prod_{l=1}^L \Big(1 + \frac{\alpha-1}{l}\Big)^2}\\
&\sim& \frac{1}{L\log(L)^2}\exp\Big[\frac{C}{n}\Big]\Big(1 + \mathcal{O}(\frac{1}{n})\Big)
\]
Since $\E[\mathcal{G}(x,x)^2]>K_L^R(x,x)^2$, the lower bound is given by:
\[
\frac{\E[\mathcal{G}(x,x)^2]}{K_L^R(x,x)^2} &\gtrapprox& \max \Bigg[1,\frac{1}{L\log(L)^2}\exp\Big[\frac{C}{n}\Big]\Big(1 + \mathcal{O}(\frac{1}{n})\Big)\Bigg]
\]
\end{proof}
\end{theorem}